\documentclass{article}
\usepackage[T1]{fontenc}
\usepackage[utf8]{inputenc}
\usepackage{lmodern}
\usepackage[a4paper,margin=25mm,footskip=10mm]{geometry}
\usepackage{microtype}
\usepackage{graphicx}
\usepackage{tikz}
\usetikzlibrary{arrows.meta,positioning,calc}
\usepackage{enumitem} 
\usepackage{float,subcaption} 
\usepackage{booktabs,multirow,makecell} 
\usepackage{algorithm,algpseudocode,listings}  
\usepackage{authblk} 
\usepackage{tcolorbox}
\usepackage{upgreek} 

\usepackage[tbtags]{amsmath}
\usepackage{amssymb,amsthm}
\usepackage{mathtools}
\usepackage{mathrsfs}
\usepackage{physics2}
\usephysicsmodule{ab} 

\usepackage{sota-commands}

\usepackage[round]{natbib} 
\usepackage{bm} 
\usepackage[colorlinks,linktoc=page,linkcolor=purple,citecolor=teal]{hyperref}
\usepackage[noabbrev,capitalize]{cleveref}

\allowdisplaybreaks

\newtheorem{theorem}{Theorem}[section]
\newtheorem*{theorem*}{Theorem}

\newtheorem*{result*}{Result}
\newtheorem{proposition}[theorem]{Proposition}
\newtheorem*{proposition*}{Proposition}

\newtheorem*{conjecture*}{Conjecture}
\newtheorem{lemma}[theorem]{Lemma}
\newtheorem*{lemma*}{Lemma}
\newtheorem{corollary}[theorem]{Corollary}
\newtheorem*{corollary*}{Corollary}

\theoremstyle{definition}
\newtheorem{assumption}[theorem]{Assumption}
\newtheorem*{assumption*}{Assumption}
\newtheorem{definition}[theorem]{Definition}
\newtheorem*{definition*}{Definition}

\theoremstyle{remark}
\newtheorem{remark}[theorem]{Remark}
\newtheorem*{remark*}{Remark}

\crefname{assumption}{Assumption}{Assumptions}

\usepackage{pifont}
\newcommand{\cmark}{\ding{51}}
\newcommand{\xmark}{\ding{55}}

\newcommand{\MP}{\mathrm{MP}}

\newcommand{\dist}{\mathsf{dist}}
\newcommand{\Etrain}{\mathscr{L}}
\newcommand{\Etest}{\mathscr{R}}
\newcommand{\ind}{\mathbb{I}}
\DeclareMathOperator{\dirac}{\updelta}

\renewcommand{\tilde}{\widetilde}
\renewcommand{\bar}{\overline}

\begin{document}

\title{High-Dimensional Limit of Stochastic Gradient Flow\\via Dynamical Mean-Field Theory}

\author[1,2]{Sota Nishiyama\thanks{\texttt{snishiyama@g.ecc.u-tokyo.ac.jp}}}
\author[1,2,3]{Masaaki Imaizumi\thanks{\texttt{imaizumi@g.ecc.u-tokyo.ac.jp}}}

\affil[1]{The University of Tokyo}
\affil[2]{RIKEN Center for Advanced Intelligence Project}
\affil[3]{Kyoto University}

\maketitle

\begin{abstract}
    Modern machine learning models are typically trained via multi-pass stochastic gradient descent (SGD) with small batch sizes, and understanding their dynamics in high dimensions is of great interest.
    However, an analytical framework for describing the high-dimensional asymptotic behavior of multi-pass SGD with small batch sizes for nonlinear models is currently missing.
    In this study, we address this gap by analyzing the high-dimensional dynamics of a stochastic differential equation called a \emph{stochastic gradient flow} (SGF), which approximates multi-pass SGD in this regime.
    In the limit where the number of data samples $n$ and the dimension $d$ grow proportionally, we derive a closed system of low-dimensional and continuous-time equations and prove that it characterizes the asymptotic distribution of the SGF parameters.
    Our theory is based on the dynamical mean-field theory (DMFT) and is applicable to a wide range of models encompassing generalized linear models and two-layer neural networks.
    We further show that the resulting DMFT equations recover several existing high-dimensional descriptions of SGD dynamics as special cases, thereby providing a unifying perspective on prior frameworks such as online SGD and high-dimensional linear regression.
    Our proof builds on the existing DMFT technique for gradient flow and extends it to handle the stochasticity in SGF using tools from stochastic calculus.
\end{abstract}

\tableofcontents

\section{Introduction}

Stochastic gradient descent (SGD) \citep{robbins1951stochastic} is a fundamental optimization algorithm widely used in machine learning.
Stochasticity in SGD affects its dynamics in a nontrivial way, and it has been observed empirically and theoretically that SGD exhibits different behaviors from the noiseless counterparts, such as gradient descent and gradient flow \citep{keskar2017largebatch,jastrzebski2017three}.
Thus, understanding the dynamics of SGD and how the noise influences learned solutions is crucial for analyzing and improving machine learning algorithms.

To understand the behavior of SGD in modern machine learning with high-dimensional models and large datasets, many works have developed frameworks for \emph{high-dimensional asymptotic analysis} for SGD, where one derives a low-dimensional equation that precisely characterizes the macroscopic behavior of the learning problem. For noiseless gradient descent and its continuous-time version, gradient flow, methods such as random matrix theory and dynamical mean-field theory (DMFT) have proven to be useful for deriving such low-dimensional equations \citep{advani2020highdimensional,bodin2021model,mignacco2020dynamical,celentano2021highdimensional,bordelon2024dynamical}. For SGD, there are several frameworks that apply to restricted settings. In particular, \emph{one-pass} or \emph{online} SGD, where each data sample is used only once during training, has been studied extensively \citep{saad1995online,goldt2019dynamics,benarous2022highdimensional}. For analysis of \emph{multi-pass} SGD, where data samples are used multiple times during training, \citet{mignacco2020dynamical,gerbelot2024rigorous} have developed a DMFT-based approach applicable to SGD with a proportionally large batch size with respect to the dataset size. Restricted to least-square linear settings, \citet{paquette2021sgd,paquette2025homogenization} analyzed multi-pass SGD with small batch sizes and derived low-dimensional Volterra equations for tracking the dynamics of summary statistics, such as the training and test errors, using the theory of \emph{homogenized SGD} (HSGD) and random matrix theory. A more detailed discussion of related work is provided in \cref{app:related_work}.

Despite these developments, there are two main challenges:
\begin{enumerate}
    \item The high-dimensional analysis of multi-pass SGD with small batch sizes and nonlinear models is currently missing.
    \item Existing analytical frameworks have been developed in isolation, and their relationships remain unclear.
\end{enumerate}
In particular, existing frameworks apply to (i) online SGD, which cannot capture the effect of overfitting, (ii) SGDs with proportionally large batch sizes, which are unrealistic in practice for large datasets, and (iii) linear models, which cannot capture the complexity of nonlinear neural networks in full.

In this work, we address these questions by studying the high-dimensional dynamics of a stochastic differential equation (SDE) called a \emph{stochastic gradient flow} (SGF), which approximates the dynamics of multi-pass SGD with a small batch size. We derive a DMFT-type system of equations that precisely describes the SGF dynamics in high dimensions and connect our new DMFT equations to existing frameworks for SGD.

We summarize our contributions below.
\begin{description}
    \item[Novel DMFT analysis of SGF.] Extending the DMFT framework, we derive a system of low-dimensional stochastic processes that characterizes the empirical distribution of the entries of the SGF parameters in the proportional high-dimensional limit, where the number of data samples $n$ and the parameter dimension $d$ diverge proportionally. See \cref{fig:roadmap-sgd-dmft} for a logical roadmap of our theory and \cref{tab:comparison_frameworks} for comparison with existing frameworks.
    \item[Clarifying connections among previous studies.] In certain limits of parameters, the DMFT equation for SGF reduces to the equations derived for noiseless gradient flow and for online SGD in previous works. Furthermore, in the special case of linear regression, we show that the DMFT equation for SGF simplifies significantly and matches the Volterra equations for describing the summary statistics of SGD dynamics for linear regression derived in \citet{paquette2021sgd,paquette2025homogenization}. These results show that our DMFT analysis of SGF provides a unified view on several lines of research on high-dimensional asymptotics of SGD.
    \item[Proof technique.] We build on the approach by \citet{celentano2021highdimensional}, who rigorously derived the DMFT equation for noiseless gradient flow. To handle the additional stochasticity in SGF, we make use of tools from stochastic calculus, a novel truncation argument, and a further manipulation of the DMFT equations using Stein's lemma.
\end{description}

Concurrent with our work, \citet{fan2026highdimensional} worked on a similar problem and derived a DMFT-type equation for describing the high-dimensional asymptotic behavior of both SGD and SGF. Our work was developed independently and offers a complementary perspective; see the discussion that appears after \cref{thm:dmft_sgf}.

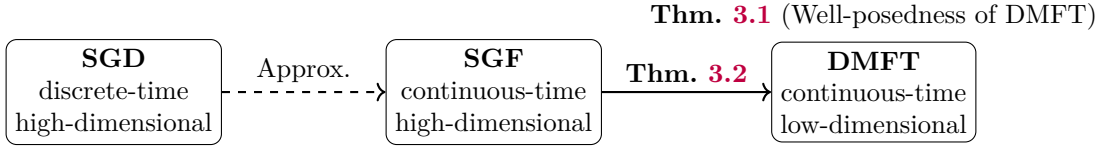
\begin{figure}[t]
    \centering
    \begin{tikzpicture}

        \node[draw, rounded corners, align=center] (SGD) at (0,0)
        {\textbf{SGD}\\{discrete-time}\\{high-dimensional}};

        \node[draw, rounded corners, align=center] (SGF) at (5,0)
        {\textbf{SGF}\\{continuous-time}\\{high-dimensional}};

        \node[draw, rounded corners, align=center] (DMFT) at (10,0)
        {\textbf{DMFT}\\{continuous-time}\\{low-dimensional}};

        \draw[->, thick, dashed] (SGD) -- node[above]
        {Approx.} (SGF);

        \draw[->, thick] (SGF) -- node[above]
        {\textbf{Thm.~\ref{thm:dmft_sgf}}} (DMFT);

        \node[align=center] (Thm1) at (10,1)
        {\textbf{Thm.~\ref{thm:dmft_sol}} (Well-posedness of DMFT)};
    \end{tikzpicture}

    \caption{
        A logical roadmap of our framework.
        We model the SGD dynamics with the stochastic gradient flow (SGF), and analyze the SGF in high dimensions.
        We show that the empirical distribution of SGF parameters converges to the DMFT solution (\cref{thm:dmft_sgf}),
        which uniquely exists (\cref{thm:dmft_sol}).
    }
    \label{fig:roadmap-sgd-dmft}
\end{figure}

\begin{table}[h]
    \caption{Comparison of frameworks for high-dimensional analysis of SGD.}
    \label{tab:comparison_frameworks}
    \begin{center}
        \small
        \begin{tabular}{c|cccc}
            \toprule
                       & \makecell{\textbf{Online SGD}                            \\\citep{benarous2022highdimensional}} & \makecell{\textbf{Previous DMFT}\\\citep{mignacco2020dynamical}} & \makecell{\textbf{HSGD}\\\citep{paquette2025homogenization}} & \textbf{Ours} \\ \midrule
            Multi-pass & \xmark                        & \cmark & \cmark & \cmark \\ \midrule
            \makecell{Small                                                       \\batch sizes} & \cmark                        & \xmark & \cmark & \cmark \\ \midrule
            \makecell{Nonlinear                                                   \\models}          & \cmark                        & \cmark & \xmark & \cmark \\
            \bottomrule
        \end{tabular}
    \end{center}
\end{table}

\paragraph{Notation.}

For vectors $\bx = (x_1,...,x_d)^\transpose$ and $\by = (y_1,...,y_d)^\transpose \in \reals^d$, $\bx \odot \by$ denotes entry-wise multiplication, i.e., $\bx \odot \by = (x_1 y_1, \dots, x_d y_d)^\transpose \in \reals^d$.
$\bI_d \in \reals^{d\times d}$ denotes the $d \times d$ identity matrix.
$\bone_d \in \reals^{d}$ denotes the all-ones vector $\bone_d = (1,\dots,1)^\transpose$.
$\GP(0,Q)$ denotes a centered Gaussian process with covariance kernel $Q$.
$\ind(\cdot)$ denotes the indicator function that returns $1$ if the argument is true and $0$ otherwise.
$\norm{\cdot}_2$ denotes the $\ell_2$ norm for vectors and the 2-operator norm for matrices.
$\frobnorm{\cdot}$ denotes the Frobenius norm for matrices.
$W_p$ denotes the $p$-Wasserstein distance between probability measures.
$\plim$ denotes convergence in probability.
$\sfP(X)$ denotes the law of a random variable $X$. $\sfP(X,Y)$ denotes the joint law of random variables $X$ and $Y$.
For a matrix $\bx \in \reals^{d \times m}$, we denote by $\hat\sfP(\bx)$ the empirical distribution of its rows, i.e., $\hat\sfP(\bx) = \frac{1}{d} \sum_{i=1}^d \delta_{x_i}$.
Similarly, $\hat\sfP(\bx,\by)$ denotes the empirical joint distribution of the rows of $\bx$ and $\by$, i.e., $\hat\sfP(\bx,\by) = \frac{1}{d} \sum_{i=1}^d \delta_{x_i,y_i}$.
$\norm{\cdot}_{\psi_2}$ denotes the sub-Gaussian norm ($\psi_2$-Orlicz norm).

\section{Setup}
\label{sec:setup}

\paragraph{Stochastic gradient descent.}
Fix $m \in \naturals$, a learning rate $\eta > 0$, a batch size $B \in \naturals$, a data matrix $\bX\in\reals^{n\times d}$, a noise vector $\bz\in\reals^n$, the initial parameter $\btheta^0\in\reals^{d \times m}$, and functions $h \colon \reals^m \times \reals_{\geq 0} \to \reals^m; (\theta,t) \mapsto h_t(\theta)$ and $\ell \colon \reals^m \times \reals \times \reals_{\geq 0} \to \reals^m; (r,z,t) \mapsto \ell_t(r;z)$.
We consider the following stochastic process for $\hat\btheta^k \in\reals^{d\times m}$ for $k = 0, 1, \dots$, initialized with $\hat\btheta^0 = \btheta^0$:
\begin{align}
    \hat\btheta^{k+1} = \hat\btheta^{k} - \eta \cdot \ab(\frac{1}{d} h_{t_k}(\hat\btheta^{k}) + \frac{1}{B} \sum_{i\in \calB^k} \bx_i \ell_{t_k}(\hat r_i^{k};z_i)^\transpose) \,, \quad \hat\br^{k} = \bX \hat\btheta^{k} \in \reals^{n \times m} \,, \label{eq:sgd}
\end{align}
where $t_k \coloneqq \eta k / d$ and $\bx_i \in \reals^d$ and $\hat r_i^k \in \reals^m$ are the $i$-th row of $\bX$ and $\hat\br^k$, respectively.
Here, $h_t$ is applied row-wise to $\hat\btheta^k$. In each update $k$, the mini-batch $\calB^k$ is sampled uniformly at random from all subsets of $[n] \coloneqq \{1,2,\dots,n\}$ with size $B$.

This stochastic process includes the mini-batch stochastic gradient descent on the following training objective with the learning rate $\eta$ and the batch size $B$:
\begin{align}
    \calL(\btheta) \coloneqq \frac{1}{n} \sum_{i=1}^n L(\btheta^\transpose \bx_i;z_i) + \frac{1}{d} \sum_{i=1}^d H(\theta_i) \,,
\end{align}
where $\theta_i \in \reals^m$ is the $i$-th row of $\btheta \in \reals^{d \times m}$, $L \colon \reals^m \times \reals \to \reals$ is a loss function, and $H \colon \reals^m \to \reals$ is a regularization function.
In this case, we have $h_t(\theta) = \nabla_{\theta} H(\theta)$ and $\ell_t(r;z) = \nabla_r L(r;z)$.
This setting includes generalized linear models ($m = 1$) and two-layer neural networks with width $m$.
In the following, we refer to the general process \eqref{eq:sgd} as an `SGD' for simplicity, although it may not correspond to the gradient descent on any objective function.

Note that $\bz$ is interpreted as a noise vector independent of the data $\bX$. The setting of supervised learning with (noisy) target labels $y_i = f_*(\btheta^{*\transpose} \bx_i) + z_i$ for some target parameter $\btheta^* \in \reals^{d \times m}$ and target function $f_* \colon \reals^m \to \reals$ is included in the above setting by considering the loss function $L(\btheta^\transpose \bx,\btheta^{*\transpose} \bx; z)$ and treating the pair $(\btheta,\btheta^*) \in \reals^{d \times 2m}$ as a single SGD parameter; see \cref{sec:planted_model} for details.

\paragraph{Stochastic gradient flow.}
Instead of working directly with the discrete-time SGD iteration \eqref{eq:sgd}, we work with its continuous-time analogue modeled by a stochastic differential equation (SDE) with the matching first and second moments, called a \emph{stochastic gradient flow} (SGF) \citep{ali2020implicit} (also called a \emph{stochastic modified equation} (SME) in the literature \citep{li2017stochastic,li2019stochastic}):
\begin{align}
    \de \btheta^t = - \ab(h_t(\btheta^t) + \frac{1}{\delta} \bX^\transpose \ell_{t}(\br^t;\bz) ) \de t + \sqrt{\frac{\tau}{\delta}} \sum_{i=1}^n \bx_i \ell_t(r^t_i;z_i)^\transpose \de B_i^t \,, \quad \br^t = \bX \btheta^t \,, \label{eq:sgf}
\end{align}
where $\bB^t=(B_i^t)_{i\in[n]}$ is a Brownian motion in $\reals^n$, and we set $\delta \coloneqq n/d$ and $\tau \coloneqq \eta/B$.
Here, $h_t$ and $\ell_t$ are applied row-wise to $\btheta^t$ and $(\br^t,\bz)$, respectively.
The parameter $\tau$ is defined as the ratio of the learning rate $\eta$ and the batch size $B$, which controls the noise intensity of the SGF and is called a \emph{temperature} parameter in the literature \citep{jastrzebski2017three}.
The time $t$ is scaled so that it corresponds to approximately $k d / \eta$ SGD updates.
Again, although we call \eqref{eq:sgf} an `SGF', it may not correspond to the gradient flow of any objective function.
In the noiseless case $\tau = 0$, the SGF \eqref{eq:sgf} reduces to the gradient flow dynamics studied in \citet{celentano2021highdimensional}, thus generalizing their setting to stochastic dynamics.

To see that the SGF \eqref{eq:sgf} approximates the SGD \eqref{eq:sgd}, consider the one-step increment of the SGD \eqref{eq:sgd} given by $\Delta\hat\btheta^k = \hat\btheta^{k+1} - \hat\btheta^k$ and compute its first and second moments conditioned on $\hat\btheta^k$:
\begin{align}
    \E[\Delta \hat\btheta^k  \mid \hat\btheta^k]  & = -\frac{\eta}{d} \ab(h_{t_k}(\hat\btheta^k) + \frac{1}{\delta} \bX^\transpose \ell_{t_k}(\hat\br^k;\bz)) \,,                                   \\
    \Cov(\Delta \hat\btheta^k \mid \hat\btheta^k) & = \frac{\eta^2 (n-B)}{Bn(n-1)}\sum_{i=1}^n \bx_i \ell_{t_k}(\hat r_i^{k};z_i)^\transpose \otimes \bx_i \ell_{t_k}(\hat r_i^{k};z_i) \notag         \\
                                                  & \qquad + \frac{\eta^2(n-B)}{Bn^2(n-1)} \bX^\transpose \ell_{t_k}(\bX\hat\btheta^k;\bz) \otimes \bX^\transpose \ell_{t_k}(\bX\hat\btheta^k;\bz) \,.
\end{align}
Here, the outer product $\bA \otimes \bB$ for matrices $\bA,\bB\in\reals^{d \times m}$ is interpreted as $\operatorname{vec}(\bA) \operatorname{vec}(\bB)^\transpose \in \reals^{dm \times dm}$, where $\operatorname{vec}$ is the vectorization operator.
In the proportional high-dimensional limit where $n,d \to \infty$ with $n/d \to \delta \in (0,\infty)$ and $B = o(n)$, the leading term of the covariance simplifies to
\begin{align}
    \Cov(\Delta \hat\btheta^k \mid \hat\btheta^k) & = \frac{\eta}{d} \cdot \frac{\tau}{\delta} \sum_{i=1}^n \bx_i \ell_{t_k}(\hat r_i^{k};z_i)^\transpose \otimes \bx_i \ell_{t_k}(\hat r_i^{k};z_i) + (\text{sub-leading terms}) \,.
\end{align}
Thus, the first and second moments of the increment $\Delta \hat\btheta^k$ match those of the increment of the SGF \eqref{eq:sgf} up to the leading order in $d$ over the time interval of length $\Delta t = \eta/d$.

The SDE approximation of SGD has been used extensively in the literature as a continuous model of discrete-time SGD \citep{mandt2017stochastic,jastrzebski2017three,chaudhari2018stochastic,ali2020implicit,pesme2021implicit}.
Tools from the It\^o stochastic calculus can be used to analyze the dynamics of the SDE in detail, and this approach has led to fruitful insights into the behavior of SGD.

Although we do not prove that the high-dimensional behaviors of the SGD \eqref{eq:sgd} and the SGF \eqref{eq:sgf} match (and indeed they do not match exactly, as shown in the concurrent work \citep{fan2026highdimensional}), we take the SGF \eqref{eq:sgf} as an approximate continuous model of the SGD \eqref{eq:sgd} and analyze its high-dimensional limiting behavior.
It is shown that in the fixed-dimensional setting, when the learning rate is small ($\eta \to 0$), the dynamics of SGD is well approximated by that of the SGF \eqref{eq:sgf} \citep{li2017stochastic,li2019stochastic,cheng2020stochastic}.
Hence, we expect that the SGF \eqref{eq:sgf} approximates the SGD \eqref{eq:sgd} well when $\tau$ is sufficiently small, which we empirically confirm by numerical experiments in \cref{sec:numerical}.

\paragraph{Assumptions.}
Next, we state the assumptions used in the analysis.
\begin{assumption}[Data distribution] \label{ass:data}
    \noindent
    \begin{itemize}
        \item The entries $\bX=(x_{ij})_{i\in[n],j\in[d]}$ are independent, satisfying $\E x_{ij}=0$, $ \E x_{ij}^2=1/d$, and $\norm{x_{ij}}_{\psi_2}\leq C/\sqrt{d}$ for some constant $C > 0$.
        \item Proportional high-dimensional asymptotics: $n,d\to\infty,\;n/d\to\delta\in(0,\infty).$
        \item $\bz\in\reals^n$ and $\btheta^0\in\reals^{d\times m}$ are independent of $\bX$, and for all $p \geq 1$, their empirical distributions $\hat \sfP(\btheta^0)$ and $\hat \sfP(\bz)$ converge in $p$-Wasserstein distance to $\sfP(\theta^0)$ and $\sfP(z)$ respectively, almost surely as $d \to \infty$.
    \end{itemize}
\end{assumption}
The distribution of the data $\bX$ is not restricted to the Gaussian distribution, and hence our analysis is \emph{universal} with respect to the data distribution.
The last condition on $\bz$ and $\btheta^0$ is satisfied, for example, when their entries are i.i.d.\ samples from distributions with bounded moments of all orders.

The following assumption is used to guarantee that the SGF solution does not grow too fast and that the DMFT equation to be introduced is well-defined.
\begin{assumption}[Function regularity] \label{ass:func}
    There exists a constant $M > 0$ such that the following holds.
    \begin{itemize}
        \item $h_t(\theta)$ and its Jacobian $\De h = (\nabla_\theta h, \partial_t h)$ are Lipschitz continuous in $t$ and $\theta$, i.e., for $t_1,t_2 \geq 0$ and $\theta_1,\theta_2 \in \reals^m$,
              \begin{equation}
                  \norm{h_{t_1}(\theta_1) - h_{t_2}(\theta_2)}_2 + \norm{\De h_{t_1}(\theta_1) - \De h_{t_2}(\theta_2)}_2 \leq M(\norm{\theta_1-\theta_2}_2 + \abs{t_1-t_2}) \,.
              \end{equation}
        \item $\ell_t(r;z)$, its Jacobian $\De\ell = (\nabla_r \ell, \partial_t \ell)$, and its Hessian $\De^2 \ell$ are Lipschitz continuous in $t$ and $r$ for any $z \in \reals$, i.e., for $t_1,t_2 \geq 0$ and $r_1,r_2 \in \reals^m$,
              \begin{multline}
                  \norm{\ell_{t_1}(r_1;z) - \ell_{t_2}(r_2;z)}_2 + \norm{\De\ell_{t_1}(r_1;z) - \De\ell_{t_2}(r_2;z)}_2 + \norm{\De^2 \ell_{t_1}(r_1;z) - \De^2 \ell_{t_2}(r_2;z)}_2 \\
                  \leq M(\norm{r_1-r_2}_2 + \abs{t_1-t_2}) \,.
              \end{multline}
        \item $\ell_t(r;z)$ has polynomial growth for any $t \geq 0$, i.e., there exists some $p \geq 1$ such that $\norm{\ell_t(r;z)}_2 \leq M(1 + \norm{r}_2 + \abs{z})^p$.
    \end{itemize}
\end{assumption}

\section{Main Result}

In this section, we first present the DMFT equation that characterizes the high-dimensional dynamics of the SGF \eqref{eq:sgf}, extending the DMFT framework for noiseless gradient flow \citep{celentano2021highdimensional}.
Then, we state our main theorem that the empirical distribution of the SGF parameters converges to the solution of the DMFT equation in the proportional high-dimensional limit.

\paragraph{DMFT equation.}
The DMFT framework provides a low-dimensional effective description of the high-dimensional dynamics \eqref{eq:sgf} by averaging out microscopic fluctuations to capture its macroscopic behavior.
In our setting, we obtain the following system of equations.
We (informally) define the DMFT equation $\mathfrak{S}$ for functions $C_\theta,C_\ell,R_\theta,\allowbreak R_\ell \colon \reals_{\geq 0}^2 \to \reals^{m \times m}$ and $\Gamma \colon \reals_{\geq 0} \to \reals^{m \times m}$ as the following system of stochastic integro-differential equations.
\begin{equation}
    \begin{aligned}
        \diff{}{t} \theta^t & = u^t - (h_t(\theta^t) + \Gamma(t)\theta^t) - \int_0^t R_\ell(t,s) \theta^{s} \de s \,, \quad u \sim \GP(0,C_\ell/\delta) \,,         \\
        r^t                 & = w^t -\frac{1}{\delta} \int_0^t R_\theta(t,s) \ell_s(r^s;z) (\de s + \sqrt{\tau\delta} \de B^s) \,, \quad w \sim \GP(0,C_\theta) \,, \\
        C_\theta(t,t')      & = \E[\theta^t \theta^{t'\transpose}] \,, \quad R_\theta(t,t') = \E\ab[\diffp{\theta^t}{u^{t'}}] \,,                                   \\
        C_\ell(t,t')        & = \E[\ell_t(r^t;z) (1 + \sqrt{\tau\delta} \dot B^t) \ell_{t'}(r^{t'};z)^\transpose (1 + \sqrt{\tau\delta} \dot B^{t'})] \,,           \\
        R_\ell(t,t')        & = \E\ab[\diffp{\ell_t(r^t;z)}{w^{t'}}] \,, \quad \Gamma(t) = \E[\nabla_r\ell_t(r^t;z)] \,,
    \end{aligned} \label{eq:dmft_informal}
\end{equation}
where $B^t$ is a Brownian motion in $\reals$, $\dot B^t = \difs{B^t}{t}$ is the formal derivative of $B^t$, and we set $R_\theta(t,t') = R_\ell(t,t') = 0$ for $t < t'$. The integral $\int f(s) (\de s + \de B^s)$ is interpreted as $\int f(s) \de s + \int f(s) \de B^s$ where the second term is the It\^o integral.
All expectations are with respect to the Gaussian processes $u$ and $w$, the Brownian motion $B^t$, and the random variables $\theta^0$ and $z$.
The formal definition of the DMFT equation $\mathfrak{S}$ appears in \cref{app:rigorous_dmft}.
Note that the appearance of the derivative $\dot B^t$ in the definition of $C_\ell$ is purely formal, and $C_\ell$ will be replaced by a more rigorous object $\Sigma_\ell$ defined in \cref{eq:Sigma_ell} in the rigorous definition.

The DMFT equation \eqref{eq:dmft_informal} consists of three types of objects:
\begin{itemize}
    \item \emph{Effective processes} $\theta^t \in \reals^m$ and $r^t \in \reals^m$ that represent the asymptotic distribution of the entries of the SGF parameters $\btheta^t$ and the predictions $\br^t = \bX \btheta^t$, respectively.
    \item \emph{Correlation functions} $C_\theta(t,t')$ and $C_\ell(t,t')$ that represent the covariance of the stochastic processes $\theta^t$ and $\ell_t(r^t;z)$, respectively.
    \item \emph{Response functions} $R_\theta(t,t')$ and $R_\ell(t,t')$ that represent the sensitivity of the stochastic processes $\theta^t$ and $\ell_t(r^t;z)$ to perturbations in the Gaussian processes $u^{t'}$ and $w^{t'}$, respectively.
\end{itemize}
The DMFT equation is a self-consistent system: the law of the effective processes $(\theta^t,r^t)$ determines the correlation and response functions $(C_\theta,C_\ell,R_\theta,R_\ell,\Gamma)$, which in turn determine the law of $(\theta^t,r^t)$.

The Brownian motion $B^t$ in the equation for $r^t$ captures the stochasticity in the SGF \eqref{eq:sgf}.
When $\tau = 0$, the Brownian motion vanishes, and the DMFT equation \eqref{eq:dmft_informal} reduces to that for noiseless gradient flow derived in \citet{celentano2021highdimensional}.
Also, the discretized version of our DMFT equation given in \cref{app:dmft_disc} has a similar form to the DMFT equation for SGD with proportionally large batch sizes \citep{mignacco2020dynamical,gerbelot2024rigorous}.
Their DMFT equation captures the SGD noise by Bernoulli random variables acting as batch selectors, while our DMFT equation has Gaussian randomness with matching moments as the Bernoulli variables.

Although the above definition \eqref{eq:dmft_informal} is informal due to the presence of the functional derivatives $\difsp{\theta^t}{u^{t'}}$ and $\difsp{\ell_t(r^t;z)}{w^{t'}}$ and the derivative of the Brownian motion $\dot B^t$, it aligns well with the existing DMFT literature \citep{mignacco2020dynamical,celentano2021highdimensional,gerbelot2024rigorous,fan2025dynamical} and can be intuitively understood.
Furthermore, in numerical analysis of the DMFT equation, we discretize time, and the above DMFT definition can be made rigorous.
For subsequent rigorous analysis, we work with the rigorous definition given in \cref{app:rigorous_dmft}.

\paragraph{Main results.}

We first establish the existence and uniqueness of the solution of the DMFT equation $\mathfrak{S}$.
We provide a proof in \cref{app:proof_dmft_sol}.

\begin{theorem}[Existence and uniqueness of the DMFT equation] \label{thm:dmft_sol}
    Suppose \cref{ass:data,ass:func} hold.
    Then, there exists some $T_* > 0$ such that for any $T \in [0, T_*]$, the DMFT equation $\mathfrak{S}$ admits a unique bounded fixed point $(C_\theta, \Sigma_\ell, R_\theta, R_\ell, \Gamma)$ on the interval $[0,T]$.
    Moreover, the stochastic processes $\theta^t$ and $r^t$ have continuous sample paths.

    If, in addition, either $\tau = 0$ or $\nabla_r^2 \ell_t(r;z) = 0$ for all $t \geq 0$, $r \in \reals^m$, and $z \in \reals$, the result holds globally for all $T > 0$ (implying $T_* = \infty$).
\end{theorem}

\Cref{thm:dmft_sol} guarantees the global existence and uniqueness of the solution for the noiseless case $\tau = 0$ (thus reducing to the result for gradient flow in \citet[Theorem 1]{celentano2021highdimensional}) and the case $\nabla_r^2 \ell_t(r;z) = 0$, i.e., when $\ell_t(r;z)$ is linear in $r$. The latter case covers linear regression settings, which we discuss in detail in \cref{sec:linear_regression}.
In the general case, we provide a rough estimate of the time horizon $T_* \gtrsim \delta^2/(\tau^2 M^7 m^2)$ up to the leading order in $\tau,M,m,1/\delta$ and ignoring the constant factor; see the proof of \cref{thm:dmft_sol} in \cref{app:proof_dmft_sol} for details.
To keep the analysis simple, we do not optimize this bound.
We believe that the solution exists globally for general nonlinear $\ell_t$ with $\tau > 0$ since the SGF \eqref{eq:sgf} does not blow up in finite time under the assumptions considered in this work, but proving this requires further technical work, which we leave for future work.

Next, we state our main theorem that provides the asymptotic characterization of the SGF \eqref{eq:sgf} by the unique solution of the DMFT equation $\mathfrak{S}$.

\begin{theorem}[DMFT characterization of SGF] \label{thm:dmft_sgf}
    Suppose \cref{ass:data,ass:func} hold and take $T_* > 0$ in \cref{thm:dmft_sol}.
    Then, for any $T \in [0,T_*]$, $L \in \naturals$, and $0 \leq t_1 < \dots < t_L \leq T$, we have
    \begin{align}
        \plim_{n,d\to\infty} W_2\ab(\hat \sfP(\btheta^{t_1},\dots,\btheta^{t_L}),\sfP(\theta^{t_1},\dots,\theta^{t_L})) & = 0 \,, \\
        \plim_{n,d\to\infty} W_2\ab(\hat \sfP(\br^{t_1},\dots,\br^{t_L},\bz),\sfP(r^{t_1},\dots,r^{t_L},z))             & = 0 \,.
    \end{align}
\end{theorem}

This result extends \citet[Theorem 2]{celentano2021highdimensional} to the SGF dynamics.

By a standard result on the Wasserstein distance \citep[Theorem 6.9]{villani2009optimal}, this result is equivalent to the following fact: for any continuous functions $\psi \colon \reals^m \to \reals$ and $\tilde \psi \colon \reals^{m} \times \reals \to \reals$ with at most quadratic growth, and any $L \in \naturals$ and $0 \leq t_1 < \dots < t_L \leq T$, we have
\begin{align}
    \plim_{n,d\to\infty} \max_{l=1,\dots,L} \abs*{\frac{1}{d}\sum_{i=1}^d \psi(\theta_i^{t_l}) - \E[\psi(\theta^{t_l})]}           & = 0 \,, \\
    \plim_{n,d\to\infty} \max_{l=1,\dots,L} \abs*{\frac{1}{n}\sum_{i=1}^n \tilde \psi(r_i^{t_l},z_i) - \E[\tilde \psi(r^{t_l},z)]} & = 0 \,.
\end{align}

\paragraph{Comparison with the concurrent work \citep{fan2026highdimensional}.}
A concurrent work by \citet{fan2026highdimensional} also derives a DMFT equation characterizing the high-dimensional dynamics of SGD and SGF (which they call SME).
Their DMFT equation for SGF is equivalent to ours, although defined differently; they define the response function $R_\ell$ as a linear operator, while we define it as a continuous function given by the expectation of a stochastic process.
Another important difference is that they consider bounded $\ell_t$, while we allow for unbounded, Lipschitz continuous $\ell_t$. This allows us to rigorously apply our theory to linear regression settings which involve unbounded loss gradients (see \cref{sec:linear_regression}).

\paragraph{Proof sketch.}
Below, we provide a proof sketch. The full proof appears in \cref{app:proof_dmft_sgf}.

We follow the approach of the proof of \citet[Theorem 2]{celentano2021highdimensional}, which showed similar results for noiseless gradient flow.
We need additional care to handle the stochastic terms in our SGF setting and to derive correction terms due to the stochasticity.
The proof proceeds as follows.
\begin{enumerate}
    \item We discretize the SGF \eqref{eq:sgf} with step size $\gamma > 0$ and control the discretization error (\cref{lem:sgf_discrete}).
    We follow the standard Euler--Maruyama discretization scheme for SDEs \citep[Section 10]{kloeden1992numerical} to develop a bound applicable to our high-dimensional settings.
    \item We map the discretized SGF to an \emph{approximate message passing} (AMP) iteration and apply the general AMP universality result \citep{wang2024universality} to characterize the asymptotic distribution of the AMP iterates using a low-dimensional \emph{state evolution} recursion (\cref{lem:dmft_sgf_discrete}).
    Due to the unboundedness of the Gaussian increments arising from the Brownian motion $B$, which prohibits direct application of the existing AMP results, we first truncate the Gaussian increments and then control the truncation error.
    We then show that the state evolution recursion is equivalent to the discretized version of the DMFT equation $\mathfrak{S}$.
    Employing the Gaussianity of the increments and utilizing Stein's lemma (the Gaussian integration by parts), we derive the correction terms in the DMFT equation due to the stochasticity.
    \item Finally, we take the continuous-time limit $\gamma \to 0$ and show that the solution of the discretized DMFT equation converges to the unique solution of the DMFT equation established in \cref{thm:dmft_sol} (\cref{lem:dmft_discrete}).
\end{enumerate}

\section{Applications and Special Cases}
\label{sec:applications}

\subsection{Infinite Data Limit and Online SGD}

In the infinite data limit $\delta \to \infty$, the response functions vanish and the DMFT equation $\mathfrak{S}$ reduces significantly.
As we show in \cref{app:infinite_data_limit}, the DMFT equation simplifies to the following form in the infinite data limit.
\begin{equation}
    \begin{gathered}
        \de \theta^t  = - (h_t(\theta^t) + \Gamma(t)\theta^t) \de t + \sqrt{\tau C_\ell(t)}\de W^t \,, \quad C_\theta(t,t') = \E[\theta^t \theta^{t'\transpose}] \,, \\
        C_\ell(t) = \E[\ell_t(w^t;z) \ell_t(w^t;z)^\transpose] \,, \quad \Gamma(t) = \E[\nabla_r\ell_t(w^t;z)] \,, \quad w \sim \GP(0,C_\theta) \,,
    \end{gathered} \label{eq:dmft_infinite_data}
\end{equation}
where $W^t$ is a Brownian motion in $\reals^m$.

This low-dimensional SDE allows us to recover existing equations derived from high-dimensional analysis of online SGD (see \cref{app:related_work} for related works).
As an example, in \cref{app:infinite_data_limit}, we show that this equation recovers the same characterization of online SGD for linear regression models derived in \citet{wang2017scaling}.

The connection between the infinite data limit $\delta \to \infty$ and online SGD can be intuitively understood as follows.
When we have a large amount of data, sampling a mini-batch of constant size $B$ at each iteration does not select previously seen data points with high probability, and in effect, each iteration uses fresh data points.
This corresponds to online SGD.

\subsection{Planted Models}
\label{sec:planted_model}

As explained in \cref{sec:setup}, the SGD for supervised learning with target labels generated from a ground-truth parameter $\btheta^* \in \reals^{d \times m}$ can be modeled as a special case of our general SGD setup \eqref{eq:sgd}.
Consider the following stochastic process with a planted signal $\btheta^* \in \reals^{d \times m}$:
\begin{align}
    \hat\btheta^{k+1} & = \hat\btheta^k - \eta \cdot \ab(\frac{1}{d} h_{t_k}(\hat\btheta^k) + \frac{1}{B} \sum_{i \in \calB^k}  \bx_i \ell_{t_k}(\hat r^k_i, r^*_i;z_i)^\transpose) \,, \label{eq:sgd_planted}
\end{align}
with $\hat \br^k = \bX \hat\btheta^k$ and $\br^* = \bX \btheta^*$.
We consider the SGF with the planted signal that approximates the above SGD dynamics:
\begin{align}
    \de\btheta^t & = - \ab(h_t(\btheta^t) + \frac{1}{\delta} \bX^\transpose \ell_t(\br^t,\br^*;\bz)) \de t + \sqrt{\frac{\tau}{\delta}} \sum_{i=1}^n \bx_i \ell_{t}(r^t_i, r^*_i;z_i)^\transpose \de B_i^t \,, \label{eq:sgf_planted}
\end{align}
with $\br^t = \bX \btheta^t$.
This can be mapped to our general SGF setup \eqref{eq:sgf} by concatenating the parameters $\btheta^t$ and $\btheta^*$ as follows.
\begin{equation}
    \begin{aligned}
        \de (\btheta^t,\btheta^*) & = - \ab(\ab(h_t(\btheta^t),0)+ \frac{1}{\delta} \bX^\transpose \ab(\ell_t(\bX(\btheta^t,\btheta^*);\bz),0)) \de t                                                               \\
                                  & \qquad + \sqrt{\frac{\tau}{\delta}} \sum_{i=1}^n \bx_i \ab(\ell_{t}((\btheta^t,\btheta^*)^\transpose \bx_i;z_i)^\transpose, 0) \de B_i^t \,. \label{eq:sgd_planted_transformed}
    \end{aligned}
\end{equation}
Here, $\btheta^*$ is constant over time.
By applying \cref{thm:dmft_sol,thm:dmft_sgf} to the concatenated parameter of shape $\reals^{d \times 2m}$, we obtain a DMFT characterization of the SGF with a planted signal \eqref{eq:sgf_planted} as a corollary.
The full form of the DMFT equation and the proof appear in \cref{app:dmft_planted}.

\begin{corollary}[DMFT characterization of SGF with a planted signal] \label{cor:sgf_dmft_planted}
    Suppose \cref{ass:data,ass:func} hold. Take $T_* > 0$ in \cref{thm:dmft_sol}.
    Furthermore, assume that $\btheta^* \in \reals^{d \times m}$ is independent of $\bX,\bz,\btheta^0$ and for all $p \geq 1$, its empirical distribution converge in $p$-Wasserstein distance to $\sfP(\theta^*)$ almost surely as $d \to \infty$.
    Then, for any $T \in [0, T_*]$, $L \in \naturals$, and $0 \leq t_1 < \dots < t_L \leq T$, we have
    \begin{align}
        \plim_{n,d\to\infty} W_2\ab(\hat\sfP(\btheta^{t_1},\dots,\btheta^{t_L},\btheta^*),\sfP(\theta^{t_1},\dots,\theta^{t_L},\theta^*))^2 & = 0 \,, \\
        \plim_{n,d\to\infty} W_2\ab(\hat\sfP(\br^{t_1},\dots,\br^{t_L},\br^*,\bz), \sfP(r^{t_1},\dots,r^{t_L},r^*,z))^2                     & = 0 \,,
    \end{align}
    where $(\theta^t,\theta^*,r^t,r^*,z)$ is the unique solution of the DMFT equation given in \cref{app:dmft_eq_planted}.
\end{corollary}

The proof appears in \cref{app:proof_sgf_dmft_planted}.
This result is an extension of \citet[Corollary 4.1]{celentano2021highdimensional} to the SGF dynamics.

\subsection{Linear Regression}
\label{sec:linear_regression}

We apply the DMFT framework for planted models developed in the previous section to analyze a special case of linear regression, in which the DMFT equation simplifies significantly.
Consider the linear regression problem with the following training objective and the corresponding SGD update:
\begin{align}
    \calL(\btheta) = \frac{1}{2n} \norm{\bX \btheta - \by}_2^2 \,, \quad \by = \bX\btheta^* + \bz\,. \label{eq:lin_setup}
\end{align}
The corresponding SGF for this problem is given as follows:
\begin{align}
    \de \btheta^t = -\frac{1}{\delta} \bX^\transpose (\bX \btheta - \by) \de t + \sqrt{\frac{\tau}{\delta}} \sum_{i=1}^n \bx_i(\bx_i^\transpose \btheta^t - y_i) \de B_i^t \,. \label{eq:sgf_lin}
\end{align}
Suppose, for simplicity, the case of zero initialization $\btheta^0 = 0$.
This is a special case of the SGF for planted models \eqref{eq:sgf_planted} with $m=1$, $h_t = 0$, and $\ell_t(r,r^*;z) = r - r^* - z$.
We are particularly interested in the training and test errors of the parameter $\btheta^t$ given by
\begin{align}
    \Etrain(\btheta) = \frac{1}{n} \sum_{i=1}^n (\bx_i^\transpose \btheta - y_i)^2 \,, \quad \Etest(\btheta) = \E_{(\bx,y)}[(\bx^\transpose \btheta - y)^2] \,.
\end{align}
Applying \cref{cor:sgf_dmft_planted}, we obtain a DMFT characterization of the SGF for linear regression \eqref{eq:sgf_lin}.
Note that since we have $\nabla_r^2 \ell_t(r,r^*;z) = 0$, the DMFT equation is guaranteed to have a unique solution for all $T > 0$ by \cref{thm:dmft_sol}.
Solving the DMFT equation, we can characterize the training and test errors of SGD for linear regression by a simple set of linear Volterra equations.

\begin{proposition}[DMFT characterization of SGD for linear regression] \label{prop:dmft_linear}
    Assume that the noise $\bz \in \reals^n$ and the target parameter $\btheta^* \in \reals^d$ satisfy the same assumptions as in \cref{cor:sgf_dmft_planted}.
    Let $\mu_\MP$ be the Marchenko--Pastur distribution with parameter $\delta$ given by
    \begin{equation}
        \mu_{\mathrm{MP}}(x) = \frac{\delta \sqrt{(\lambda_+ - x)(x - \lambda_-)}}{2\pi x} + \ab(1 - \delta) \dirac(x) \ind(\delta < 1) \,, \quad x \in [\lambda_-,\lambda_+] \,, \label{eq:mp_law}
    \end{equation}
    where $\lambda_\pm = (1 \pm 1/\sqrt{\delta})^2$ and $\dirac$ is the Dirac delta function (an upright $\dirac$ is used to distinguish it from $\delta = n/d$).

    Then, for any $0 \leq t_1,\dots,t_L < \infty$, we have
    \begin{align}
        \plim_{n,d\to\infty} \max_{l=1,\dots,L} \abs{\Etrain(\btheta^{t_l}) - \Etrain(t_l)} & = 0 \,, \quad \plim_{n,d\to\infty} \max_{l=1,\dots,L} \abs{\Etest(\btheta^{t_l}) - \Etest(t_l)} = 0 \,,
    \end{align}
    where $\Etrain(t)$ and $\Etest(t)$ solve the following system of linear Volterra equations:
    \begin{equation}
        \Etrain(t) = \Etrain_0(t) + \tau \int_0^t H_2(t-s) \Etrain(s) \de s \,, \quad \Etest(t) = \Etest_0(t) + \tau \int_0^t H_1(t-s) \Etrain(s) \de s \,, \label{eq:dmft_linear_errors}
    \end{equation}
    where $H_i(t) \coloneqq \int x^i \napier^{-2xt} \de\mu_\MP(x)$, and $\Etrain_0(t)$ and $\Etest_0(t)$ are the asymptotic train and test errors for the noiseless case $\tau = 0$, which are given by the following equations with $\rho^2 \coloneqq \E[(\theta^*)^2]$ and $\sigma^2 \coloneqq \E[z^2]$:
    \begin{equation}
        \begin{aligned}
            \Etrain_0(t) & = \rho^2 H_1(t) + \frac{\sigma^2}{\delta} H_0(t) + \frac{\delta-1}{\delta}\sigma^2 \,,                      \\
            \Etest_0(t)  & = \rho^2 H_0(t) + \frac{\sigma^2}{\delta} \int \frac{(1 - \napier^{-xt})^2}{x} \de\mu_\MP(x) + \sigma^2 \,.
        \end{aligned}
    \end{equation}
\end{proposition}

We derive this result in \cref{app:proof_dmft_ridge}. The result for non-zero ridge regularization with $h_t(\theta) = \lambda \theta$ with $\lambda > 0$ is also provided in the appendix.

The equation for $\Etrain$ in \eqref{eq:dmft_linear_errors} is a scalar linear Volterra integral equation, which can be solved numerically efficiently.
Once we obtain $\Etrain$, we can compute $\Etest$ using the second equation in \eqref{eq:dmft_linear_errors}.

The equations \eqref{eq:dmft_linear_errors} are equivalent to those derived in \citet{paquette2021sgd,paquette2025homogenization} using the theory of homogenized SGD and random matrix theory (up to time rescaling to match their settings).
Our result provides an alternative derivation of these equations using the DMFT framework (although our framework does not directly apply to SGD rigorously).
For further analysis of these equations, such as their exact solutions and long-time behaviors, see \citet{paquette2021sgd,paquette2025homogenization}.

\section{Numerical Simulations}
\label{sec:numerical}

The flexibility of our DMFT framework allows us to describe the dynamics of multi-pass SGD with small batch sizes (approximated by SGF) for nonlinear models, which have not been analyzed in prior works.
We illustrate this with a canonical example of a nonlinear model: logistic regression.
We conduct numerical simulations of SGD and compare the results with the theoretical prediction obtained by numerically solving the DMFT equation.

\paragraph{Setup and numerical procedure.}
The input data $\bx_i$ ($i=1,\dots,n$) is sampled independently from the isotropic Gaussian distribution $\normal(0,\bI_d/d)$ with $d=1024$ and $n = 2048$ ($\delta = 2$).
The labels $y_i \in \{\pm 1\}$ are generated according to a linear model with additive Gaussian noise: $y_i = \sign(\btheta^{*\transpose} \bx_i + z_i)$, where $\btheta^*$ is sampled from $\normal(0,\bI_d)$ and $z_i$ is sampled from $\normal(0,\sigma^2)$ with $\sigma^2 = 0.01$, both independently from other variables.
We use the following training objective with logistic loss and $\ell_2$ regularization:
\begin{align}
    \calL(\btheta) = \frac{1}{n} \sum_{i=1}^n \log(1 + \exp(-y_i \bx_i^\transpose \btheta)) + \frac{\lambda}{2d} \norm{\btheta}_2^2 \,.
\end{align}
We set the regularization parameter to $\lambda = 0.01$.
The multi-pass SGD for this model is given by the general SGD for planted models \eqref{eq:sgd_planted} with $m=1$, $h_t(\theta) = \lambda \theta$, and $\ell_t(r,r^*;z) = -y/(1 + \exp(y r))$ where $y = \sign(r^* + z)$.
Although this model is not formally covered by our theoretical setting due to the non-differentiability of $\ell$ with respect to $r^*$, we heuristically apply \cref{cor:sgf_dmft_planted} to derive the DMFT equation for this model.
We measure the performance of the model using the zero-one loss and define the train and test errors as
\begin{align}
    \Etrain(\btheta) = \frac{1}{n} \sum_{i=1}^n \ind(\sign(\bx_i^\transpose \btheta) \neq y_i) \,, \quad \Etest(\btheta) = \E_{(\bx,y)}[\ind(\sign(\bx^\transpose \btheta) \neq y)] \,.
\end{align}

The SGD is run with a batch size of $B=10$ and varying learning rates $\eta$ to change the temperature $\tau = \eta/B$.
We run the SGD on independent data for $10$ trials and compute the average and standard deviation of the train and test errors.

We solve the DMFT equation for logistic regression numerically by time discretization. The details of the numerical procedure are provided in \cref{app:numerics_details}.

\paragraph{Results.}
\begin{figure}
    \begin{center}
        \includegraphics[width=\textwidth]{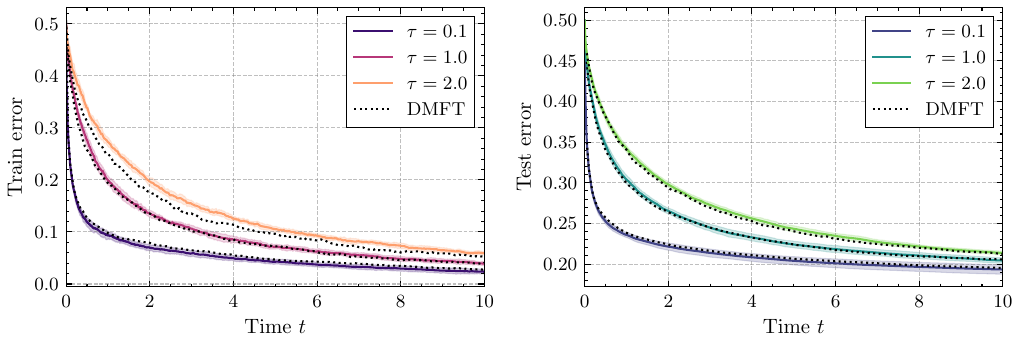}
    \end{center}
    \caption{Train (left) and test (right) error dynamics of SGD for logistic regression with various temperature values $\tau = \eta / B$. (Solid) Average errors of $10$ trials of SGD with $d=1024$ and $n=2048$. Shaded regions represent one standard deviation. (Dotted) Predictions from the DMFT equation.}
    \label{fig:simulation_logistic}
\end{figure}
We plot the train and test error dynamics of SGD for logistic regression with various temperature parameters $\tau = \eta / B$ in \cref{fig:simulation_logistic}.
The theoretical prediction from the DMFT equation shows good agreement with the empirical results from SGD simulations across different temperature settings.

\section{Discussion}

In this work, we presented a new DMFT equation that characterizes the asymptotic behavior of SGF in high dimensions, which approximates multi-pass SGD dynamics.
Our analysis opens several directions for future work, including the analysis of long-time behavior of the DMFT equations (as done for noiseless gradient flow in prior works \citep{celentano2021highdimensional}) and detailed studies of specific models (as done for gradient flow learning in prior works \citep{bordelon2024dynamical,montanari2025dynamical,nishiyama2026precise}).
We leave these directions for future investigation.

\section*{Acknowledgements}
We thank Shogo Nakakita for helpful discussions and comments on the manuscript.
Sota Nishiyama was supported by WINGS-FMSP at the University of Tokyo.
Masaaki Imaizumi was supported by JSPS KAKENHI (24K02904), JST CREST (JPMJCR21D2), JST FOREST (JPMJFR216I), and JST BOOST (JPMJBY24A9).

\bibliography{dmft-sgd}

\begin{thebibliography}{56}
\providecommand{\natexlab}[1]{#1}
\providecommand{\url}[1]{\texttt{#1}}
\expandafter\ifx\csname urlstyle\endcsname\relax
  \providecommand{\doi}[1]{doi: #1}\else
  \providecommand{\doi}{doi: \begingroup \urlstyle{rm}\Url}\fi

\bibitem[Advani et~al.(2020)Advani, Saxe, and Sompolinsky]{advani2020highdimensional}
M.~S. Advani, A.~M. Saxe, and H.~Sompolinsky.
\newblock High-dimensional dynamics of generalization error in neural networks.
\newblock \emph{Neural Networks}, 132:\penalty0 428--446, 2020.
\newblock \doi{10.1016/j.neunet.2020.08.022}.

\bibitem[Agoritsas et~al.(2018)Agoritsas, Biroli, Urbani, and Zamponi]{agoritsas2018outofequilibrium}
E.~Agoritsas, G.~Biroli, P.~Urbani, and F.~Zamponi.
\newblock Out-of-equilibrium dynamical mean-field equations for the perceptron model.
\newblock \emph{Journal of Physics A: Mathematical and Theoretical}, 51\penalty0 (8):\penalty0 085002, 2018.
\newblock \doi{10.1088/1751-8121/aaa68d}.

\bibitem[Ali et~al.(2020)Ali, Dobriban, and Tibshirani]{ali2020implicit}
A.~Ali, E.~Dobriban, and R.~Tibshirani.
\newblock The implicit regularization of stochastic gradient flow for least squares.
\newblock In \emph{Proceedings of the 37th {{International Conference}} on {{Machine Learning}}}, volume 119, pages 233--244, 2020.

\bibitem[Arnaboldi et~al.(2023)Arnaboldi, Stephan, Krzakala, and Loureiro]{arnaboldi2023highdimensional}
L.~Arnaboldi, L.~Stephan, F.~Krzakala, and B.~Loureiro.
\newblock From high-dimensional \& mean-field dynamics to dimensionless {{ODEs}}: a unifying approach to {{SGD}} in two-layers networks.
\newblock In \emph{Proceedings of {{Thirty Sixth Conference}} on {{Learning Theory}}}, volume 195, pages 1199--1227, 2023.

\bibitem[Arous et~al.(2021)Arous, Gheissari, and Jagannath]{arous2021online}
G.~B. Arous, R.~Gheissari, and A.~Jagannath.
\newblock Online stochastic gradient descent on non-convex losses from high-dimensional inference.
\newblock \emph{Journal of Machine Learning Research}, 22\penalty0 (106):\penalty0 1--51, 2021.

\bibitem[Bayati and Montanari(2011)]{bayati2011dynamics}
M.~Bayati and A.~Montanari.
\newblock The dynamics of message passing on dense graphs, with applications to compressed sensing.
\newblock \emph{IEEE Transactions on Information Theory}, 57\penalty0 (2):\penalty0 764--785, 2011.
\newblock \doi{10.1109/TIT.2010.2094817}.

\bibitem[Ben~Arous et~al.(2022)Ben~Arous, Gheissari, and Jagannath]{benarous2022highdimensional}
G.~Ben~Arous, R.~Gheissari, and A.~Jagannath.
\newblock High-dimensional limit theorems for {{SGD}}: effective dynamics and critical scaling.
\newblock In \emph{Advances in {{Neural Information Processing Systems}}}, volume~35, pages 25349--25362, 2022.

\bibitem[Berger and Mizel(1980)]{berger1980volterra}
M.~A. Berger and V.~J. Mizel.
\newblock Volterra equations with {{It\^o}} integrals---{{I}}.
\newblock \emph{Journal of Integral Equations}, 2\penalty0 (3):\penalty0 187--245, 1980.

\bibitem[Biehl and Schwarze(1995)]{biehl1995learning}
M.~Biehl and H.~Schwarze.
\newblock Learning by on-line gradient descent.
\newblock \emph{Journal of Physics A: Mathematical and General}, 28\penalty0 (3):\penalty0 643--656, 1995.
\newblock \doi{10.1088/0305-4470/28/3/018}.

\bibitem[Bodin and Macris(2021)]{bodin2021model}
A.~Bodin and N.~Macris.
\newblock Model, sample, and epoch-wise descents: exact solution of gradient flow in the random feature model.
\newblock In \emph{Advances in {{Neural Information Processing Systems}}}, volume~34, pages 21605--21617, 2021.

\bibitem[Bordelon et~al.(2024)Bordelon, Atanasov, and Pehlevan]{bordelon2024dynamical}
B.~Bordelon, A.~Atanasov, and C.~Pehlevan.
\newblock A dynamical model of neural scaling laws.
\newblock In \emph{Proceedings of the 41st {{International Conference}} on {{Machine Learning}}}, volume 235, pages 4345--4382, 2024.

\bibitem[Bordelon et~al.(2025)Bordelon, Atanasov, and Pehlevan]{bordelon2025how}
B.~Bordelon, A.~Atanasov, and C.~Pehlevan.
\newblock How feature learning can improve neural scaling laws.
\newblock \emph{Journal of Statistical Mechanics: Theory and Experiment}, 2025\penalty0 (8):\penalty0 084002, 2025.
\newblock \doi{10.1088/1742-5468/adefb1}.

\bibitem[Celentano et~al.(2020)Celentano, Montanari, and Wu]{celentano2020estimation}
M.~Celentano, A.~Montanari, and Y.~Wu.
\newblock The estimation error of general first order methods.
\newblock In \emph{Proceedings of {{Thirty Third Conference}} on {{Learning Theory}}}, volume 125, pages 1078--1141, 2020.

\bibitem[Celentano et~al.(2021)Celentano, Cheng, and Montanari]{celentano2021highdimensional}
M.~Celentano, C.~Cheng, and A.~Montanari.
\newblock The high-dimensional asymptotics of first order methods with random data.
\newblock arXiv:2112.07572, 2021.

\bibitem[Chaudhari and Soatto(2018)]{chaudhari2018stochastic}
P.~Chaudhari and S.~Soatto.
\newblock Stochastic gradient descent performs variational inference, converges to limit cycles for deep networks.
\newblock In \emph{2018 {{Information Theory}} and {{Applications Workshop}} ({{ITA}})}, pages 1--10, 2018.
\newblock \doi{10.1109/ITA.2018.8503224}.

\bibitem[Cheng et~al.(2020)Cheng, Yin, Bartlett, and Jordan]{cheng2020stochastic}
X.~Cheng, D.~Yin, P.~Bartlett, and M.~Jordan.
\newblock Stochastic gradient and {{Langevin}} processes.
\newblock In \emph{Proceedings of the 37th {{International Conference}} on {{Machine Learning}}}, volume 119, pages 1810--1819, 2020.

\bibitem[{Collins-Woodfin} et~al.(2024){Collins-Woodfin}, Paquette, Paquette, and Seroussi]{collins-woodfin2024hitting}
E.~{Collins-Woodfin}, C.~Paquette, E.~Paquette, and I.~Seroussi.
\newblock Hitting the high-dimensional notes: an {{ODE}} for {{SGD}} learning dynamics on {{GLMs}} and multi-index models.
\newblock \emph{Information and Inference: A Journal of the IMA}, 13\penalty0 (4):\penalty0 iaae028, 2024.
\newblock \doi{10.1093/imaiai/iaae028}.

\bibitem[Crisanti et~al.(1993)Crisanti, Horner, and Sommers]{crisanti1993spherical}
A.~Crisanti, H.~Horner, and H.-J. Sommers.
\newblock The spherical p-spin interaction spin-glass model: the dynamics.
\newblock \emph{Zeitschrift f\"ur Physik B Condensed Matter}, 92\penalty0 (2):\penalty0 257--271, 1993.

\bibitem[Cugliandolo(2024)]{cugliandolo2024recent}
L.~F. Cugliandolo.
\newblock Recent applications of dynamical mean-field methods.
\newblock \emph{Annual Review of Condensed Matter Physics}, 15\penalty0 (1):\penalty0 177--213, 2024.
\newblock \doi{10.1146/annurev-conmatphys-040721-022848}.

\bibitem[Cugliandolo and Kurchan(1993)]{cugliandolo1993analytical}
L.~F. Cugliandolo and J.~Kurchan.
\newblock Analytical solution of the off-equilibrium dynamics of a long-range spin-glass model.
\newblock \emph{Physical Review Letters}, 71\penalty0 (1):\penalty0 173--176, 1993.
\newblock \doi{10.1103/PhysRevLett.71.173}.

\bibitem[Fan and Wang(2026)]{fan2026highdimensional}
Z.~Fan and L.~Wang.
\newblock High-dimensional learning dynamics of multi-pass stochastic gradient descent in multi-index models.
\newblock arXiv:2601.21093, 2026.

\bibitem[Fan et~al.(2025{\natexlab{a}})Fan, Ko, Loureiro, Lu, and Shen]{fan2025dynamical}
Z.~Fan, J.~Ko, B.~Loureiro, Y.~M. Lu, and Y.~Shen.
\newblock Dynamical mean-field analysis of adaptive {{Langevin}} diffusions: propagation-of-chaos and convergence of the linear response.
\newblock arXiv:2504.15556, 2025{\natexlab{a}}.

\bibitem[Fan et~al.(2025{\natexlab{b}})Fan, Ko, Loureiro, Lu, and Shen]{fan2025dynamicala}
Z.~Fan, J.~Ko, B.~Loureiro, Y.~M. Lu, and Y.~Shen.
\newblock Dynamical mean-field analysis of adaptive {{Langevin}} diffusions: replica-symmetric fixed point and empirical {{Bayes}}.
\newblock arXiv:2504.15558, 2025{\natexlab{b}}.

\bibitem[Feng et~al.(2022)Feng, Venkataramanan, Rush, and Samworth]{feng2022unifying}
O.~Y. Feng, R.~Venkataramanan, C.~Rush, and R.~J. Samworth.
\newblock A unifying tutorial on approximate message passing.
\newblock \emph{Foundations and Trends in Machine Learning}, 15\penalty0 (4):\penalty0 335--536, 2022.
\newblock \doi{10.1561/2200000092}.

\bibitem[Gerbelot et~al.(2024)Gerbelot, Troiani, Mignacco, Krzakala, and Zdeborov{\'a}]{gerbelot2024rigorous}
C.~Gerbelot, E.~Troiani, F.~Mignacco, F.~Krzakala, and L.~Zdeborov{\'a}.
\newblock Rigorous dynamical mean-field theory for stochastic gradient descent methods.
\newblock \emph{SIAM Journal on Mathematics of Data Science}, 6\penalty0 (2):\penalty0 400--427, 2024.
\newblock \doi{10.1137/23M1594388}.

\bibitem[Glasgow(2024)]{glasgow2024sgd}
M.~Glasgow.
\newblock {{SGD}} finds then tunes features in two-layer neural networks with near-optimal sample complexity: a case study in the {{XOR}} problem.
\newblock In \emph{International {{Conference}} on {{Learning Representations}}}, 2024.

\bibitem[Goldt et~al.(2019)Goldt, Advani, Saxe, Krzakala, and Zdeborov{\'a}]{goldt2019dynamics}
S.~Goldt, M.~Advani, A.~M. Saxe, F.~Krzakala, and L.~Zdeborov{\'a}.
\newblock Dynamics of stochastic gradient descent for two-layer neural networks in the teacher-student setup.
\newblock In \emph{Advances in {{Neural Information Processing Systems}}}, volume~32, 2019.

\bibitem[Han(2025)]{han2025entrywise}
Q.~Han.
\newblock Entrywise dynamics and universality of general first order methods.
\newblock \emph{The Annals of Statistics}, 53\penalty0 (4), 2025.
\newblock \doi{10.1214/25-AOS2544}.

\bibitem[Han and Imaizumi(2025)]{han2025precise}
Q.~Han and M.~Imaizumi.
\newblock Precise gradient descent training dynamics for finite-width multi-layer neural networks.
\newblock arXiv:2505.04898, 2025.

\bibitem[Ichikawa and Hukushima(2024)]{ichikawa2024learning}
Y.~Ichikawa and K.~Hukushima.
\newblock Learning dynamics in linear {{VAE}}: posterior collapse threshold, superfluous latent space pitfalls, and speedup with {{KL}} annealing.
\newblock In \emph{Proceedings of the 27th {{International Conference}} on {{Artificial Intelligence}} and {{Statistics}}}, volume 238, pages 1936--1944, 2024.

\bibitem[Ichikawa et~al.(2025)Ichikawa, Kashiwamura, and Sakata]{ichikawa2025highdimensional}
Y.~Ichikawa, S.~Kashiwamura, and A.~Sakata.
\newblock High-dimensional learning dynamics of quantized models with straight-through estimator.
\newblock arXiv:2510.10693, 2025.

\bibitem[Jastrz{\k e}bski et~al.(2017)Jastrz{\k e}bski, Kenton, Arpit, Ballas, Fischer, Bengio, and Storkey]{jastrzebski2017three}
S.~Jastrz{\k e}bski, Z.~Kenton, D.~Arpit, N.~Ballas, A.~Fischer, Y.~Bengio, and A.~Storkey.
\newblock Three factors influencing minima in {{SGD}}.
\newblock arXiv:1711.04623, 2017.

\bibitem[Keskar et~al.(2017)Keskar, Mudigere, Nocedal, Smelyanskiy, and Tang]{keskar2017largebatch}
N.~S. Keskar, D.~Mudigere, J.~Nocedal, M.~Smelyanskiy, and P.~T.~P. Tang.
\newblock On large-batch training for deep learning: generalization gap and sharp minima.
\newblock In \emph{International {{Conference}} on {{Learning Representations}}}, 2017.

\bibitem[Kloeden and Platen(1992)]{kloeden1992numerical}
P.~E. Kloeden and E.~Platen.
\newblock \emph{Numerical {{Solution}} of {{Stochastic Differential Equations}}}.
\newblock Springer Berlin Heidelberg, 1992.
\newblock \doi{10.1007/978-3-662-12616-5}.

\bibitem[Li et~al.(2017)Li, Tai, and E]{li2017stochastic}
Q.~Li, C.~Tai, and W.~E.
\newblock Stochastic modified equations and adaptive stochastic gradient algorithms.
\newblock In \emph{Proceedings of the 34th {{International Conference}} on {{Machine Learning}}}, volume~70, pages 2101--2110, 2017.

\bibitem[Li et~al.(2019)Li, Tai, and E]{li2019stochastic}
Q.~Li, C.~Tai, and W.~E.
\newblock Stochastic modified equations and dynamics of stochastic gradient algorithms {{I}}: mathematical foundations.
\newblock \emph{Journal of Machine Learning Research}, 20\penalty0 (40):\penalty0 1--47, 2019.

\bibitem[Mandt et~al.(2017)Mandt, Hoffman, and Blei]{mandt2017stochastic}
S.~Mandt, M.~D. Hoffman, and D.~M. Blei.
\newblock Stochastic gradient descent as approximate {{Bayesian}} inference.
\newblock \emph{Journal of Machine Learning Research}, 18\penalty0 (134):\penalty0 1--35, 2017.

\bibitem[Mignacco et~al.(2020)Mignacco, Krzakala, Urbani, and Zdeborov{\'a}]{mignacco2020dynamical}
F.~Mignacco, F.~Krzakala, P.~Urbani, and L.~Zdeborov{\'a}.
\newblock Dynamical mean-field theory for stochastic gradient descent in {{Gaussian}} mixture classification.
\newblock In \emph{Advances in {{Neural Information Processing Systems}}}, volume~33, pages 9540--9550, 2020.

\bibitem[Montanari and Urbani(2025)]{montanari2025dynamical}
A.~Montanari and P.~Urbani.
\newblock Dynamical decoupling of generalization and overfitting in large two-layer networks.
\newblock In \emph{Advances in {{Neural Information Processing Systems}}}, 2025.

\bibitem[Nishiyama and Imaizumi(2026)]{nishiyama2026precise}
S.~Nishiyama and M.~Imaizumi.
\newblock Precise dynamics of diagonal linear networks: a unifying analysis by dynamical mean-field theory.
\newblock In \emph{The 29th {{International Conference}} on {{Artificial Intelligence}} and {{Statistics}}}, 2026.

\bibitem[Paquette et~al.(2021)Paquette, Lee, Pedregosa, and Paquette]{paquette2021sgd}
C.~Paquette, K.~Lee, F.~Pedregosa, and E.~Paquette.
\newblock {{SGD}} in the large: average-case analysis, asymptotics, and stepsize criticality.
\newblock In \emph{Proceedings of {{Thirty Fourth Conference}} on {{Learning Theory}}}, volume 134, pages 3548--3626, 2021.

\bibitem[Paquette et~al.(2025)Paquette, Paquette, Adlam, and Pennington]{paquette2025homogenization}
C.~Paquette, E.~Paquette, B.~Adlam, and J.~Pennington.
\newblock Homogenization of {{SGD}} in high-dimensions: exact dynamics and generalization properties.
\newblock \emph{Mathematical Programming}, 214\penalty0 (1-2):\penalty0 1--90, 2025.
\newblock \doi{10.1007/s10107-024-02171-3}.

\bibitem[Pesme et~al.(2021)Pesme, {Pillaud-Vivien}, and Flammarion]{pesme2021implicit}
S.~Pesme, L.~{Pillaud-Vivien}, and N.~Flammarion.
\newblock Implicit bias of {{SGD}} for diagonal linear networks: a provable benefit of stochasticity.
\newblock In \emph{Advances in {{Neural Information Processing Systems}}}, volume~34, pages 29218--29230, 2021.

\bibitem[Ren et~al.(2025)Ren, Nichani, Wu, and Lee]{ren2025emergence}
Y.~Ren, E.~Nichani, D.~Wu, and J.~D. Lee.
\newblock Emergence and scaling laws in {{SGD}} learning of shallow neural networks.
\newblock In \emph{Advances in {{Neural Information Processing Systems}}}, 2025.

\bibitem[Riegler and Biehl(1995)]{riegler1995online}
P.~Riegler and M.~Biehl.
\newblock On-line backpropagation in two-layered neural networks.
\newblock \emph{Journal of Physics A: Mathematical and General}, 28\penalty0 (20):\penalty0 L507--L513, 1995.
\newblock \doi{10.1088/0305-4470/28/20/002}.

\bibitem[Robbins and Monro(1951)]{robbins1951stochastic}
H.~Robbins and S.~Monro.
\newblock A stochastic approximation method.
\newblock \emph{The Annals of Mathematical Statistics}, 22\penalty0 (3):\penalty0 400--407, 1951.

\bibitem[Saad and Solla(1995{\natexlab{a}})]{saad1995exact}
D.~Saad and S.~A. Solla.
\newblock Exact solution for on-line learning in multilayer neural networks.
\newblock \emph{Physical Review Letters}, 74\penalty0 (21):\penalty0 4337--4340, 1995{\natexlab{a}}.
\newblock \doi{10.1103/PhysRevLett.74.4337}.

\bibitem[Saad and Solla(1995{\natexlab{b}})]{saad1995online}
D.~Saad and S.~A. Solla.
\newblock On-line learning in soft committee machines.
\newblock \emph{Physical Review E}, 52\penalty0 (4):\penalty0 4225--4243, 1995{\natexlab{b}}.
\newblock \doi{10.1103/PhysRevE.52.4225}.

\bibitem[Sarao~Mannelli et~al.(2020)Sarao~Mannelli, Biroli, Cammarota, Krzakala, Urbani, and Zdeborov{\'a}]{saraomannelli2020marvels}
S.~Sarao~Mannelli, G.~Biroli, C.~Cammarota, F.~Krzakala, P.~Urbani, and L.~Zdeborov{\'a}.
\newblock Marvels and pitfalls of the {{Langevin}} algorithm in noisy high-dimensional inference.
\newblock \emph{Physical Review X}, 10\penalty0 (1):\penalty0 011057, 2020.
\newblock \doi{10.1103/PhysRevX.10.011057}.

\bibitem[Sompolinsky and Zippelius(1981)]{sompolinsky1981dynamic}
H.~Sompolinsky and A.~Zippelius.
\newblock Dynamic theory of the spin-glass phase.
\newblock \emph{Physical Review Letters}, 47\penalty0 (5):\penalty0 359--362, 1981.
\newblock \doi{10.1103/PhysRevLett.47.359}.

\bibitem[Sompolinsky and Zippelius(1982)]{sompolinsky1982relaxational}
H.~Sompolinsky and A.~Zippelius.
\newblock Relaxational dynamics of the {{Edwards-Anderson}} model and the mean-field theory of spin-glasses.
\newblock \emph{Physical Review B}, 25\penalty0 (11):\penalty0 6860--6875, 1982.
\newblock \doi{10.1103/PhysRevB.25.6860}.

\bibitem[Veiga et~al.(2022)Veiga, Stephan, Loureiro, Krzakala, and Zdeborov{\'a}]{veiga2022phase}
R.~Veiga, L.~Stephan, B.~Loureiro, F.~Krzakala, and L.~Zdeborov{\'a}.
\newblock Phase diagram of stochastic gradient descent in high-dimensional two-layer neural networks.
\newblock In \emph{Advances in {{Neural Information Processing Systems}}}, volume~35, pages 23244--23255, 2022.

\bibitem[Villani(2009)]{villani2009optimal}
C.~Villani.
\newblock \emph{Optimal {{Transport}}}, volume 338 of \emph{Grundlehren der mathematischen {{Wissenschaften}}}.
\newblock Springer Berlin Heidelberg, 2009.
\newblock \doi{10.1007/978-3-540-71050-9}.

\bibitem[Wang et~al.(2017)Wang, Mattingly, and Lu]{wang2017scaling}
C.~Wang, J.~Mattingly, and Y.~M. Lu.
\newblock Scaling limit: exact and tractable analysis of online learning algorithms with applications to regularized regression and {{PCA}}.
\newblock arXiv:1712.04332, 2017.

\bibitem[Wang et~al.(2019)Wang, Hu, and Lu]{wang2019solvable}
C.~Wang, H.~Hu, and Y.~Lu.
\newblock A solvable high-dimensional model of {{GAN}}.
\newblock In \emph{Advances in {{Neural Information Processing Systems}}}, volume~32, 2019.

\bibitem[Wang et~al.(2024)Wang, Zhong, and Fan]{wang2024universality}
T.~Wang, X.~Zhong, and Z.~Fan.
\newblock Universality of approximate message passing algorithms and tensor networks.
\newblock \emph{The Annals of Applied Probability}, 34\penalty0 (4), 2024.
\newblock \doi{10.1214/24-AAP2056}.

\end{thebibliography}

\newpage
\appendix

\tableofcontents

\section{Related Works}
\label{app:related_work}

\paragraph{One-pass SGD.}

The study of online SGD using low-dimensional ODEs was pioneered in the statistical physics literature on learning in perceptrons and two-layer neural networks \citep{biehl1995learning,saad1995exact,saad1995online,riegler1995online}.
They derived a closed system of low-dimensional ODEs for macroscopic parameters such as the correlation between the student and teacher weights and analyzed their dynamics, which closely approximates the original online SGD dynamics in high dimensions.
The ODEs typically involve correction terms that account for the stochasticity of the dynamics, and analyzing these ODEs provides insights into how the noise affects the training dynamics.

Recently, these works have been put on a rigorous footing by \citet{goldt2019dynamics} using techniques developed in \citet{wang2017scaling,wang2019solvable}.
\citet{benarous2022highdimensional} extended these techniques to more general models and general scaling of learning rates.
Analysis of online SGD in high dimensions using similar techniques has been applied to a wide range of models due to its versatility and simplicity \citep{arous2021online,ichikawa2024learning,ren2025emergence,glasgow2024sgd,veiga2022phase,arnaboldi2023highdimensional,collins-woodfin2024hitting,ichikawa2025highdimensional}.

\paragraph{Multi-pass SGD.}

DMFT has recently gained attention as a powerful framework for analyzing high-dimensional random dynamics, including multi-pass GD and SGD, by reducing them to low-dimensional effective processes.
DMFT was originally developed in spin glass theory \citep{sompolinsky1981dynamic,sompolinsky1982relaxational,crisanti1993spherical,cugliandolo1993analytical} and has been applied to analyzing various high-dimensional optimization dynamics \citep{agoritsas2018outofequilibrium,saraomannelli2020marvels,cugliandolo2024recent}.

In the context of SGD dynamics, \citet{mignacco2020dynamical} derived DMFT equations for multi-pass gradient flow and SGD in shallow neural networks, heuristically using statistical physics techniques.
To avoid the problem of vanishing stochasticity in the continuous-time limit $\eta \to 0$, they considered a variant of SGD called \emph{persistent SGD} to retain nontrivial noise in the continuous-time limit.
Their analysis depends on the batch size growing proportionally to the number of samples.
In contrast, we work with a stochastic gradient flow which approximates mini-batch SGD with small (sublinear) batch sizes compared to the number of samples, which is a common setting in practice.

There are several rigorous works that derived DMFT equations for GD/SGD.
\citet{celentano2021highdimensional} rigorously derived DMFT equations for gradient flow dynamics in shallow neural networks by using time discretization and mapping to approximate message passing \citep{bayati2011dynamics,feng2022unifying}.
We build upon their proof technique to analyze SGD dynamics in this work.
\citet{gerbelot2024rigorous} derived DMFT equations for discrete-time GD and SGD for shallow neural networks with batch sizes proportional to the number of samples and a constant number of updates.
More recently, \citet{fan2025dynamical,fan2025dynamicala} derived DMFT equations for Langevin dynamics of Bayesian linear regression.
A closely related line of work is the study of \emph{general first order methods} (GFOMs), which provides a framework for analyzing a broad class of iterative algorithms, including GD, using a low-dimensional recursion similar to DMFT \citep{celentano2020estimation,han2025entrywise,han2025precise}.

DMFT equations have been used for analyzing long-time behavior of optimization dynamics and provided insights into deep learning phenomena such as scaling laws and timescale separation \citep{bordelon2024dynamical,bordelon2025how,montanari2025dynamical,nishiyama2026precise}.

For linear models, there is a framework that analyzes multi-pass SGD with small batch sizes and proportionally many updates using continuous-time equations.
\citet{paquette2021sgd} derived a low-dimensional and continuous-time Volterra equation characterizing the training loss dynamics of SGD in high-dimensional linear regression models.
\citet{paquette2025homogenization} extended this work and introduced an SDE called \emph{homogenized SGD} (HSGD) as a high-dimensional equivalent of SGD dynamics in linear regression.
The HSGD framework allows deriving equations for macroscopic quantities of SGD dynamics, such as training and test errors.
Our work provides a similar framework for broader settings, including generalized linear models and shallow neural networks.

\section{Definition of the DMFT Equation}
\label{app:rigorous_dmft}

In this section, we provide a rigorous definition of the DMFT system $\mathfrak{S}$ introduced informally in \cref{eq:dmft_informal}.
The key idea is to define all objects through well-defined auxiliary stochastic processes, avoiding functional derivatives and the formal derivative of the Brownian motion.
We also introduce a discretized DMFT equation for which the connection between the informal definition and the rigorous definition is more transparent.
The discretized DMFT equation is also used as an intermediate step in the proof of \cref{thm:dmft_sgf} and as a numerical method for solving the DMFT equation (see \cref{app:numerics_details}).

\subsection{Rigorous Definition of the DMFT System}

We rigorously define the DMFT system $\mathfrak{S}$ for functions $C_\theta,\Sigma_\ell,R_\theta,R_\ell \colon \reals_{\geq 0}^2 \to \reals^{m \times m}$ and $\Gamma \colon \reals_{\geq 0} \to \reals^{m \times m}$ self-consistently as follows.
First, given $\Sigma_\ell,R_\ell,\Gamma$, define stochastic processes $\{\theta^t \in \reals^m \}_{t \geq 0}$ and $\{\rho_\theta^{t,t'} \in \reals^{m \times m} \}_{t \geq t' \geq 0}$ by the following equations.
\begin{align}
    \theta^t           & = \theta^0 + U^t - \int_0^t \ab(h_s(\theta^s) + \Gamma(s) \theta^s + \int_0^s R_\ell(s,s') \theta^{s'} \de s') \de s \,, \quad U \sim \GP(0,\Sigma_\ell/\delta) \,, \label{eq:theta} \\
    \rho_\theta^{t,t'} & = I_m - \int_{t'}^t \ab((\nabla_\theta h_s(\theta^s) + \Gamma(s)) \rho_\theta^{s,t'} + \int_{t'}^s R_\ell(s,s') \rho_\theta^{s',t'} \de s') \de s \,. \label{eq:rho_theta}
\end{align}
Then, set $C_\theta,R_\theta$ as
\begin{align}
    C_\theta(t,t') & = \E[\theta^t \theta^{t'\transpose}] \,, \label{eq:C_theta}       \\
    R_\theta(t,t') & = \E[\rho_\theta^{t,t'}] \quad (t \geq t') \,, \label{eq:R_theta}
\end{align}
and $R_\theta(t,t') = 0$ for $t < t'$. Here, the expectation is with respect to the randomness of $\theta^0 \sim \sfP(\theta^0)$ and the Gaussian process $U$.

Next, given $C_\theta,R_\theta$, define stochastic processes $\{r^t \in \reals^m\}_{t \geq 0}$ and $\{\rho_\ell^{t,t'},D_\ell^{t,t'} \in \reals^{m \times m} \}_{t \geq t' \geq 0}$ by the following equations.
\begin{align}
    r^t              & = w^t -\frac{1}{\delta} \int_0^t R_\theta(t,s) \ell_s(r^s;z) (\de s + \sqrt{\tau\delta} \de B^s) \,, \quad w \sim \GP(0,C_\theta) \,, \label{eq:r}                                                           \\
    \rho_\ell^{t,t'} & = \nabla_r \ell_t(r^t;z) \rho_r^{t,t'} \,, \label{eq:rho_ell}                                                                                                                                                \\
    D_\ell^{t,t'}    & = \nabla_r \ell_t(r^t;z) \ab( -\frac{1}{\delta} \int_{t'}^t R_\theta(t,s) D_\ell^{s,t'} (\de s + \sqrt{\tau\delta} \de B^s)) + \nabla_r^2 \ell_t(r^t;z)[D_r^{t,t'}] \cdot \rho_r^{t,t'} \,, \label{eq:D_ell}
\end{align}
where $B^t$ is a Brownian motion in $\reals$, and we defined the auxiliary processes $\rho_r^{t,t'} \in \reals^{m \times m}$ and $D_r^{t,t'} \in \reals^m$ as
\begin{align}
    \rho_r^{t,t'} & = -\frac{1}{\delta} \int_{t'}^t R_\theta(t,s) \rho_\ell^{s,t'} (\de s + \sqrt{\tau\delta} \de B^s) - \frac{1}{\delta} R_\theta(t,t') \nabla_r \ell_{t'}(r^{t'};z) \,, \label{eq:rho_r}                 \\
    D_r^{t,t'}    & = -\frac{1}{\delta} \int_{t'}^t R_\theta(t,s) \nabla_r \ell_s(r^s;z) D_r^{s,t'} (\de s + \sqrt{\tau\delta} \de B^s) - \sqrt{\frac{\tau}{\delta}} R_\theta(t,t') \ell_{t'}(r^{t'};z) \,, \label{eq:D_r}
\end{align}
and the notation $\nabla_r^2 \ell_t(r^t;z)[v]$ for $v \in \reals^m$ in \cref{eq:D_ell} denotes the product
\begin{align}
    \sum_{i=1}^m \partial_{r_i} \nabla_r \ell_t(r^t;z) v_i \in \reals^{m \times m} \,.
\end{align}
Then, set $\Sigma_\ell,R_\ell,\Gamma$ as
\begin{align}
    \Sigma_\ell(t,t') & = \E[L^t L^{t'\transpose}] \,, \quad L^t \coloneqq \int_0^t \ell_s(r^s;z) (\de s + \sqrt{\tau\delta} \de B^s) \,, \label{eq:Sigma_ell} \\
    R_\ell(t,t')      & = \E[\rho_\ell^{t,t'}] + \sqrt{\tau\delta} \E[D_\ell^{t,t'}] \quad (t \geq t') \,, \label{eq:R_ell}                                    \\
    \Gamma(t)         & = \E[\nabla_r\ell_t(r^t;z)] \,, \label{eq:Gamma}
\end{align}
and $R_\ell(t,t') = 0$ for $t < t'$. Here, the expectation is with respect to the randomness of $z \sim \sfP(z)$, the Gaussian process $w$, and the Brownian motion $B$.

Then, the solution of the DMFT system $\mathfrak{S}$ is defined as a fixed point of the above two mappings.

Note that the above definition is written purely in terms of standard It\^o integrals and is thus amenable to rigorous analysis.

\subsection{Discretized DMFT System}
\label{app:dmft_disc}

We present a discretized version of the DMFT system defined above, in which the time variable $t$ is discretized with step size $\gamma > 0$.

Let $t_i \coloneqq i \gamma$ for $i = 0,1,2,\dots$.
We define the discretized DMFT system $\mathfrak{S}^\gamma$ by the following equations.
\begin{subequations}
    \begin{align}
        \theta_\gamma^{t_i}            & = \theta^0 + U_\gamma^{t_i} - \gamma \sum_{j=0}^{i-1} \ab(h_{t_j}(\theta_\gamma^{t_j}) + \Gamma^\gamma(t_j)\theta_\gamma^{t_j} + \gamma \sum_{k=0}^{j-1} R_\ell^\gamma(t_j,t_k) \theta_\gamma^{t_k}) \,, \label{eq:dmft_disc_theta}                               \\
        \rho_{\theta,\gamma}^{t_i,t_j} & = I_m - \gamma \sum_{k=j+1}^{i-1} \ab((\nabla_\theta h_{t_k}(\theta_\gamma^{t_k}) + \Gamma^\gamma(t_k)) \rho_{\theta,\gamma}^{t_k,t_j} + \gamma \sum_{l=j+1}^{k-1} R_\ell^\gamma(t_k,t_l) \rho_{\theta,\gamma}^{t_l,t_j})  \,, \label{eq:dmft_disc_rhotheta}      \\
        r_\gamma^{t_i}                 & = w_\gamma^{t_i} -\frac{1}{\delta} \sum_{j=0}^{i-1} R_\theta^\gamma(t_i,t_j) \ell_{t_j}(r_\gamma^{t_j};z) (\gamma + \sqrt{\tau \delta} (B^{t_{j+1}} - B^{t_j})) \,, \label{eq:dmft_disc_r}                                                                        \\
        \rho_{\ell,\gamma}^{t_i,t_j}   & = \nabla_r \ell_{t_i}(r_\gamma^{t_i};z) \rho_{r,\gamma}^{t_i,t_j} \,, \label{eq:dmft_disc_rhol}                                                                                                                                                                   \\
        \rho_{r,\gamma}^{t_i,t_j}      & = -\frac{1}{\delta} \sum_{k=j+1}^{i-1} R_\theta^\gamma(t_i,t_k) \rho_{\ell,\gamma}^{t_k,t_j} (\gamma + \sqrt{\tau\delta} (B^{t_{k+1}} - B^{t_k})) - \frac{1}{\delta} R_\theta^\gamma(t_i,t_j) \nabla_r \ell_{t_j}(r_\gamma^{t_j};z) \,, \label{eq:dmft_disc_rhor} \\
        D_{\ell,\gamma}^{t_i,t_j}      & = \nabla_r \ell_{t_i}(r_\gamma^{t_i};z) \ab( -\frac{1}{\delta} \sum_{k=j+1}^{i-1} R_\theta^\gamma(t_i,t_k) D_{\ell,\gamma}^{t_k,t_j} (\gamma + \sqrt{\tau\delta} (B^{t_{k+1}} - B^{t_k}))) \notag                                                                 \\
                                       & \qquad + \nabla_r^2 \ell_{t_i}(r_\gamma^{t_i};z)[D_{r,\gamma}^{t_i,t_j}] \cdot \rho_{r,\gamma}^{t_i,t_j} \,, \label{eq:dmft_disc_Drhol}                                                                                                                           \\
        D_{r,\gamma}^{t_i,t_j}         & = -\frac{1}{\delta} \sum_{k=j+1}^{i-1} R_\theta^\gamma(t_i,t_k) \nabla_r \ell_{t_k}(r_\gamma^{t_k};z) D_{r,\gamma}^{t_k,t_j} (\gamma + \sqrt{\tau\delta} (B^{t_{k+1}} - B^{t_k})) \notag                                                                          \\
                                       & \qquad - \sqrt{\frac{\tau}{\delta}} R_\theta^\gamma(t_i,t_j) \ell_{t_j}(r_\gamma^{t_j};z) \,, \label{eq:dmft_disc_Dr}
    \end{align} \label{eq:dmft_disc_process}
\end{subequations}
where $(U_\gamma^{t_i},w_\gamma^{t_i})_{i \geq 0}$ satisfies
\begin{align}
    \E[U_\gamma^{t_i} U_\gamma^{t_j\transpose}] & = \Sigma_\ell^\gamma(t_i,t_j) / \delta \,, \quad \E[w_\gamma^{t_i} w_\gamma^{t_j\transpose}] = C_\theta^\gamma(t_i,t_j) \,.
\end{align}
Then, set $C_\theta^\gamma,R_\theta^\gamma,\Sigma_\ell^\gamma,R_\ell^\gamma,\Gamma^\gamma$ as
\begin{align}
    \begin{gathered}
        C_\theta^\gamma(t_i,t_j) = \E[\theta_\gamma^{t_i} \theta_\gamma^{t_j\transpose}] \,, \quad R_\theta^\gamma(t_i,t_j) = \E[\rho_{\theta,\gamma}^{t_i,t_j}] \,, \\
        \Sigma_\ell^\gamma(t_i,t_j) = \E[L_\gamma^{t_i} L_\gamma^{t_j\transpose}] \,, \quad L_\gamma^{t_i} \coloneqq \sum_{k=0}^{i-1} \ell_{t_k}(r_\gamma^{t_k};z) (\gamma + \sqrt{\tau\delta} (B^{t_{k+1}} - B^{t_k})) \,, \\
        R_\ell^\gamma(t_i,t_j) = \E[\rho_{\ell,\gamma}^{t_i,t_j}] + \sqrt{\tau\delta} \E[D_{\ell,\gamma}^{t_i,t_j}] \,, \quad \Gamma^\gamma(t_i) = \E\ab[\nabla_r\ell_{t_i}(r_\gamma^{t_i};z)] \,,
    \end{gathered} \label{eq:dmft_disc_func}
\end{align}
where we set $R_\theta^\gamma(t_i,t_j) = R_\ell^\gamma(t_i,t_j) = 0$ for $i \leq j$.

We will show that the solution of the discretized DMFT equation $\mathfrak{S}^\gamma$ converges to the unique solution of the continuous-time DMFT equation $\mathfrak{S}$ as $\gamma \to 0$ in \cref{lem:dmft_discrete}.

\subsection{Correspondence to the Informal Definition}

Once discretized, it is easy to see the correspondence between the rigorous definition of the DMFT system given above and the informal definition given in \cref{eq:dmft_informal}.
We distinguish the variables in the two definitions by writing bars over the variables in the informal definition, e.g., $\bar \theta^t$, $\bar r^t$, etc.

We first discretize the informal definition in \cref{eq:dmft_informal} with step size $\gamma > 0$ in the same manner as in the previous section. We obtain
\begin{equation}
    \begin{aligned}
        \frac{\bar\theta^{t_{i+1}} - \bar\theta^{t_i}}{\gamma} & = \bar u^{t_i} - (h_{t_i}(\bar \theta^{t_i}) + \bar\Gamma(t_i)\bar\theta^{t_i}) - \gamma \sum_{j=0}^{i-1} \bar R_\ell(t_i,t_j) \bar\theta^{t_j} \,, \quad \bar u \sim \GP(0,\bar C_\ell/\delta) \,,                               \\
        \bar r^{t_i}                                           & = \bar w^{t_i} - \frac{1}{\delta} \sum_{j=1}^{i-1} \bar R_\theta(t_i,t_j) \ell_{t_j}(\bar r^{t_j};z) (\gamma + \sqrt{\tau\delta} (\bar B^{t_{j+1}} - \bar B^{t_j})) \,, \quad \bar w \sim \GP(0,\bar C_\theta) \,,                \\
        \bar C_\theta(t_i,t_j)                                 & = \E[\bar \theta^{t_i} \bar \theta^{t_j\transpose}] \,, \quad \bar R_\theta(t_i,t_j) = \frac{1}{\gamma} \E\ab[\diffp{\bar \theta^{t_i}}{\bar u^{t_j}}] \quad (i > j) \,,                                                          \\
        \bar C_\ell(t_i,t_j)                                   & = \E\ab[\ell_{t_i}(\bar r^{t_i};z) \ab(1 + \sqrt{\tau\delta} \frac{\bar B^{t_{i+1}} - \bar B^{t_i}}{\gamma}) \ell_{t_j}(\bar r^{t_j};z)^\transpose \ab(1 + \sqrt{\tau\delta} \frac{\bar B^{t_{j+1}} - \bar B^{t_j}}{\gamma})] \,, \\
        \bar R_\ell(t_i,t_j)                                   & = \frac{1}{\gamma} \E\ab[\diffp{\ell_{t_i}(\bar r^{t_i};z)}{\bar w^{t_j}}] \,, \quad \bar \Gamma(t_i) = \E[\nabla_r\ell_{t_i}(\bar r^{t_i};z)] \quad (i > j) \,.
    \end{aligned} \label{eq:dmft_informal_disc}
\end{equation}

Then, we transform the above equations to show their correspondence to the discretized DMFT system $\mathfrak{S}^\gamma$.
The equations for $\bar r^{t_i}$, $\bar C_\theta$, and $\bar \Gamma$ directly correspond to definitions of $r_\gamma^{t_i}$, $C_\theta^\gamma$, and $\Gamma^\gamma$ in $\mathfrak{S}^\gamma$.
Next, we show correspondence for $\bar\theta^{t_i}$ and $\bar C_\ell$.
Summing the equation for $\bar \theta^{t_i}$ in \eqref{eq:dmft_informal_disc} over $i$ and multiplying by $\gamma$, we obtain
\begin{align}
    \bar \theta^{t_i} - \bar\theta^0 = \gamma \sum_{j=0}^{i-1} \bar u^{t_j} - \gamma \sum_{j=0}^{i-1} \ab(h_{t_j}(\bar \theta^{t_j}) + \bar \Gamma(t_j)\bar \theta^{t_j} + \gamma \sum_{k=0}^{j-1} \bar R_\ell(t_j,t_k) \bar \theta^{t_k}) \,, \quad \bar u \sim \GP(0,\bar C_\ell/\delta) \,.
\end{align}
Let $\bar U^{t_i} \coloneqq \gamma \sum_{j=0}^{i-1} \bar u^{t_j}$. $\bar U^{t_i}$ is a Gaussian process with covariance given by
\begin{align}
    \frac{1}{\delta} \bar \Sigma_\ell(t_i,t_j) \coloneqq \E[\bar U^{t_i} \bar U^{t_j\transpose}] = \gamma^2 \sum_{k=0}^{i-1} \sum_{l=0}^{j-1} \E[\bar u^{t_k} \bar u^{t_l\transpose}] = \frac{\gamma^2}{\delta} \sum_{k=0}^{i-1} \sum_{l=0}^{j-1} \bar C_\ell(t_k,t_l) = \frac{1}{\delta} \E[\bar L^{t_i} \bar L^{t_j\transpose}] \,,
\end{align}
where we set
\begin{align}
    \bar L^{t_i} \coloneqq \sum_{j=0}^{i-1} \ell_{t_j}(\bar r^{t_j};z) (\gamma + \sqrt{\tau\delta} (\bar B^{t_{j+1}} - \bar B^{t_j})) \,.
\end{align}
Thus, $\bar \theta^{t_i}$ and $\bar \Sigma_\ell(t_i,t_j) = \gamma^2 \sum_{k=0}^{i-1} \sum_{l=0}^{j-1} \bar C_\ell(t_k,t_l)$ correspond to $\theta_\gamma^{t_i}$ and $\Sigma_\ell^\gamma$ in $\mathfrak{S}^\gamma$.
Furthermore, differentiating $\bar \theta^{t_i}$ with respect to $\bar u^{t_j}$ ($j < i$), we obtain
\begin{align}
    \diffp{\bar \theta^{t_i}}{\bar u^{t_j}} = \gamma I_m - \sum_{k=j+1}^{i-1} \ab((\nabla_\theta h_{t_k}(\bar \theta^{t_k}) + \bar \Gamma(t_k)) \diffp{\bar \theta^{t_k}}{\bar u^{t_j}} + \gamma \sum_{l=j+1}^{k-1} \bar R_\ell(t_k,t_l) \diffp{\bar \theta^{t_l}}{\bar u^{t_j}}) \,.
\end{align}
This shows correspondence for $\gamma^{-1} \difsp{\bar \theta^{t_i}}{\bar u^{t_j}}$ and $\rho_{\theta,\gamma}^{t_i,t_j}$ and thus $\bar R_\theta$ and $R_\theta^\gamma$.
Finally, we show correspondence for $\bar R_\ell$.
Differentiating $\ell_{t_i}(\bar r^{t_i};z)$ with respect to $\bar w^{t_j}$ ($j < i$), we obtain
\begin{align}
    \diffp{\ell_{t_i}(\bar r^{t_i};z)}{\bar w^{t_j}} & = \nabla_r \ell_{t_i}(\bar r^{t_i};z) \diffp{\bar r^{t_i}}{\bar w^{t_j}} \,,                                                                                                         \\
    \diffp{\bar r^{t_i}}{\bar w^{t_j}}               & = -\frac{1}{\delta} \sum_{k=j+1}^{i-1} \bar R_\theta(t_i,t_k) \diffp{\ell_{t_k}(\bar r^{t_k};z)}{\bar w^{t_j}} (\gamma + \sqrt{\tau\delta} (\bar B^{t_{k+1}} - \bar B^{t_k})) \notag \\
                                                     & \qquad - \frac{1}{\delta} \bar R_\theta(t_i,t_j) \nabla_r \ell_{t_j}(\bar r^{t_j};z) (\gamma + \sqrt{\tau\delta} (B^{t_{j+1}} - B^{t_j})) \,.
\end{align}
Let $\bar \rho_\ell^{t_i,t_j}$ and $\bar \rho_r^{t_i,t_j}$ be the solution of the following equation:
\begin{align}
    \bar\rho_\ell^{t_i,t_j} & = \nabla_r \ell_{t_i}(\bar r^{t_i};z) \bar\rho_r^{t_i,t_j} \,,                                                                                                                                                               \\
    \bar\rho_r^{t_i,t_j}    & = -\frac{1}{\delta} \sum_{k=j+1}^{i-1} \bar R_\theta(t_i,t_k) \bar\rho_\ell^{t_k,t_j} (\gamma + \sqrt{\tau\delta} (B^{t_{k+1}} - B^{t_k})) - \frac{1}{\delta} \bar R_\theta(t_i,t_j) \nabla_r \ell_{t_j}(\bar r^{t_j};z) \,.
\end{align}
These equations correspond to definitions of $\rho_{\ell,\gamma}^{t_i,t_j}$ and $\rho_{r,\gamma}^{t_i,t_j}$.
By the linearity of the above equations, we see that
\begin{align}
    \diffp{\ell_{t_i}(\bar r^{t_i};z)}{\bar w^{t_j}} = \bar\rho_\ell^{t_i,t_j} (\gamma + \sqrt{\tau\delta} (B^{t_{j+1}} - B^{t_j})) \,.
\end{align}
Let $\bar G^j \coloneqq (\bar B^{t_{j+1}} - \bar B^{t_j}) / \sqrt{\gamma}$. Then, we have $\bar G^j \sim \normal(0,1)$ i.i.d.\ for $j=0,1,2,\dots$.
Taking the expectation of the above equation, we obtain
\begin{align}
    \bar R_\ell(t_i,t_j) = \frac{1}{\gamma} \E\ab[\diffp{\ell_{t_i}(\bar r^{t_i};z)}{\bar w^{t_j}}] = \E[\bar\rho_\ell^{t_i,t_j}] + \sqrt{\frac{\tau\delta}{\gamma}} \E[\bar\rho_\ell^{t_i,t_j} \bar G^j] \,.
\end{align}
By Stein's lemma (Gaussian integration by parts), we obtain
\begin{align}
    \E[\bar\rho_\ell^{t_i,t_j} \bar G^j] & = \E\ab[\diffp{\bar\rho_\ell^{t_i,t_j}}{\bar G^j}] \,.
\end{align}
Differentiating $\bar\rho_\ell^{t_i,t_j}$ with respect to $\bar G^j$ and using independence of $\bar r^{t_k}$ and $\bar G^j$ for $k \leq j$, we obtain
\begin{align}
    \diffp{\bar\rho_\ell^{t_i,t_j}}{\bar G^j} & = \nabla_r \ell_{t_i}(\bar r^{t_i};z) \diffp{\bar\rho_r^{t_i,t_j}}{\bar G^j} + \nabla^2_r \ell_{t_i}(\bar r^{t_i};z)\ab[\diffp{\bar r^{t_i}}{\bar G^j}] \cdot \bar\rho_r^{t_i,t_j} \,,                                                                                \\
    \diffp{\bar\rho_r^{t_i,t_j}}{\bar G^j}    & = -\frac{1}{\delta} \sum_{k=j+1}^{i-1} \bar R_\theta(t_i,t_k) \diffp{\bar\rho_\ell^{t_k,t_j}}{\bar G^j} (\gamma + \sqrt{\tau\delta\gamma} \bar G^k) \,,                                                                                                               \\
    \diffp{\bar r^{t_i}}{\bar G^j}            & = -\frac{1}{\delta} \sum_{k=j+1}^{i-1} \bar R_\theta(t_i,t_k) \nabla_r \ell_{t_k}(\bar r^{t_k};z) \diffp{\bar r^{t_k}}{\bar G^j} (\gamma + \sqrt{\tau\delta\gamma} \bar G^k) - \sqrt{\frac{\tau\gamma}{\delta}} \bar R_\theta(t_i,t_j) \ell_{t_j}(\bar r^{t_j};z) \,.
\end{align}
This shows that $\gamma^{-1/2} \difsp{\bar\rho_\ell^{t_i,t_j}}{\bar G^j}$ and $\gamma^{-1/2} \difsp{\bar r^{t_i}}{\bar G^j}$ correspond to $D_{\ell,\gamma}^{t_i,t_j}$ and $D_{r,\gamma}^{t_i,t_j}$.
Thus, we have correspondence for $\bar R_\ell$ and $R_\ell^\gamma$.

Summarizing, the informal definition corresponds to the formal definition by the following correspondence (informal definition on the left, rigorous definition on the right):
\begin{equation}
    \begin{gathered}
        \bar \theta^{t_i} \Leftrightarrow \theta_\gamma^{t_i} \,, \quad \frac{1}{\gamma} \diffp{\bar\theta^{t_i}}{\bar u^{t_j}} \Leftrightarrow \rho_{\theta,\gamma}^{t_i,t_j} \,, \quad \bar C_\theta(t_i,t_j) \Leftrightarrow C_\theta^\gamma(t_i,t_j) \,, \quad \bar R_\theta(t_i,t_j) \Leftrightarrow R_\theta^\gamma(t_i,t_j) \,, \\
        \bar r^{t_i} \Leftrightarrow r_\gamma^{t_i} \,, \quad \bar \rho_\ell^{t_i,t_j} \Leftrightarrow \rho_{\ell,\gamma}^{t_i,t_j} \,, \quad \frac{1}{\sqrt{\gamma}} \diffp{\bar \rho_\ell^{t_i,t_j}}{\bar G^{t_j}} \Leftrightarrow D_{\ell,\gamma}^{t_i,t_j} \,, \quad \frac{1}{\gamma} \E\ab[\diffp{\ell_{t_i}(\bar r^{t_i};z)}{\bar w^{t_j}}] \Leftrightarrow \E[\rho_{\ell,\gamma}^{t_i,t_j}] + \sqrt{\tau\delta} \E[D_{\ell,\gamma}^{t_i,t_j}] \\
        \gamma^2 \sum_{k=0}^{i-1} \sum_{l=0}^{j-1} \bar C_\ell(t_k,t_l) \Leftrightarrow \Sigma_\ell^\gamma(t_i,t_j) \,, \quad \bar R_\ell(t_i,t_j) \Leftrightarrow R_\ell^\gamma(t_i,t_j) \,, \quad \bar \Gamma(t_i) \Leftrightarrow \Gamma^\gamma(t_i) \,.
    \end{gathered}
\end{equation}
Although these two definitions are equivalent in discrete time, the informal definition does not have a well-defined continuous-time limit as $\gamma \to 0$, while the rigorous definition does.
Thus, for theoretical purposes, we work with the rigorous definition. In numerics, however, we work with the informal definition after time discretization, as it leads to simpler numerical schemes (see \cref{app:numerics_details} for details).

\section{Proof of Theorem \ref{thm:dmft_sol}}
\label{app:proof_dmft_sol}

We prove \cref{thm:dmft_sol} using a contraction mapping argument similar to that of \citet[Theorem 1]{celentano2021highdimensional} and \citet[Theorem 2.4]{fan2025dynamical}. It proceeds as follows.
\begin{enumerate}
    \item For $T > 0$, we define \emph{admissible spaces} $\calS_\theta(T)$ and $\calS_\ell(T)$ for the DMFT objects $(C_\theta,R_\theta)$ and $(\Sigma_\ell,R_\ell,\Gamma)$, respectively.
    We define mappings $\calT_{\theta\to\ell}\colon \calS_\theta(T) \to \calS_\ell(T)$ and $\calT_{\ell\to\theta} \colon \calS_\ell(T) \to \calS_\theta(T)$ such that the fixed point of their composition $\calT \coloneqq \calT_{\ell\to\theta} \circ \calT_{\theta\to\ell}$ solves the DMFT system.
    We show that for sufficiently small $T > 0$, these mappings are well-defined.
    \item Next, we construct a metric on the function spaces $\calS_\theta(T)$ and $\calS_\ell(T)$ such that $\calT$ is a contraction.
    \item Finally, we apply Banach's fixed point theorem to show the uniqueness and existence of the fixed point of $\calT$.
\end{enumerate}

\subsection{Admissible Spaces $\calS_\theta(T)$ and $\calS_\ell(T)$}
\label{app:function_spaces}

\newcommand{\calScont}{\calS^\mathrm{cont}}

Since the following quantities are bounded by assumptions, we take $M > 0$ sufficiently large so that we have
\begin{align}
    \max\ab\{\E\norm{\theta^0}_2^2, \; \sup_{t \in [0,T]} h_t(0), \; \sup_{t \in [0,T]} \E \norm{\ell_t(0;z)}_2^{2p} \} \leq M \,,
\end{align}
for $p = 1,2$.

For $T > 0$, we define admissible spaces $\calS_\theta(T)$ and $\calS_\ell(T)$ as follows.

\begin{definition}[Admissible space $\calS_\theta(T)$] \label{def:admissible_theta}
    Let $D \subset (0,T)$ be a finite set.
    Let $\calS_\theta(T)$ be a set of function pairs $(C_\theta,R_\theta)$ defined on $[0,T]^2$.
    We say that $\calS_\theta(T)$ is \emph{admissible} if there exist constants $\Phi_\theta, M_\theta > 0$ such that every $(C_\theta,R_\theta) \in \calS_\theta(T)$ satisfies the following.
    \begin{itemize}
        \item $C_\theta$ is a covariance kernel (in particular, it satisfies $C_\theta(t,t') = C_\theta(t',t)^\transpose$) and satisfies $\norm{C_\theta(t,t)}_2 \leq \Phi_\theta$ for all $t \in [0,T]$ and $C_\theta(0,0) = \E[\theta^0 \theta^{0\transpose}]$.
              Furthermore, $C_\theta(t,t')$ is uniformly continuous over $t,t' \in I$ for each maximal interval $I$ of $[0,T]\setminus D$ and satisfies
              \begin{align}
                  \norm{C_\theta(t,t) - 2C_\theta(t,t') + C_\theta(t',t')}_2 \leq M_\theta \abs{t-t'} \,, \label{eq:bound_C_theta}
              \end{align}
              for any $t, t' \in I$.
        \item $R_\theta$ satisfies $R_\theta(t,t') = 0$ for $0 \leq t < t' \leq T$ and $\norm{R_\theta(t,t')}_2 \leq \Phi_\theta$ for $0 \leq t' \leq t \leq T$.
              Furthermore, $R_\theta(t,t')$ is uniformly continuous over $t \in I$ and $t' \in I'$ for any two maximal intervals $I,I'$ of $[0,T]\setminus D$.
    \end{itemize}

    We define $\calScont_\theta(T)$ as the subset of $\calS_\theta(T)$ with $D = \emptyset$ in the above definition.
\end{definition}

\begin{definition}[Admissible space $\calS_\ell(T)$] \label{def:admissible_ell}
    Let $D \subset (0,T)$ be a finite set.
    Let $\calS_\ell(T)$ be a set of function triples $(\Sigma_\ell,R_\ell,\Gamma)$ defined on $[0,T]^2$ and $[0,T]$.
    We say that $\calS_\ell(T)$ is \emph{admissible} if there exist constants $\Phi_\ell, M_\ell > 0$ such that every $(\Sigma_\ell,R_\ell,\Gamma) \in \calS_\ell(T)$ satisfies the following.
    \begin{itemize}
        \item $\Sigma_\ell(t,s)$ is a covariance kernel (in particular, it satisfies $\Sigma_\ell(t,s) = \Sigma_\ell(s,t)^\transpose$) and satisfies $\norm{\Sigma_\ell(t,t)}_2 \leq \Phi_\ell$ for $t \in [0,T]$ and $\Sigma_\ell(0,0) = \E[\ell_0(r^0;z) \ell_0(r^0;z)^\transpose]$ for $r^0 \sim \normal(0, \E[\theta^0 \theta^{0\transpose}])$.
              Furthermore, $\Sigma_\ell(t,t')$ is uniformly continuous over $t,t' \in I$ for each maximal interval $I$ of $[0,T]\setminus D$ and satisfies
              \begin{align}
                  \norm{\Sigma_\ell(t,t) - 2 \Sigma_\ell(t,t') + \Sigma_\ell(t',t')}_2 \leq M_\ell \abs{t-t'} \,, \label{eq:bound_Sigma_l}
              \end{align}
              for any $t, t' \in I$.
        \item $R_\ell(t,t')$ satisfies $R_\ell(t,t') = 0$ for $0 \leq t < t' \leq T$ and $\norm{R_\ell(t,t')}_2 \leq \Phi_\ell$ for $0 \leq t' \leq t \leq T$.
              Furthermore, $R_\ell(t,t')$ is uniformly continuous over $t \in I$ and $t' \in I'$ for any two maximal intervals $I,I'$ of $[0,T]\setminus D$.
        \item $\Gamma(t)$ satisfies $\norm{\Gamma(t)}_2 \leq M$ for $t \in [0,T]$ and $\Gamma(0) = \E[\nabla_r \ell_0(r^0;z)]$ for $r^0 \sim \normal(0, \E[\theta^0 \theta^{0\transpose}])$.
              Furthermore, $\Gamma(t)$ is uniformly continuous over $t \in I$ for each maximal interval $I$ of $[0,T]\setminus D$.
    \end{itemize}

    We define $\calScont_\ell(T)$ as the subset of $\calS_\ell(T)$ with $D = \emptyset$ in the above definition.
\end{definition}

In the above definitions, we allow for discontinuities at a finite set of time points $D$ to handle the discretized DMFT system later in the proof of \cref{thm:dmft_sgf} in \cref{app:proof_dmft_sgf}.

We now show that the stochastic processes are uniquely defined given functions in admissible spaces.

\begin{lemma} \label{lem:process}
    Given an admissible space $\calS_\theta(T)$ and any element $(C_\theta,R_\theta) \in \calS_\theta(T)$, there exists a unique tuple of stochastic processes $\{r^t,\rho_\ell^{t,t'},D_\ell^{t,t'}\}_{0\leq t' \leq t \leq T}$ satisfying \cref{eq:r,eq:rho_ell,eq:D_ell}.
    Furthermore, for any $(C_\theta,R_\theta) \in \calScont_\theta(T)$, the processes $\{r^t,\rho_\ell^{t,t'},D_\ell^{t,t'}\}_{0\leq t' \leq t \leq T}$ have continuous sample paths.

    Similarly, given an admissible space $\calS_\ell(T)$ and any element $(\Sigma_\ell,R_\ell,\Gamma) \in \calS_\ell(T)$, there exists a unique pair of stochastic processes $\{\theta^t,\rho_\theta^{t,t'}\}_{0\leq t' \leq t \leq T}$ satisfying \cref{eq:theta,eq:rho_theta}.
    Furthermore, for any $(\Sigma_\ell,R_\ell,\Gamma) \in \calScont_\ell(T)$, the processes $\{\theta^t,\rho_\theta^{t,t'}\}_{0\leq t' \leq t \leq T}$ have continuous sample paths.
\end{lemma}

\begin{proof}
    First, we show that $r^t$ is uniquely defined.
    Let $\{w^t\}_{t \in [0,T]}$ be a centered Gaussian process with covariance kernel $C_\theta$.
    Then, for any maximal interval $I$ of $[0,T]\setminus D$ and any $t, t' \in I$, we have
    \begin{align}
        \E \norm{w^t - w^{t'}}_2^4 & \leq 3m^2 \norm{C_\theta(t,t) - 2 C_\theta(t,t') + C_\theta(t',t')}_2^2 \leq 3m^2 M_\theta^2 (t-t')^2 \,.
    \end{align}
    By the Kolmogorov continuity theorem, there exists a modification of $w^t$ that is locally H\"older continuous on $I$.
    Then, for each maximal interval $I$, $r^t$ follows a nonlinear Volterra stochastic integral equation of the second kind with a Lipschitz nonlinearity, a continuous kernel, and a continuous forcing term.
    By \citet[Theorem 3.A]{berger1980volterra}, it has a unique continuous solution adapted to the filtration $\calF_t^\ell$ generated by $(\bB^s)_{s \leq t}$.
    Applying this argument inductively over the maximal intervals of $[0,T]\setminus D$, we conclude that $r^t$ is uniquely defined over $[0,T]$.
    The well-posedness of $\rho_\ell^{t,t'}$ and $D_\ell^{t,t'}$ can be shown similarly using the continuity of $r^t$.

    Next, we show that $\theta^t$ is uniquely defined.
    Let $\{U^t\}_{t \in [0,T]}$ be a centered Gaussian process with covariance kernel $\Sigma_\ell/\delta$.
    Again, by the Kolmogorov continuity theorem, there exists a modification of $U^t$ that is locally H\"older continuous on each $I$.
    From \cref{eq:theta}, $\theta^t$ satisfies the following equation:
    \begin{align}
        \theta^t & = \theta^0 + U^t - \int_0^t \ab(h_s(\theta^s) + \Gamma(s)\theta^s + \ab(\int_{s}^t R_\ell(t',s) \de t') \theta^{s}) \de s \,. \label{eq:theta_alt}
    \end{align}
    For each $I$, this is a nonlinear Volterra integral equation of the second kind with a continuous kernel and a continuous forcing term.
    Again, by \citet[Theorem 3.A]{berger1980volterra}, it has a unique continuous solution.
    Applying this argument inductively over the maximal intervals of $[0,T]\setminus D$, we conclude that $\theta^t$ is uniquely defined over $[0,T]$.
    The well-posedness of $\rho_\theta^{t,t'}$ can be shown similarly using the continuity of $\theta^t$.
\end{proof}

Next, we define mappings between the admissible spaces.

First, we define the map $\calT_{\theta \to \ell} \colon (C_\theta,R_\theta) \mapsto (\Sigma_\ell,R_\ell,\Gamma)$.
Given $(C_\theta,R_\theta) \in \calS_\theta$, take the unique processes $r^t,\rho_\ell^{t,t'},D_\ell^{t,t'}$ satisfying \cref{eq:r,eq:rho_ell,eq:D_ell} whose existence is guaranteed by \cref{lem:process}.
Then, we define $(\Sigma_\ell,R_\ell,\Gamma)$ by \cref{eq:Sigma_ell,eq:R_ell,eq:Gamma}.

Next, we define the map $\calT_{\ell \to \theta} \colon (\Sigma_\ell,R_\ell,\Gamma) \mapsto (C_\theta,R_\theta)$.
Given $(\Sigma_\ell,R_\ell,\Gamma) \in \calS_\ell$, take the unique processes $\theta^t,\rho_\theta^{t,t'}$ satisfying \cref{eq:theta,eq:rho_theta} whose existence is guaranteed by \cref{lem:process}.
Then, we define $(C_\theta,R_\theta)$ by \cref{eq:C_theta,eq:R_theta}.

Finally, we define the composite map $\calT = \calT_{\ell \to \theta} \circ \calT_{\theta \to \ell}$.

In the following lemma, we show that for sufficiently small $T > 0$, $\calT_{\theta \to \ell}$ and $\calT_{\ell \to \theta}$ map $\calS_\theta(T)$ into $\calS_\ell(T)$ and $\calS_\ell(T)$ into $\calS_\theta(T)$, respectively.
We defer the proof to \cref{app:proof_mapping}.

\begin{lemma} \label{lem:mapping}
    \begin{enumerate}
        \item There exists some $T_* > 0$ such that, for any $0 < T \leq T_*$, there exist admissible spaces $\calS_\theta(T)$ and $\calS_\ell(T)$ such that $\calT_{\theta \to \ell}$ maps $\calS_\theta(T)$ into $\calScont_\ell(T)$ and $\calT_{\ell \to \theta}$ maps $\calS_\ell(T)$ to $\calScont_\theta(T)$.
        \item If either $\tau = 0$ or $\nabla^2 \ell_t(r;z) = 0$, $T$ can be taken arbitrarily large (thus $T_* = \infty$).
    \end{enumerate}
\end{lemma}

\subsection{Equipping Metrics on $\calS_\theta(T)$ and $\calS_\ell(T)$}

In the following, we fix $T > 0$ such that \cref{lem:mapping} holds and fix admissible spaces $\calS_\theta \coloneqq \calS_\theta(T)$ and $\calS_\ell \coloneqq \calS_\ell(T)$.

We equip the spaces $\calS_\theta$ and $\calS_\ell$ with metrics. For a constant $\lambda > 0$, we define
\begin{subequations}
    \begin{align}
        \dist_\lambda(C_\theta^1,C_\theta^2)       & \coloneqq \inf_{w_1 \sim \GP(0,C_\theta^1), w_2 \sim \GP(0,C_\theta^2)} \sup_{0 \leq t \leq T} \napier^{-\lambda t} \sqrt{\E\norm{w_1^t - w_2^t}_2^2} \,,                     \\
        \dist_\lambda(\Sigma_\ell^1,\Sigma_\ell^2) & \coloneqq \inf_{U_1 \sim \GP(0,\Sigma_\ell^1/\delta), U_2 \sim \GP(0,\Sigma_\ell^2/\delta)} \sup_{0 \leq t \leq T} \napier^{-\lambda t} \sqrt{\E\norm{U_1^t - U_2^t}_2^2} \,, \\
        \dist_\lambda(R_\theta^1,R_\theta^2)       & \coloneqq \sup_{0 \leq s \leq t \leq T} \napier^{-\lambda t} \norm{R_\theta^1(t,s) - R_\theta^2(t,s)} \,,                                                                     \\
        \dist_\lambda(R_\ell^1,R_\ell^2)           & \coloneqq \sup_{0 \leq s \leq t \leq T} \napier^{-\lambda t} \norm{R_\ell^1(t,s) - R_\ell^2(t,s)} \,,                                                                         \\
        \dist_\lambda(\Gamma^1,\Gamma^2)           & \coloneqq \sup_{0 \leq t \leq T} \napier^{-\lambda t} \norm{\Gamma^1(t) - \Gamma^2(t)} \,.
    \end{align} \label{eq:metrics}
\end{subequations}
In the first two definitions, the infima are taken over all couplings of the Gaussian processes with given marginal covariances.
Finally, for $X^i = (C_\theta^i,R_\theta^i) \in \calS_\theta$ and $Y^i = (\Sigma_\ell^i,R_\ell^i,\Gamma^i) \in \calS_\ell$, we define the distances
\begin{align}
    \dist_\lambda(X^1,X^2) & \coloneqq \dist_\lambda(C_\theta^1,C_\theta^2) + \dist_\lambda(R_\theta^1,R_\theta^2) \,,                                                     \\
    \dist_\lambda(Y^1,Y^2) & \coloneqq \sqrt{\lambda} \dist_\lambda(\Sigma_\ell^1,\Sigma_\ell^2) + \dist_\lambda(R_\ell^1,R_\ell^2) + \dist_\lambda(\Gamma^1,\Gamma^2) \,.
\end{align}
Notice the $\sqrt{\lambda}$ factor in front of $\dist_\lambda(\Sigma_\ell^1,\Sigma_\ell^2)$.

We show several properties of the metric spaces $(\calS_\theta,\dist_\lambda)$ and $(\calS_\ell,\dist_\lambda)$.

\begin{lemma} \label{lem:complete}
    The metric spaces $(\calS_\theta,\dist_\lambda)$ and $(\calS_\ell,\dist_\lambda)$ are complete.
\end{lemma}

\begin{proof}
    For finite $T$, the distance $\dist_\lambda$ for $R_\theta$, $R_\ell$, and $\Gamma$ are equivalent to $L^\infty$ distance and hence complete.
    Completeness for $C_\theta$ and $\Sigma_\ell$ are shown in the proof of \citet[Theorem 2.4(b)]{fan2025dynamical}.
\end{proof}

\begin{lemma} \label{lem:modulus_theta_ell}
    Let $X^i=(C_\theta^i,R_\theta^i) \in \calS_\theta$ and $Y^i=\calT_{\theta \to \ell}(X^i) = (\Sigma_\ell^i,R_\ell^i,\Gamma^i) \in \calS_\ell$ for $i=1,2$.
    Then, there exists a constant $K > 0$ such that for any sufficiently large $\lambda > 0$ in the definitions of the metrics \eqref{eq:metrics}, we have
    \begin{align}
        \dist_\lambda(Y^1,Y^2) \leq K \cdot \dist_\lambda(X^1,X^2) \,.
    \end{align}
\end{lemma}

\begin{lemma} \label{lem:modulus_ell_theta}
    Let $Y^i=(C_\ell^i,R_\ell^i,\Gamma^i) \in \calS_\ell$ and $X^i=\calT_{\ell \to \theta}(Y^i) = (C_\theta^i,R_\theta^i) \in \calS_\theta$ for $i=1,2$.
    Then, for any $\eps > 0$, for sufficiently large $\lambda > 0$, we have
    \begin{align}
        \dist_\lambda(X^1,X^2) \leq \eps \cdot \dist_\lambda(Y^1,Y^2) \,.
    \end{align}
\end{lemma}

We defer the proof of the last two lemmas to \cref{app:proof_modulus_theta_ell,app:proof_modulus_ell_theta}.

Finally, we show that $\calT$ is a contraction mapping under the above metrics, finishing the proof of \cref{thm:dmft_sol}.
Take $T \in [0,T_*]$ where $T_*$ is as in \cref{lem:mapping}, and take admissible spaces $\calS_\theta \coloneqq \calS_\theta(T)$ and $\calS_\ell \coloneqq \calS_\ell(T)$.
By \cref{lem:modulus_theta_ell,lem:modulus_ell_theta}, we can choose $\eps < 1/K$ and $\lambda$ sufficiently large such that $\calT = \calT_{\ell \to \theta} \circ \calT_{\theta \to \ell}$ is a contraction mapping on the metric space $(\calS_\theta,\dist_\lambda)$ which is complete by \cref{lem:complete}.
By the Banach fixed-point theorem, there exists a unique fixed point $(C_\theta,R_\theta) \in \calScont_\theta$ such that $\calT(C_\theta,R_\theta) = (C_\theta,R_\theta)$.
Thus, this $(C_\theta,R_\theta)$ and $(\Sigma_\ell,R_\ell,\Gamma) = \calT_{\theta \to \ell}(C_\theta,R_\theta) \in \calScont_\ell$ together form a unique pair of fixed points satisfying the DMFT equation $\mathfrak{S}$.

The continuity of the sample paths follows from \cref{lem:process}.

\subsection{Proof of Lemma \ref{lem:mapping}} \label{app:proof_mapping}

We will use the following bounds repeatedly in the proof.

\begin{lemma} \label{lem:bound_int}
    Let $X_1$ be a random variable and $X_2^s,X_3^s$ be stochastic processes in $\reals^m$ adapted to the Brownian motion $(B^{s'})_{0 \leq s' < s}$.
    Then, for any $0 \leq t' \leq t$ and any integer $p \geq 1$, there exists a constant $C_p > 0$ such that we have
    \begin{align}
        & \E\norm*{X_1 + \int_{t'}^{t} X_2^s \de s + \int_{t'}^{t} X_3^s \de B^s}_2^{2p} \notag \\
        & \qquad \leq 3^{2p-1} \ab(\E\norm{X_1}_2^{2p} + (t - t')^{2p-1} \int_{t'}^t \E\norm{X_2^s}_2^{2p} \de s + C_p (t - t')^{p-1} \int_{t'}^t \E\norm{X_3^s}_2^{2p} \de s) \,. \label{eq:bound_int}
    \end{align}
\end{lemma}

\begin{proof}
    By Jensen's inequality, we have
    \begin{align}
        & \E\norm*{X_1 + \int_{t'}^{t} X_2^s \de s + \int_{t'}^{t} X_3^s \de B^s}_2^{2p} \notag \\
        & \leq 3^{2p-1} \ab(\E\norm{X_1}_2^{2p} + (t - t')^{2p-1} \int_{t'}^t \E\norm{X_2^s}_2^{2p} \de s + \E\norm*{\int_{t'}^t X_3^s \de B^s}_2^{2p}) \,.
    \end{align}
    By the Burkholder--Davis--Gundy inequality and Jensen's inequality again, we have
    \begin{align}
        \E\norm*{\int_{t'}^t X_3^s \de B^s}_2^{2p} & \leq C_p \E \ab(\int_{t'}^t \norm{X_3^s}_2^{2} \de s)^p \leq C_p (t - t')^{p-1} \int_{t'}^t \E \norm{X_3^s}_2^{2p} \de s \,.
    \end{align}
    This proves the claim. Note that we can set $C_1=1$ by the It\^o isometry.
\end{proof}

\subsubsection{Construction of the Admissible Spaces}

Take constants $\Phi_{C_\theta} \geq 7(M+4T^2M^2)$ and $\Phi_{R_\theta} \geq 2$. Define constants $\Phi_{C_\ell},\Phi_{\Sigma_\ell},\Phi_{R_\ell},\Phi_i \, (i=1,\dots,5), M_\ell,M_\theta$, and $\bar\lambda$ as follows:
\begin{align}
    \Phi_{C_\ell}      & \coloneqq (2M + 6M^2 m \Phi_{C_\theta}) \exp\ab(6M^2\ab(\frac{T}{\delta^2} + \frac{\tau}{\delta}) \Phi_{R_\theta}^2 T) \,, \label{eq:Phi_Cl}                \\
    \Phi_{\Sigma_\ell} & \coloneqq (T + \tau\delta)\Phi_{C_\ell} \,, \label{eq:Phi_Sigmal}                                                                                           \\
    \Phi_{R_\ell}      & \coloneqq \sqrt{2\Phi_1 + 2\tau\delta\Phi_5} \,, \label{eq:Phi_Rl} \\
    \Phi_1      & \coloneqq \frac{3M^4}{\delta^2} \Phi_{R_\theta}^2 \exp\ab(3M^2\ab(\frac{T}{\delta^2} + \frac{\tau}{\delta}) \Phi_{R_\theta}^2 T) \,, \\
    \Phi_2 & \coloneqq \frac{27M^4  \Phi_{R_\theta}^4}{\delta^4} \exp\ab(27 M^4 \ab(\frac{T^3}{\delta^4} + \frac{C_2\tau^2 T}{\delta^2}) \Phi_{R_\theta}^4 T) \,, \\
    \Phi_3 & \coloneqq \ab(8M + 648 M^4 m^2 \Phi_{C_\theta}^2) \exp\ab(216 M^4 \ab(\frac{T^3}{\delta^4} + \frac{C_2\tau^2 T}{\delta^2}) \Phi_{R_\theta}^4 T) \,, \\
    \Phi_4 & \coloneqq \frac{27\tau^2 \Phi_{R_\theta}^4 \Phi_3 \napier^4}{\delta^2} \exp\ab(27 M^4 \ab(\frac{T^3}{\delta^4} + \frac{C_2\tau^2 T}{\delta^2}) \Phi_{R_\theta}^4 T) \,, \\
    \Phi_5 & \coloneqq 2 M^2 \sqrt{\Phi_2 \Phi_4} \napier^{2} \exp\ab(4M^2 \ab(\frac{T}{\delta^2} + \frac{\tau}{\delta}) \Phi_{R_\theta}^2 T) \,, \\
    M_\ell & \coloneqq 2(T+ \tau\delta) \Phi_{C_\ell} \napier^{2\bar\lambda T} \,, \\
    M_\theta & \coloneqq 2\ab(\frac{m M_\ell}{\delta} + 2 T \ab(2M^2 + 3M^2 \Phi_{C_\theta} \napier^{2\bar\lambda T} + T^2\Phi_{R_\ell}^2 \Phi_{C_\theta} \napier^{4\bar\lambda T})) \,, \\
    \bar \lambda & \coloneqq \max\ab\{2T (3M^2 + T \Phi_{R_\ell}^2),\frac{6m\Phi_{\Sigma_\ell}}{\delta M},2(2M+\Phi_{R_\ell} T)\} \,,
\end{align}
where $C_2 > 0$ is the constant in \cref{lem:bound_int} for $p=2$.

For $T > 0$, we define the function spaces $\calS_\theta \coloneqq \calS_\theta(T)$ and $\calS_\ell \coloneqq \calS_\ell(T)$ as follows.
We define $\calS_\theta$ as the space of pairs of functions $(C_\theta,R_\theta)$ satisfying the continuity conditions and initial conditions in \cref{def:admissible_theta} and the following bounds for $0 \leq t' \leq t \leq T$:
\begin{align}
    \norm{C_\theta(t,t)}_2 \leq \Phi_{C_\theta} \napier^{2\bar\lambda t} \,, \quad \norm{R_\theta(t,t')}_2 \leq \Phi_{R_\theta} \napier^{\bar\lambda (t - t')} \,.
\end{align}
Then, $\calS_\theta$ is admissible with parameters $\Phi_\theta = \max\{\Phi_{C_\theta} \napier^{2\bar\lambda T},\Phi_{R_\theta} \napier^{\bar \lambda T}\}$ and $M_\theta$.

We define $\calS_\ell$ as the space of triples of functions $(\Sigma_\ell,R_\ell,\Gamma)$ satisfying the continuity conditions and initial conditions in \cref{def:admissible_ell} and the following bounds for $0 \leq t' \leq t \leq T$:
\begin{align}
    \norm{\Sigma_\ell(t,t)}_2 \leq \frac{\Phi_{\Sigma_\ell} \napier^{2\bar\lambda t}}{\bar\lambda} \,, \quad \norm{R_\ell(t,t')}_2 \leq \Phi_{R_\ell} \napier^{\bar\lambda (t - t')} \,, \quad \norm{\Gamma(t)}_2 \leq M \,.
\end{align}
Then, $\calS_\ell$ is admissible with parameters $\Phi_\ell = \max\{\Phi_{\Sigma_\ell} \napier^{2\bar\lambda T}/\bar\lambda,\Phi_{R_\ell} \napier^{\bar \lambda T}\}$ and $M_\ell$.

In the following, we show that
\begin{enumerate}
    \item For sufficiently small $T$ with $\bar\lambda T \leq 1$, the mappings $\calT_{\theta \to \ell}$ and $\calT_{\ell \to \theta}$ map $\calS_\theta$ and $\calS_\ell$ into each other, respectively.
    \item If either $\tau = 0$ or $\nabla_r^2 \ell_t(r;z) = 0$, the above holds for any $T > 0$.
\end{enumerate}
Note that it is possible to take $\bar \lambda T \leq 1$ since $\bar \lambda$ is monotonically increasing in $T$. Then, we can take $T_*$ as the supremum of such $T$. We provide a rough estimate of $T_* \gtrsim \delta^2/(\tau^2 M^7 m^2)$ where $\gtrsim$ hides subleading terms in $M,m,\tau,1/\delta$ and the constant factor in \cref{app:estimate_T}.

The proof is almost identical for both cases; the only difference lies in bounding $R_\ell$.

\subsubsection{$\calT_{\theta \to \ell}$ maps $\calS_\theta$ into $\calS_\ell$.}

\paragraph{Condition for $\Sigma_\ell$.}
We have by the assumptions and the triangle inequality that
\begin{align}
    \E \norm{\ell_t(r^t;z)}_2^2 \leq \E \ab(\norm{\ell_t(0;z)}_2 + M \norm{r^t}_2)^2 \leq 2 \E \norm{\ell_t(0;z)}_2^2 + 2 M^2 \E\norm{r^t}_2^2 \leq 2M + 2M^2 \E \norm{r^t}_2^2 \,.
\end{align}
We apply \cref{lem:bound_int} to $r^t$ in \cref{eq:r} with
\begin{align}
    \E\norm{X_1}_2^2   & \coloneqq \E \norm{w^t}_2^2 = \tr(C_\theta(t,t)) \leq m \norm{C_\theta(t,t)}_2 \leq m \Phi_{C_\theta}\napier^{2\bar \lambda t} \,,                                                      \\
    \E\norm{X_2^s}_2^2 & \coloneqq \E \norm*{\frac{1}{\delta} R_\theta(t,s) \ell_s(r^s;z)}_2^2 \leq \frac{1}{\delta^2} \Phi_{R_\theta}^2 \napier^{2\bar \lambda (t-s)} \E\norm{\ell_s(r^s;z)}_2^2 \,,            \\
    \E\norm{X_3^s}_2^2 & \coloneqq \E \norm*{\sqrt{\frac{\tau}{\delta}} R_\theta(t,s) \ell_s(r^s;z)}_2^2 \leq \frac{\tau}{\delta} \Phi_{R_\theta}^2 \napier^{2\bar \lambda (t-s)} \E\norm{\ell_s(r^s;z)}_2^2 \,,
\end{align}
to obtain
\begin{align}
    \E\norm{r^t}_2^2 & \leq 3 \ab(m \Phi_{C_\theta}\napier^{2\bar \lambda t} + \ab(\frac{T}{\delta^2} + \frac{\tau}{\delta}) \Phi_{R_\theta}^2 \int_0^t \napier^{2\bar \lambda (t-s)} \E\norm{\ell_s(r^s;z)}_2^2 \de s ) \,.
\end{align}
Thus, we have
\begin{align}
    \napier^{-2\bar\lambda t}\E \norm{\ell_t(r^t;z)}_2^2 \leq 2M + 6M^2 m \Phi_{C_\theta} + 6M^2\ab(\frac{T}{\delta^2} + \frac{\tau}{\delta}) \Phi_{R_\theta}^2 \int_0^t \napier^{-2\bar \lambda s} \E\norm{\ell_s(r^s;z)}_2^2 \de s \,.
\end{align}
By Gr\"onwall's inequality, we have
\begin{align}
    \E\norm{\ell_t(r^t;z)}_2^2 \leq (2M + 6M^2 m \Phi_{C_\theta}) \exp\ab(6M^2\ab(\frac{T}{\delta^2} + \frac{\tau}{\delta}) \Phi_{R_\theta}^2 T) \napier^{2\bar \lambda t} = \Phi_{C_\ell} \napier^{2\bar \lambda t} \,, \label{eq:ell_bound}
\end{align}
where we used the definition of $\Phi_{C_\ell}$ in \cref{eq:Phi_Cl}.

Now, we check the condition $\norm{\Sigma_\ell(t,t)}_2 \leq \Phi_{\Sigma_\ell}\napier^{2\bar\lambda t}/\bar\lambda$.
Using \cref{lem:bound_int}, we have
\begin{align}
    \norm{\Sigma_\ell(t,t)}_2 & = \norm{\E[L^t L^{t\transpose}]}_2 \leq \E\norm{L^t}_2^2 \leq 2 (T + \tau\delta) \int_0^t \E \norm{\ell_s(r^s;z)}_2^2 \de s \notag                                                                                                                                    \\
                              & \leq 2 (T + \tau \delta) \int_0^t \Phi_{C_\ell}\napier^{2\bar \lambda s} \de s = \frac{(T + \tau\delta)\Phi_{C_\ell}}{\bar \lambda} (\napier^{2\bar \lambda t}-1) \leq \frac{\Phi_{\Sigma_\ell}}{\bar\lambda} \napier^{2\bar\lambda t} \,, \label{eq:Sigma_ell_bound}
\end{align}
where we used the definition of $\Phi_{\Sigma_\ell}$ in \cref{eq:Phi_Sigmal}.

Next, we check the condition \eqref{eq:bound_Sigma_l}. We have
\begin{align}
     & \norm{\Sigma_\ell(t,t) - 2\Sigma_\ell(t,t') + \Sigma_\ell(t',t')}_2 = \norm{\E[(L^t - L^{t'}) (L^t - L^{t'})^\transpose]}_2 \leq \E \norm{L^t - L^{t'}}_2^2 \notag                                                                     \\
     & = \E\norm*{\int_{t'}^t \ell_s(r^s;z) (\de s + \sqrt{\tau\delta} \de B^s)}_2^2 \leq 2 (T + \tau \delta) \int_{t'}^t \E \norm{\ell_s(r^s;z)}_2^2 \de s \notag \\
     & \leq 2 (T + \tau \delta) \Phi_{C_\ell} \napier^{2\bar \lambda T} \abs{t - t'} \leq M_\ell \abs{t-t'} \,.
\end{align}

\paragraph{Condition for $R_\ell$: case (1).}
We have
\begin{align}
    \norm{R_\ell(t,t')}_2^2 & \leq 2 \E\norm{\rho_\ell^{t,t'}}_2^2 + 2 \tau\delta \E\norm{D_\ell^{t,t'}}_2^2 \,.
\end{align}

First, we bound $\E\norm{\rho_\ell^{t,t'}}_2^2$.
By the Lipschitz continuity of $\ell_t(r;z)$ in $r$, we have $\norm{\nabla_r \ell_t(r^t;z)}_2 \leq M$ for all $t$. Therefore, we have
\begin{align}
    \E\norm{\rho_\ell^{t,t'}}_2^2 & = \E\norm{\nabla_r \ell_t(r^t;z) \rho_r^{t,t'}}_2^2 \leq M^2 \E\norm{\rho_r^{t,t'}}_2^2 \,.
\end{align}
Applying \cref{lem:bound_int} to $\rho_r^{t,t'}$ in \cref{eq:rho_ell}, we have
\begin{align}
    \E\norm{\rho_r^{t,t'}}_2^2 & \leq 3 \ab(\frac{M^2}{\delta^2} \Phi_{R_\theta}^2 \napier^{2\bar\lambda(t-t')} + \ab(\frac{T}{\delta^2} + \frac{\tau}{\delta}) \Phi_{R_\theta}^2 \int_{t'}^t \napier^{2\bar\lambda(t-s)} \E \norm{\rho_\ell^{s,t'}}_2^2 \de s ) \,.
\end{align}
Thus, we have
\begin{align}
    \napier^{-2\bar\lambda(t-t')}\E\norm{\rho_\ell^{t,t'}}_2^2 & \leq \frac{3 M^4}{\delta^2} \Phi_{R_\theta}^2 + 3 M^2 \ab(\frac{T}{\delta^2} + \frac{\tau}{\delta}) \Phi_{R_\theta}^2 \int_{t'}^t \napier^{-2\bar\lambda(s-t')} \E \norm{\rho_\ell^{s,t'}}_2^2 \de s \,.
\end{align}
By Gr\"onwall's inequality, we have
\begin{align}
    \napier^{-2\bar\lambda(t-t')}\E\norm{\rho_\ell^{t,t'}}_2^2 \leq \frac{3M^4}{\delta^2} \Phi_{R_\theta}^2 \exp\ab(3M^2\ab(\frac{T}{\delta^2} + \frac{\tau}{\delta}) \Phi_{R_\theta}^2 T) \leq \Phi_1 \,.
\end{align}

Next, we bound $\E\norm{D_\ell^{t,t'}}_2^2$. We have
\begin{align}
    \E\norm{D_\ell^{t,t'}}_2^2 & = 4M^2 \ab(\frac{T}{\delta^2} + \frac{\tau}{\delta}) \int_{t'}^t \norm{R_\theta(t,s)}_2^2 \E \norm{D_\ell^{s,t'}}_2^2 \de s + 2 \E[\norm{\nabla_r^2 \ell_t(r^t;z)}_2^2 \norm{D_r^{t,t'}}_2^2 \norm{\rho_r^{t,t'}}_2^2] \notag \\
    & \leq 4M^2 \ab(\frac{T}{\delta^2} + \frac{\tau}{\delta}) \Phi_{R_\theta}^2 \int_{t'}^t \napier^{2\bar\lambda (t-s)} \E \norm{D_\ell^{s,t'}}_2^2 \de s + 2 M^2 \sqrt{\E\norm{D_r^{t,t'}}_2^4} \sqrt{\E\norm{\rho_r^{t,t'}}_2^4} \,,
\end{align}
where we used that $\norm{\nabla_r^2 \ell_t(r^t;z)}_2 \leq M$ by assumption.
We first bound $\E\norm{\rho_r^{t,t'}}_2^4$. Applying \cref{lem:bound_int}, we obtain
\begin{align}
    \E\norm{\rho_r^{t,t'}}_2^4 & \leq 27 \Bigg(\ab(\frac{T^3}{\delta^4} + \frac{C_2\tau^2 T}{\delta^2}) \int_{t'}^t \norm{R_\theta(t,s)}_2^4 \E[\norm{\nabla_r\ell_s(r^s;z)}_2^4 \norm{\rho_r^{s,t'}}_2^4] \de s \notag \\
    & \qquad + \frac{1}{\delta^4} \norm{R_\theta(t,t')}_2^4 \E \norm{\nabla_r \ell_{t'}(r^{t'};z)}_2^4 \Bigg) \notag \\
    & \leq 27 \ab(\ab(\frac{T^3}{\delta^4} + \frac{C_2\tau^2 T}{\delta^2}) \Phi_{R_\theta}^4 M^4 \int_{t'}^t \napier^{4\bar\lambda(t-s)} \E\norm{\rho_r^{s,t'}}_2^4 \de s + \frac{M^4}{\delta^4} \Phi_{R_\theta}^4 \napier^{4\bar\lambda (t-t')} ) \,.
\end{align}
By Gr\"onwall's inequality, we have
\begin{align}
    \E\norm{\rho_r^{t,t'}}_2^4 & \leq \frac{27M^4  \Phi_{R_\theta}^4}{\delta^4} \exp\ab(27 M^4 \ab(\frac{T^3}{\delta^4} + \frac{C_2\tau^2 T}{\delta^2}) \Phi_{R_\theta}^4 T) \napier^{4\bar\lambda (t-t')} \leq \Phi_2 \napier^{4\bar\lambda(t-t')} \,.
\end{align}
We next bound $\E\norm{D_r^{t,t'}}_2^4$. Applying \cref{lem:bound_int}, we have
\begin{align}
    \E\norm{D_r^{t,t'}}_2^4 & \leq 27 \Bigg(\ab(\frac{T^3}{\delta^4} + \frac{C_2\tau^2 T}{\delta^2}) \int_{t'}^t \norm{R_\theta(t,s)}_2^4 \E[\norm{\nabla_r\ell_s(r^s;z)}_2^4 \norm{D_r^{s,t'}}_2^4] \de s \notag \\
    & \qquad + \frac{\tau^2}{\delta^2} \norm{R_\theta(t,t')}_2^4 \E \norm{\ell_{t'}(r^{t'};z)}_2^4 \Bigg) \notag \\
    & \leq 27 \ab(\ab(\frac{T^3}{\delta^4} + \frac{C_2\tau^2 T}{\delta^2}) \Phi_{R_\theta}^4 M^4 \int_{t'}^t \napier^{4\bar\lambda(t-s)} \E\norm{D_r^{s,t'}}_2^4 \de s + \frac{\tau^2}{\delta^2} \Phi_{R_\theta}^4 \napier^{4\bar\lambda (t-t')} \E\norm{\ell_{t'}(r^{t'};z)}_2^4 ) \,.
\end{align}
$\E\norm{\ell_{t'}(r^{t'};z)}_2^4$ can be bounded as
\begin{align}
    & \E\norm{\ell_{t}(r^{t};z)}_2^4 \leq 8 \E \norm{\ell_{t}(0;z)}_2^4 + 8 M^4 \E\norm{r^{t}}_2^4 \notag \\
    & \leq 8M + 8M^4 \cdot 27 \ab(\E\norm{w^t}_2^4 + \ab(\frac{T^3}{\delta^4} + \frac{C_2\tau^2 T}{\delta^2}) \int_0^t \norm{R_\theta(t,s)}_2^4 \E\norm{\ell_s(r^s;z)}_2^4 \de s ) \notag \\
    & \leq 8M + 216M^4 \ab(3m^2 \Phi_{C_\theta}^2 \napier^{4\bar\lambda t} + \ab(\frac{T^3}{\delta^4} + \frac{C_2\tau^2 T}{\delta^2}) \Phi_{R_\theta}^4 \int_0^t \napier^{4\bar\lambda(t-s)} \E\norm{\ell_s(r^s;z)}_2^4 \de s ) \,.
\end{align}
By Gr\"onwall's inequality, we have
\begin{align}
    \E\norm{\ell_{t}(r^{t};z)}_2^4 & \leq \ab(8M + 648 M^4 m^2 \Phi_{C_\theta}^2) \exp\ab(216 M^4 \ab(\frac{T^3}{\delta^4} + \frac{C_2\tau^2 T}{\delta^2}) \Phi_{R_\theta}^4 T) \napier^{4\bar\lambda t} \leq \Phi_3 \napier^{4\bar\lambda t} \,.
\end{align}
Taking $T$ small enough such that $\bar\lambda T \leq 1$, we have $\napier^{4\bar\lambda t} \leq \napier^4$.
Thus, we have
\begin{align}
    \E\norm{D_r^{t,t'}}_2^4 & \leq 27 \ab(\ab(\frac{T^3}{\delta^4} + \frac{C_2\tau^2 T}{\delta^2}) \Phi_{R_\theta}^4 M^4 \int_{t'}^t \napier^{4\bar\lambda(t-s)} \E\norm{D_r^{s,t'}}_2^4 \de s + \frac{\tau^2}{\delta^2} \Phi_{R_\theta}^4 \Phi_3 \napier^4 \napier^{4\bar\lambda (t-t')} ) \,.
\end{align}
By Gr\"onwall's inequality, we have
\begin{align}
    \E\norm{D_r^{t,t'}}_2^4 & \leq \frac{27\tau^2 \Phi_{R_\theta}^4 \Phi_3 \napier^4}{\delta^2} \exp\ab(27 M^4 \ab(\frac{T^3}{\delta^4} + \frac{C_2\tau^2 T}{\delta^2}) \Phi_{R_\theta}^4 T) \napier^{4\bar\lambda (t-t')} \leq \Phi_4 \napier^{4} \,.
\end{align}
Combining the above bounds, we have
\begin{align}
    \E\norm{D_\ell^{t,t'}}_2^2 & \leq 4M^2 \ab(\frac{T}{\delta^2} + \frac{\tau}{\delta}) \Phi_{R_\theta}^2 \int_{t'}^t \napier^{2\bar\lambda (t-s)} \E \norm{D_\ell^{s,t'}}_2^2 \de s + 2 M^2 \sqrt{\Phi_4 \napier^{4}} \sqrt{\Phi_2 \napier^{4\bar\lambda(t-t')}} \,.
\end{align}
By Gr\"onwall's inequality, we have
\begin{align}
    \E\norm{D_\ell^{t,t'}}_2^2 & \leq 2 M^2 \sqrt{\Phi_2 \Phi_4} \napier^{2} \exp\ab(4M^2 \ab(\frac{T}{\delta^2} + \frac{\tau}{\delta}) \Phi_{R_\theta}^2 T) \napier^{2\bar\lambda (t-t')} \leq \Phi_5 \napier^{2\bar\lambda(t-t')}\,.
\end{align}
Thus, we have
\begin{align}
    \norm{R_\ell(t,t')}_2^2 & \leq (2 \Phi_1 + 2\tau\delta \Phi_5) \napier^{2\bar\lambda(t-t')} = \Phi_{R_\ell}^2 \napier^{2\bar\lambda(t-t')} \,.
\end{align}

\paragraph{Condition for $R_\ell$: case (2).}
When $\nabla_r^2 \ell_t(r;z) = 0$, we have $D_\ell^{t,t'} = 0$.
Thus, when $\tau = 0$ or $\nabla^2 \ell_t(r;z) = 0$, we have
\begin{align}
    \norm{R_\ell(t,t')}_2^2 \leq \E\norm{\rho_\ell^{t,t'}}_2^2 \leq \Phi_1 \napier^{-2\bar\lambda (t-t')} \leq \Phi_{R_\ell}^2 \napier^{2 \bar \lambda (t-t')} \,.
\end{align}
This holds without taking $T$ small.

\paragraph{Condition for $\Gamma$.}

We have
\begin{align}
    \norm{\Gamma(t)}_2 \leq \E\norm{\nabla_r \ell_t(r^t;z)}_2 \leq M \,.
\end{align}

\subsubsection{$\calT_{\ell \to \theta}$ maps $\calS_\ell$ into $\calS_\theta$.}

\paragraph{Condition for $C_\theta$.}
We have
\begin{align}
    \E\norm{U^t}_2^2 = \frac{1}{\delta} \tr(\Sigma_\ell(t,t)) \leq \frac{m}{\delta} \norm{\Sigma_\ell(t,t)}_2 \leq \frac{m \Phi_{\Sigma_\ell}}{\delta \bar \lambda} \napier^{2\bar \lambda t} \,.
\end{align}
By the Lipschitz continuity of $h$, we have $\E \norm{h_t(\theta^t)}_2^2 \leq 2\norm{h_t(0)}_2^2 + 2\E\norm{h_t(\theta^t) - h_t(0)}_2^2 \leq 2M^2 + 2M^2 \E\norm{\theta^t}_2^2$.
Thus, we have
\begin{align}
     & \napier^{-2\bar \lambda t}\E \norm{\theta^t}_2^2 \notag                                                                                                                                                                                                                             \\
     & \leq 3 \napier^{-2\bar \lambda t} \Bigg(\E\norm{\theta^0}_2^2 + \E \norm{U^t}_2^2 \notag \\
     & \qquad + 2 T \int_0^t \ab( \E \norm{h_s(\theta^s)}_2^2 + \norm{\Gamma(s)}_2^2 \E \norm{\theta^{s}}_2^2 + T \int_0^{s} \norm{R_\ell(s,s')}_2^2 \E \norm{\theta^{s'}}_2^2 \de s') \de s\Bigg) \notag                   \\
     & \leq 3 \Bigg(\napier^{-2\bar\lambda t} (M + 4T^2M^2) + \frac{m \Phi_{\Sigma_\ell}}{\delta \bar\lambda} \notag \\
     & \qquad + 2 T \int_0^t \napier^{-2\bar\lambda (t-s)} \ab( 3M^2 \napier^{-2 \bar\lambda s} \E \norm{\theta^{s}}_2^2 + T \Phi_{R_\ell}^2 \int_0^{s} \napier^{-2\bar\lambda s'} \E \norm{\theta^{s'}}_2^2 \de s') \de s\Bigg) \notag \\
     & \leq 3 \ab(M + 4T^2M^2 + \frac{m \Phi_{\Sigma_\ell}}{\delta \bar\lambda} + 2 T (3M^2 + T^2 \Phi_{R_\ell}^2) \ab(\int_0^t \napier^{-2\bar\lambda (t-s)} \de s) \sup_{s \in [0,T]} \napier^{-2 \bar\lambda s} \E \norm{\theta^{s}}_2^2 ) \notag                                                 \\
     & \leq 3 \ab(M + 4T^2M^2 + \frac{m \Phi_{\Sigma_\ell}}{\delta \bar\lambda} + \frac{T (3M^2 + T^2 \Phi_{R_\ell}^2)}{\bar\lambda} \sup_{s \in [0,T]} \napier^{-2 \bar\lambda s} \E \norm{\theta^{s}}_2^2) \notag                                                                                  \\
\end{align}
By $\bar\lambda \geq 2T (3M^2 + T^2 \Phi_{R_\ell}^2)$, we have
\begin{align}
    \sup_{t \in [0,T]} \napier^{-2\bar\lambda t}\E \norm{\theta^t}_2^2 \leq 6 \ab(M + 4T^2M^2 + \frac{m \Phi_{\Sigma_\ell}}{\delta \bar\lambda}) \,.
\end{align}
By $\bar\lambda \geq 6m\Phi_{\Sigma_\ell}/(\delta (M + 4T^2M^2))$, the right-hand side is bounded by $7 (M + 4T^2M^2) \leq \Phi_{C_\theta}$.
Therefore, we have
\begin{align}
    \norm{C_\theta(t,t)}_2 \leq \E \norm{\theta^t}_2^2 \leq \Phi_{C_\theta} \napier^{2\bar\lambda t} \,.
\end{align}

Next, we check the condition \eqref{eq:bound_C_theta}. We have
\begin{align}
     & \norm{C_\theta(t,t) - 2C_\theta(t,t') + C_\theta(t',t')}_2 = \norm{\E[(\theta^t - \theta^{t'}) (\theta^t - \theta^{t'})^\transpose]}_2 \leq \E \norm{\theta^t - \theta^{t'}}_2^2 \notag   \\
     & \leq \E\norm*{U^t - U^{t'} - \int_{t'}^t \ab(h_s(\theta^s) + \Gamma(s) \theta^s + \int_0^s R_\ell(s,s')\theta^{s'} \de s') \de s}_2^2 \notag                                               \\
     & \leq 2\ab(\E\norm{U^t - U^{t'}}_2^2 + 2 (t - t')^2 \sup_{s \in [0,T]} \ab\{ 2M^2 + 3M^2 \E\norm{\theta^s}_2^2 + T \int_0^s \norm{R_\ell(s,s')}_2^2 \E\norm{\theta^{s'}}_2^2 \de s' \}) \notag    \\
     & \leq 2\ab(\frac{m M_\ell}{\delta} \abs{t - t'} + 2 (t - t')^2 \ab(2M^2 + 3M^2 \Phi_{C_\theta} \napier^{2 \bar\lambda T} + T^2 \Phi_{R_\ell}^2 \Phi_{C_\theta} \napier^{4 \bar\lambda T})) \notag \\
     & \leq M_\theta \abs{t - t'} \,,
\end{align}
where we used
\begin{align}
    \E\norm{U^t - U^{t'}}_2^2 \leq \frac{m}{\delta} \norm{\Sigma_\ell(t,t) - 2\Sigma_\ell(t,t') + \Sigma_\ell(t',t')}_2 \leq \frac{m M_\ell}{\delta} \abs{t - t'} \,.
\end{align}

\paragraph{Condition for $R_\theta$.}
By the Lipschitz continuity of $h$, we have $\norm{\nabla_\theta h_t(\theta)}_2 \leq M$ for all $t$ and $\theta$.
Thus, we have
\begin{align}
    \norm{\rho_\theta^{t,t'}}_2 & \leq 1 + \int_{t'}^t \ab( (\norm{\nabla_\theta h_s(\theta^s)}_2 + \norm{\Gamma(s)}_2) \norm{\rho_\theta^{s,t'}}_2 + \int_{t'}^s \norm{R_\ell(s,s')}_2 \norm{\rho_\theta^{s',t'}}_2 \de s' ) \de s \notag \\
    & \leq 1 + \int_{t'}^t \ab( 2M \norm{\rho_\theta^{s,t'}}_2 + \Phi_{R_\ell} \int_{t'}^s \napier^{\bar\lambda(s-s')} \norm{\rho_\theta^{s',t'}}_2 \de s' ) \de s \,,
\end{align}
and thus
\begin{align}
    & \napier^{-\bar\lambda (t-t')} \norm{\rho_\theta^{t,t'}}_2 \notag \\
    & \leq 1 + \int_{t'}^t \napier^{-\bar\lambda(t-s)} \ab(2M \napier^{-\bar\lambda(s-t')} \norm{\rho_\theta^{s,t'}}_2 + \Phi_{R_\ell} \int_{t'}^s \napier^{\bar\lambda(s'-t')} \norm{\rho_\theta^{s',t'}}_2 \de s') \de s \notag \\
                                                          & \leq 1 + (2M + \Phi_{R_\ell} T) \ab(\int_{t'}^t \napier^{-\bar\lambda(t-s)} \de s) \sup_{s \in [t',t]} \napier^{-\bar\lambda(s-t')} \norm{\rho_\theta^{s,t'}}_2 \notag                                                  \\
                                                          & \leq 1 + \frac{2M + \Phi_{R_\ell} T}{\bar \lambda} \sup_{s \in [t',t]} \napier^{-\bar\lambda(s-t')} \norm{\rho_\theta^{s,t'}}_2 \,.
\end{align}
By $\bar\lambda \geq 2(2M + \Phi_{R_\ell} T)$, we have
\begin{align}
    \norm{\rho_\theta^{t,t'}}_2 \leq 2 \napier^{\bar \lambda (t-t')} \,, \label{eq:bound_rho_theta}
\end{align}
and thus
\begin{align}
    \norm{R_\theta(t,t')}_2 \leq \E \norm{\rho_\theta^{t,t'}}_2 \leq 2 \napier^{\bar \lambda (t-t')} \leq \Phi_{R_\theta} \napier^{\bar\lambda(t-t')} \,.
\end{align}

\subsubsection{A Rough Estimate of $T_*$}
\label{app:estimate_T}

We derive a rough lower bound on $T_*$ up to the leading dependencies on $M,m,\tau,1/\delta$ and ignoring constant factors.
Take $T \leq \min\{1,(M^2(1/\delta^2 + \tau/\delta) \Phi_{R_\theta}^2)^{-1}\}$. Then, the exponents in the definitions of $\Phi_{C_\ell}$ and $\Phi_i\;(i=1,\dots,5)$ are all bounded by constants since
\begin{align}
    M^2\ab(\frac{T}{\delta^2} + \frac{\tau}{\delta}) \Phi_{R_\theta}^2 T \leq 1 \,, \quad M^4 \ab(\frac{T^2}{\delta^4} + \frac{C_2 \tau^2}{\delta^2}) \Phi_{R_\theta}^4 T^2 \lesssim 1 \,.
\end{align}
Then, the $\Phi$ quantities can be bounded as
\begin{equation}
    \begin{gathered}
        \Phi_{C_\theta} \lesssim M \,, \quad \Phi_{R_\theta} \lesssim  1 \,, \quad \Phi_{C_\ell} \lesssim M^3 m \,, \quad \Phi_{\Sigma_\ell} \lesssim \tau M^3 m \,, \quad \Phi_1 \lesssim \frac{M^4}{\delta^2} \,, \quad \Phi_2 \lesssim \frac{M^4}{\delta^4} \,, \\
        \Phi_3 \lesssim M^6 m^2 \,, \quad \Phi_4 \lesssim \frac{\tau^2 M^6 m^2}{\delta^2} \,, \quad \Phi_5 \lesssim \frac{\tau M^7 m}{\delta^3} \,, \quad \Phi_{R_\ell} \leq \frac{\tau M^{7/2} m^{1/2}}{\delta}
    \end{gathered}
\end{equation}
Then, $\bar\lambda$ is bounded as
\begin{align}
    \bar\lambda \lesssim (M^2 + \Phi_{R_\ell}^2) + \frac{m \Phi_{\Sigma_\ell}}{\delta M} + (M + \Phi_{R_\ell}) \lesssim \frac{\tau^2 M^7 m^2}{\delta^2} \,.
\end{align}
Thus, further taking $T \lesssim \delta^2/(\tau^2 M^7 m^2)$, we have $\bar\lambda T \lesssim 1$.

\subsection{Proof of Lemma \ref{lem:modulus_theta_ell}} \label{app:proof_modulus_theta_ell}

Set $\Phi \coloneqq \max\{\Phi_\theta,\Phi_\ell\}$.
In the following, $K$ denotes a positive constant that may depend on $M$, $m$, $T$, $\delta$, $\tau$, and $\Phi$, but not on other variables, and may change from line to line.

\paragraph{Bound of $\dist_\lambda(\Sigma_\ell^1,\Sigma_\ell^2)$.}
Let $w_1 \sim \GP(0,C_\theta^1)$ and $w_2 \sim \GP(0,C_\theta^2)$ be Gaussian processes coupled such that
\begin{align}
    \sup_{t \in [0,T]} \napier^{-\lambda t} \sqrt{\E \norm{w_1^t - w_2^t}_2^2} \leq 2 \cdot \dist_\lambda(C_\theta^1,C_\theta^2) \,.
\end{align}
For $i=1,2$, let $r_i$ be the solution of
\begin{align}
    r_i^t = w_i^t - \frac{1}{\delta} \int_0^t R_\theta^i(t,s) \ell_s(r_i^s;z) (\de s + \sqrt{\tau\delta} \de B^s) \,.
\end{align}
Note that we use the same Brownian motion $B^t$ for $i=1,2$.
Applying \cref{lem:bound_int} to $r_1^t - r_2^t$ with
\begin{align}
     & \E\norm{w_1^t - w_2^t}_2^2 \leq 4\napier^{2\lambda t} \cdot \dist_\lambda(C_\theta^1,C_\theta^2)^2 \,,                                                                     \\
     & \E\norm{R_\theta^1(t,s) \ell_s(r_1^s;z) - R_\theta^2(t,s) \ell_s(r_2^s;z)}_2^2 \notag                                                                                      \\
     & \qquad \leq 2 \ab( \norm{R_\theta^1(t,s) - R_\theta^2(t,s)}_2^2 \E\norm{\ell_s(r_1^s;z)}_2^2 + \norm{R_\theta^2(t,s)}_2^2 \E\norm{\ell_s(r_1^s;z) - \ell_s(r_2^s;z)}_2^2 ) \notag \\
     & \qquad \leq 2 \ab( \Phi \napier^{2\lambda t} \cdot \dist_\lambda(R_\theta^1,R_\theta^2)^2 + \Phi^2 M^2 \E\norm{r_1^s - r_2^s}_2^2 ) \,,
\end{align}
we have
\begin{align}
     \E\norm{r_1^t - r_2^t}_2^2 & \leq 3 \Bigg( 4\napier^{2\lambda t} \cdot \dist_\lambda(C_\theta^1,C_\theta^2)^2 + 2 \ab(\frac{T}{\delta^2} + \frac{\tau}{\delta}) T \Phi \napier^{2\lambda t} \cdot \dist_\lambda(R_\theta^1,R_\theta^2)^2 \notag \\
     & \qquad + 2 \ab(\frac{T}{\delta^2} + \frac{\tau}{\delta}) \Phi^2 M^2 \int_0^t \E\norm{r_1^s - r_2^s}_2^2 \de s\Bigg) \,,
\end{align}
and thus
\begin{align}
    \napier^{-2\lambda t} \E\norm{r_1^t - r_2^t}_2^2 \leq K \cdot \dist_\lambda(X^1,X^2)^2 + K \int_0^t \napier^{-2\lambda s} \E\norm{r_1^s - r_2^s}_2^2 \de s \,.
\end{align}
By Gr\"onwall's inequality, we have
\begin{align}
    \napier^{-2 \lambda t} \E \norm{r_1^t - r_2^t}_2^2 & \leq K \napier^{K T} \cdot \dist_\lambda(X^1,X^2)^2 \leq K \cdot \dist_\lambda(X^1,X^2)^2  \,. \label{eq:dist_r}
\end{align}

For $i=1,2$, let
\begin{align}
    L^t_i \coloneqq \int_0^t \ell_s(r_i^s;z) (\de s + \sqrt{\tau \delta} \de B^s) \,.
\end{align}
Then, we have
\begin{align}
    \napier^{-2\lambda t} \E\norm{L^t_1 - L^t_2}_2^2 & \leq 2 (T + \tau\delta) \int_0^t \napier^{-2 \lambda (t - s)} \cdot \napier^{-2 \lambda s} \E \norm{\ell_s(r_1^s;z) - \ell_s(r_2^s;z)}_2^2 \de s \notag                    \\
                                                     & \leq \frac{(T + \tau \delta) M^2}{\lambda} \sup_{s \in [0,T]} \napier^{-2 \lambda s} \E \norm{r_1^s - r_2^s}_2^2 \leq \frac{K}{\lambda} \cdot \dist_\lambda(X^1,X^2)^2 \,.
\end{align}

Let $\{(U_1^t,U_2^t) \}_{t \in [0,T]}$ be a centered Gaussian process with covariance $\E\ab[\pmat{L^t_1 \\ L^t_2} \pmat{L^t_1 \\ L^t_2}^\transpose] / \delta$.
Since $U_1$ and $U_2$ have covariance kernels $\Sigma_\ell^1/\delta$ and $\Sigma_\ell^2/\delta$ respectively, we have
\begin{align}
    \dist_\lambda(\Sigma_\ell^1,\Sigma_\ell^2) \leq \sup_{t \in [0,T]} \napier^{-\lambda t} \sqrt{\E \norm{U_1^t - U_2^t}_2^2} = \sup_{t \in [0,T]} \napier^{-\lambda t} \sqrt{\E\norm{L^t_1 - L^t_2}_2^2 / \delta} \leq \sqrt{\frac{K}{\lambda}} \cdot \dist_\lambda(X^1,X^2) \,.
\end{align}

\paragraph{Bound of $\dist_\lambda(\Gamma^1,\Gamma^2)$.}

By \cref{eq:dist_r}, we have
\begin{align}
    \dist_\lambda(\Gamma^1,\Gamma^2) & = \sup_{t \in [0,T]} \napier^{-\lambda t} \norm{\Gamma^1(t) - \Gamma^2(t)}_2 \leq \sup_{t \in [0,T]} \napier^{-\lambda t} \sqrt{\E \norm{\nabla_r \ell_t(r_1^t;z) - \nabla_r \ell_t(r_2^t;z)}_2^2} \notag \\
                                     & \leq M \sup_{t \in [0,T]} \napier^{-\lambda t} \sqrt{\E \norm{r_1^t - r_2^t}_2^2} \leq K \cdot \dist_\lambda(X^1,X^2) \,. \label{eq:dist_Gamma}
\end{align}

\paragraph{Bound of $\dist_\lambda(R_\ell^1,R_\ell^2)$.}

For $i=1,2$, let
\begin{align}
    \rho_{\ell,i}^{t,t'} = \nabla_r \ell_t(r_i^t;z) \rho_{r,i}^{t,t'} \,, \quad \rho_{r,i}^{t,t'} = -\frac{1}{\delta} \int_{t'}^t R_\theta^i(t,s) \rho_{\ell,i}^{s,t'} (\de s + \sqrt{\tau\delta} \de B^s) - \frac{1}{\delta} R_\theta^i(t,t') \nabla_r \ell_{t'}(r_i^{t'};z) \,.
\end{align}
Then, we have
\begin{align}
    \norm{R_\ell^1(t,t') - R_\ell^2(t,t')}_2^2 & \leq 2 \E \norm{\rho_{\ell,1}^{t,t'} - \rho_{\ell,2}^{t,t'}}_2^2 + 2\tau\delta \E\norm{D_{\ell,1}^{t,t'} - D_{\ell,2}^{t,t'}}_2^2 \,.
\end{align}

We first bound $\E\norm{\rho_{\ell,1}^{t,t'} - \rho_{\ell,2}^{t,t'}}_2^2$. We have
\begin{align}
    \E \norm{\rho_{\ell,1}^{t,t'} - \rho_{\ell,2}^{t,t'}}_2^2 & = \E \norm{\nabla_r\ell_t(r_1^t;z) \rho_{r,1}^{t,t'} - \nabla_r\ell_t(r_2^t;z) \rho_{r,2}^{t,t'}}_2^2 \notag            \\
                                               & \leq 2 \E \norm{(\nabla_r \ell_t(r_1^t;z) - \nabla_r \ell_t(r_2^t;z)) \rho_{r,1}^{t,t'}}_2^2 + 2 \E \norm{\nabla_r \ell_t(r_2^t;z) (\rho_{r,1}^{t,t'} - \rho_{r,2}^{t,t'})}_2^2 \notag \\
                                               & \leq 2 M^2 \sqrt{\E\norm{\rho_{r,1}^{t,t'}}_2^4} \sqrt{\E\norm{r_1^t - r_2^t}_2^4} + 2 M^2 \E\norm{\rho_{r,1}^{t,t'} - \rho_{r,2}^{t,t'}}_2^2 \,. \label{eq:dist_ell}
\end{align}
First, we bound $\E \norm{\rho_{r,1}^{t,t'}}_2^4$. Applying \cref{lem:bound_int} to $\rho_{r,1}^{t,t'}$, we have
\begin{align}
    \E \norm{\rho_{r,1}^{t,t'}}_2^4 & \leq 27 \ab(\frac{\Phi^4 M^4}{\delta^4} + \ab(\frac{T^3}{\delta^4} + \frac{3 T \tau^2}{\delta^2}) \Phi^4 M^4 \int_{t'}^t \E \norm{\rho_{r,1}^{s,t'}}_2^4 \de s) \leq K + K \int_{t'}^t \E\norm{\rho_{r,1}^{s,t'}}_2^4 \,,
\end{align}
By Gr\"onwall's inequality, we have
\begin{align}
    \E \norm{\rho_{r,1}^{t,t'}}_2^4 \leq K \napier^{KT} \leq K \,. \label{eq:bound_rhor4}
\end{align}
Next, we bound $\E \norm{r_1^t - r_2^t}_2^4$ by applying \cref{lem:bound_int} with
\begin{align}
     & \E\norm{w_1^t - w_2^t}_2^4 \leq 3 (\E\norm{w_1^t - w_2^t}_2^2)^2 \leq 12 \napier^{4\lambda t} \cdot \dist_\lambda(C_\theta^1,C_\theta^2)^4 \,,                             \\
     & \E\norm{R_\theta^1(t,s) \ell_s(r_1^s;z) - R_\theta^2(t,s) \ell_s(r_2^s;z)}_2^4 \notag                                                                                      \\
     & \leq 8 \ab( \norm{R_\theta^1(t,s) - R_\theta^2(t,s)}_2^4 \E\norm{\ell_s(r_1^s;z)}_2^4 + \norm{R_\theta^2(t,s)}_2^4 \E\norm{\ell_s(r_1^s;z) - \ell_s(r_2^s;z)}_2^4 ) \notag \\
     & \leq 8 \ab( K \napier^{4\lambda t} \cdot \dist_\lambda(R_\theta^1,R_\theta^2)^4 + \Phi^4 M^4 \E\norm{r_1^s - r_2^s}_2^4 ) \,,
\end{align}
Here, we used that $\E\norm{\ell_t(r^t;z)}_2^4$ is uniformly bounded by some constant $K > 0$ by a similar argument as the bound on $\E\norm{\rho_r^{t,t'}}_2^4$.
Then, we have
\begin{align}
    \E \norm{r_1^t - r_2^t}_2^4 \leq K\napier^{4\lambda t} \cdot \dist_\lambda(X^1,X^2)^4 + K \int_0^t \E \norm{r_1^s - r_2^s}_2^4 \de s \,.
\end{align}
By Gr\"onwall's inequality, we have
\begin{align}
    \napier^{-4 \lambda t} \E \norm{r_1^t - r_2^t}_2^4 \leq K \napier^{KT} \cdot \dist_\lambda(X^1,X^2)^4 \leq K \cdot \dist_\lambda(X^1,X^2)^4 \,. \label{eq:dist_r4}
\end{align}
Finally, we bound $\E\norm{\rho_{r,1}^{t,t'} - \rho_{r,2}^{t,t'}}_2^2$. We apply \cref{lem:bound_int} to $\rho_{r,1}^{t,t'} - \rho_{r,2}^{t,t'}$ with
\begin{align}
     & \E\norm{R_\theta^1(t,t') \nabla_r \ell_{t'}(r_1^{t'};z) - R_\theta^2(t,t') \nabla_r \ell_{t'}(r_2^{t'};z)}_2^2 \notag                                                                                                      \\
     & \leq 2 \ab( \norm{R_\theta^1(t,t') - R_\theta^2(t,t')}_2^2 \E\norm{\nabla_r \ell_{t'}(r_1^{t'};z)}_2^2 + \norm{R_\theta^2(t,t')}_2^2 \E\norm{\nabla_r \ell_{t'}(r_1^{t'};z) - \nabla_r \ell_{t'}(r_2^{t'};z)}_2^2 ) \notag \\
     & \leq 2 \ab( M^2 \napier^{2\lambda t} \cdot \dist_\lambda(R_\theta^1,R_\theta^2)^2 + \Phi^2 M^2 \E\norm{r_1^{t'} - r_2^{t'}}_2^2 ) \leq K \napier^{2\lambda t} \cdot \dist_\lambda(X^1,X^2)^2 \,,                           \\
     & \E\norm{R_\theta^1(t,s) \rho_{\ell,1}^{s,t'} - R_\theta^2(t,s) \rho_{\ell,2}^{s,t'}}_2^2 \notag                                                                                                                            \\
     & \leq 2 \ab( \norm{R_\theta^1(t,s) - R_\theta^2(t,s)}_2^2 \E\norm{\rho_{\ell,1}^{s,t'}}_2^2 + \norm{R_\theta^2(t,s)}_2^2 \E\norm{\rho_{\ell,1}^{s,t'} - \rho_{\ell,2}^{s,t'}}_2^2 ) \notag                                  \\
     & \leq 2 \ab( K \napier^{2\lambda t} \cdot \dist_\lambda(R_\theta^1,R_\theta^2)^2 + \Phi^2 \E\norm{\rho_{\ell,1}^{s,t'} - \rho_{\ell,2}^{s,t'}}_2^2 ) \,.
\end{align}
Here, we used that $\E\norm{\rho_{\ell,i}^{s,t'}}_2^2$ is uniformly bounded by some constant $K > 0$ by a similar argument as the bound on $\E\norm{\rho_r^{t,t'}}_2^2$.
Then, we have
\begin{align}
    \E\norm{\rho_{r,1}^{t,t'} - \rho_{r,2}^{t,t'}}_2^2 \leq K \napier^{2\lambda t} \cdot \dist_\lambda(X^1,X^2)^2 + K \int_{t'}^t \E\norm{\rho_{\ell,1}^{s,t'} - \rho_{\ell,2}^{s,t'}}_2^2 \de s \,. \label{eq:dist_rhor}
\end{align}
Plugging \cref{eq:bound_rhor4,eq:dist_r4,eq:dist_rhor} into \cref{eq:dist_ell}, we have
\begin{align}
    \E \norm{\rho_{\ell,1}^{t,t'} - \rho_{\ell,2}^{t,t'}}_2^2 & \leq K\napier^{2\lambda t} \cdot \dist_\lambda(X^1,X^2)^2 + K \int_{t'}^t \E \norm{\rho_{\ell,1}^{s,t'} - \rho_{\ell,2}^{s,t'}}_2^2 \de s \,.
\end{align}
Applying Gr\"onwall's inequality, we have
\begin{align}
    \napier^{-2 \lambda t} \E \norm{\rho_{\ell,1}^{t,t'} - \rho_{\ell,2}^{t,t'}}_2^2 \leq K \cdot \dist_\lambda(X^1,X^2)^2 \,.
\end{align}

Next, we bound $\E\norm{D_{\ell,1}^{t,t'} - D_{\ell,2}^{t,t'}}_2^2$. We have
\begin{align}
    & \norm{D_{\ell,1}^{t,t'} - D_{\ell,2}^{t,t'}}_2 \notag \\
    & \leq \frac{1}{\delta} \norm{\nabla_r \ell_t(r_1^t;z) - \nabla_r \ell_t(r_2^t;z)}_2 \norm*{\int_{t'}^t R_\theta^1(t,s) D_{\ell,1}^{s,t'} (\de s + \sqrt{\tau\delta} \de B^s)}_2 \notag \\
    & \qquad + \frac{1}{\delta} \norm{\nabla_r \ell_t(r_2^t;z)}_2 \norm*{\int_{t'}^t (R_\theta^1(t,s) D_{\ell,1}^{s,t'} - R_\theta^2(t,s) D_{\ell,2}^{s,t'}) (\de s + \sqrt{\tau\delta} \de B^s)}_2 \notag \\
    & \qquad + \norm{\nabla_r^2 \ell_t(r_1^t;z) - \nabla_r^2 \ell_t(r_2^t;z)}_2 \norm{D_{r,1}^{t,t'} \rho_{r,1}^{t,t'}}_2 + \norm{\nabla_r^2 \ell_t(r_2^t;z)}_2 \norm{D_{r,1}^{t,t'} \rho_{r,1}^{t,t'} - D_{r,2}^{t,t'} \rho_{r,2}^{t,t'}}_2 \notag \\
    & \leq \frac{M}{\delta} \norm{r_1^t - r_2^t}_2 \norm*{\int_{t'}^t R_\theta^1(t,s) D_{\ell,1}^{s,t'} (\de s + \sqrt{\tau\delta} \de B^s)}_2 \notag \\
    & \qquad + \frac{M}{\delta} \norm*{\int_{t'}^t (R_\theta^1(t,s) D_{\ell,1}^{s,t'} - R_\theta^2(t,s) D_{\ell,2}^{s,t'}) (\de s + \sqrt{\tau\delta} \de B^s)}_2 \notag \\
    & \qquad + M \norm{r_1^t - r_2^t}_2 \norm{D_{r,1}^{t,t'}}_2 \norm{\rho_{r,1}^{t,t'}}_2 + M \norm{D_{r,1}^{t,t'} - D_{r,2}^{t,t'}}_2 \norm{\rho_{r,1}^{t,t'}}_2 + M \norm{D_{r,2}^{t,t'}}_2 \norm{\rho_{r,1}^{t,t'} - \rho_{r,2}^{t,t'}}_2 \,.
\end{align}
Squaring and taking the expectation, we have
\begin{align}
    & \E \norm{D_{\ell,1}^{t,t'} - D_{\ell,2}^{t,t'}}_2^2 \notag \\
    & \leq K \E \norm{r_1^t - r_2^t}_2^2 \norm*{\int_{t'}^t R_\theta^1(t,s) D_{\ell,1}^{s,t'} (\de s + \sqrt{\tau\delta} \de B^s)}_2^2 \notag \\
    & \qquad + K \E \norm*{\int_{t'}^t (R_\theta^1(t,s) D_{\ell,1}^{s,t'} - R_\theta^2(t,s) D_{\ell,2}^{s,t'}) (\de s + \sqrt{\tau\delta} \de B^s)}_2^2 \notag \\
    & \qquad + K \E \norm{r_1^t - r_2^t}_2^2 \norm{D_{r,1}^{t,t'}}_2^2 \norm{\rho_{r,1}^{t,t'}}_2^2 + K \E \norm{D_{r,1}^{t,t'} - D_{r,2}^{t,t'}}_2^2 \norm{\rho_{r,1}^{t,t'}}_2^2 + K \E \norm{D_{r,2}^{t,t'}}_2^2 \norm{\rho_{r,1}^{t,t'} - \rho_{r,2}^{t,t'}}_2^2 \notag \\
    & \leq K \sqrt{\E \norm{r_1^t - r_2^t}_2^4} \sqrt{\E\norm*{\int_{t'}^t R_\theta^1(t,s) D_{\ell,1}^{s,t'} (\de s + \sqrt{\tau\delta} \de B^s)}_2^4} \notag \\
    & \qquad + K \int_{t'}^t \E \norm{R_\theta^1(t,s) D_{\ell,1}^{s,t'} - R_\theta^2(t,s) D_{\ell,2}^{s,t'}}_2^2 \de s \notag \\
    & \qquad + K \sqrt{\E \norm{r_1^t - r_2^t}_2^4} \sqrt{\E \norm{D_{r,1}^{t,t'}}_2^8} \sqrt{\E\norm{\rho_{r,1}^{t,t'}}_2^8} + K \sqrt{\E \norm{D_{r,1}^{t,t'} - D_{r,2}^{t,t'}}_2^4} \sqrt{\E \norm{\rho_{r,1}^{t,t'}}_2^4} \notag \\
    & \qquad + K \sqrt{\E \norm{D_{r,2}^{t,t'}}_2^4} \sqrt{\E \norm{\rho_{r,1}^{t,t'} - \rho_{r,2}^{t,t'}}_2^4} \notag \\
    & \leq K \napier^{2\lambda t} \cdot \dist_\lambda(X^1,X^2)^2 \cdot \ab(\sqrt{\sup_{s \in [t',t]}\E \norm{D_{\ell,1}^{s,t'}}_2^2} + \sqrt{\E \norm{D_{r,1}^{t,t'}}_2^8} \sqrt{\E\norm{\rho_{r,1}^{t,t'}}_2^8}) \notag \\
    & \qquad + K \int_{t'}^t \ab(\norm{R_\theta^1(t,s) - R_\theta^2(t,s)}_2^2 \E \norm{D_{\ell,1}^{s,t'}}_2^2 + \norm{R_\theta^2(t,s)}_2^2 \E \norm{D_{\ell,1}^{s,t'} - D_{\ell,2}^{s,t'}}_2^2) \de s \notag \\
    & \qquad + K \sqrt{\E \norm{D_{r,1}^{t,t'} - D_{r,2}^{t,t'}}_2^4} \sqrt{\E \norm{\rho_{r,1}^{t,t'}}_2^4} + K \sqrt{\E \norm{D_{r,2}^{t,t'}}_2^4} \sqrt{\E \norm{\rho_{r,1}^{t,t'} - \rho_{r,2}^{t,t'}}_2^4} \,.
\end{align}
We do not repeat all the details, but by a similar argument as the bounds on $\E\norm{\rho_{r,1}^{t,t'}}_2^4$, we can show that $\E\norm{\rho_{r,i}^{t,t'}}_2^p$, $\E \norm{D_{\ell,i}^{t,t'}}_2^p$, $\E \norm{D_{r,i}^{t,t'}}_2^p$, and $\E \norm{D_{\ell,i}^{t,t'}}_2^4$ are uniformly bounded by some constant $K > 0$ for $p \geq 2$ and $i=1,2$.
Moreover, by a similar argument as the bound on $\E\norm{\rho_{r,1}^{t,t'} - \rho_{r,2}^{t,t'}}_2^2$, we can show that $\E\norm{D_{r,1}^{t,t'} - D_{r,2}^{t,t'}}_2^4$ and $\E\norm{\rho_{r,1}^{t,t'} - \rho_{r,2}^{t,t'}}_2^4$ are bounded by $K \napier^{4\lambda t} \cdot \dist_\lambda(X^1,X^2)^4$.
Thus, using these bounds and Gr\"onwall's inequality, we have
\begin{align}
    \napier^{-2\lambda t}\E \norm{D_{\ell,1}^{t,t'} - D_{\ell,2}^{t,t'}}_2^2 \leq K \cdot \dist_\lambda(X^1,X^2)^2 \,.
\end{align}

Therefore, we have
\begin{align}
    \dist_\lambda(R_\ell^1,R_\ell^2) \leq K \cdot \dist_\lambda(X^1,X^2) \,.
\end{align}
Collecting the above bounds, we have
\begin{align}
    \dist_\lambda(Y^1,Y^2) = \sqrt{\lambda} \dist_\lambda(\Sigma_\ell^1,\Sigma_\ell^2) + \dist_\lambda(R_\ell^1,R_\ell^2) + \dist_\lambda(\Gamma^1,\Gamma^2) \leq K \cdot \dist_\lambda(X^1,X^2) \,.
\end{align}

\subsection{Proof of Lemma \ref{lem:modulus_ell_theta}} \label{app:proof_modulus_ell_theta}

Again, $K$ denotes a positive constant that may depend on $M$, $m$, $T$, $\delta$, $\tau$, and $\Phi$, but not on other variables, and may change from line to line.

\paragraph{Bound of $\dist_\lambda(C_\theta^1,C_\theta^2)$.}
Let $U_1 \sim \GP(0,\Sigma_\ell^1/\delta)$ and $U_2 \sim \GP(0,\Sigma_\ell^2/\delta)$ be Gaussian processes coupled such that
\begin{align}
    \sup_{t \in [0,T]} \napier^{-\lambda t} \sqrt{\E \norm{U_1^t - U_2^t}_2^2} \leq 2 \cdot \dist_\lambda(\Sigma_\ell^1,\Sigma_\ell^2) \,.
\end{align}

For $i=1,2$, let $\theta_i$ be the solution of
\begin{align}
    \theta_i^t = \theta^0 + U_i^t - \int_0^t \ab(h_s(\theta_i^s) + \Gamma^i(s) \theta_i^{s} + \int_0^s R_\ell^i(s,s') \theta_i^{s'} \de s') \de s \,.
\end{align}
Then, we have
\begin{align}
     & \napier^{-2\lambda t} \E \norm{\theta_1^t - \theta_2^t}_2^2 \notag                                                                                                                                                                                                                \\
     & \leq 6 \napier^{-2\lambda t} \Bigg(T \int_0^t \E \norm{h_s(\theta_1^s) - h_s(\theta_2^s)}_2^2 \de s + T \int_0^t \norm{\Gamma^1(s)}_2^2 \E \norm{\theta_1^s - \theta_2^s}_2^2 \de s \notag \\
     & \qquad + T^2 \int_0^t \int_0^s \norm{R_\ell^1(s,s')}_2^2 \E \norm{\theta_1^{s'} - \theta_2^{s'}}_2^2 \de s' \de s + \E \norm{U_1^t - U_2^t}_2^2 \notag \\
     & \qquad + T \int_0^t \norm{\Gamma^1(s) - \Gamma^2(s)}_2^2 \E \norm{\theta_2^s}_2^2 \de s + T^2 \int_0^t \int_0^s \norm{R_\ell^1(s,s') - R_\ell^2(s,s')}_2^2 \E \norm{\theta_2^{s'}}_2^2 \de s' \de s \Bigg) \notag                                   \\
     & \leq 6 \Bigg(2 M^2 T \int_0^t \napier^{-2 \lambda (t - s)} \cdot \napier^{-2 \lambda s} \E \norm{\theta_1^s - \theta_2^s}_2^2 \de s \notag \\
     & \qquad + T^2 \Phi^2 \int_0^t \napier^{-2\lambda (t-s)}\int_0^s \napier^{-2\lambda s} \E \norm{\theta_1^{s'} - \theta_2^{s'}}_2^2 \de s' \de s + 4 \cdot \dist_\lambda(\Sigma_\ell^1,\Sigma_\ell^2)^2 \notag \\
     & \qquad + T m \Phi \int_0^t \napier^{-2\lambda (t-s)} \cdot \napier^{-2 \lambda s}\norm{\Gamma^1(s) - \Gamma^2(s)}_2^2 \notag \\
     & \qquad + T^2 m \Phi \int_0^t \napier^{-2\lambda (t-s)}\int_0^s \napier^{-2 \lambda s} \norm{R_\ell^1(s,s') - R_\ell^2(s,s')}_2^2 \de s' \de s \Bigg) \notag                                                                                                                       \\
     & \leq 6 \Bigg(\frac{M^2 T}{\lambda} \sup_{s \in [0,t]} \napier^{-2 \lambda s} \E \norm{\theta_1^s - \theta_2^s}_2^2 + \frac{T^3 \Phi^2}{2\lambda} \sup_{s \in [0,t]} \napier^{-2 \lambda s} \E \norm{\theta_1^s - \theta_2^s}_2^2 \notag                                         \\
     & \qquad + \frac{4}{\lambda} \cdot \lambda \dist_\lambda(\Sigma_\ell^1,\Sigma_\ell^2)^2 + \frac{T m \Phi}{2\lambda} \cdot \dist_\lambda(\Gamma^1,\Gamma^2)^2 + \frac{T^3 m \Phi}{2 \lambda} \cdot \dist_\lambda(R_\ell^1,R_\ell^2)^2 \Bigg) \notag                                  \\
     & \leq \frac{K}{\lambda} \sup_{s \in [0,t]} \napier^{-2 \lambda s} \E \norm{\theta_1^s - \theta_2^s}_2^2 + \frac{K}{\lambda} \cdot \dist_\lambda(Y^1,Y^2)^2 \,.
\end{align}
Since $K$ is independent of $\lambda$, we can take $\lambda > 2 K$ so that
\begin{align}
    \sup_{t \in [0,T]} \napier^{-2 \lambda t} \E \norm{\theta_1^t - \theta_2^t}_2^2 & \leq \frac{2 K}{\lambda} \cdot \dist_\lambda(Y^1,Y^2)^2 \,.
\end{align}
Thus, for any $\eps > 0$, we can take $\lambda$ large enough such that
\begin{align}
    \sup_{t \in [0,T]} \napier^{-2 \lambda t} \E \norm{\theta_1^t - \theta_2^t}_2^2 \leq \eps^2 \cdot \dist_\lambda(Y^1,Y^2)^2 \,. \label{eq:bound_theta_diff}
\end{align}

Let $\{(w_1^t,w_2^t) \}_{t \in [0,T]}$ be a centered Gaussian process with covariance $\E\ab[\pmat{\theta^t_1 \\ \theta^t_2} \pmat{\theta^t_1 \\ \theta^t_2}^\transpose]$.
Since $w_1$ and $w_2$ have covariance kernels $C_\theta^1$ and $C_\theta^2$ respectively, we have
\begin{align}
    \dist_\lambda(C_\theta^1,C_\theta^2) \leq \sup_{t \in [0,T]} \napier^{-\lambda t} \sqrt{\E \norm{w_1^t - w_2^t}_2^2} = \sup_{t \in [0,T]} \napier^{-\lambda t} \sqrt{\E\norm{\theta^t_1 - \theta^t_2}_2^2} \leq \eps \cdot \dist_\lambda(Y^1,Y^2) \,.
\end{align}

\paragraph{Bound of $\dist_\lambda(R_\theta^1,R_\theta^2)$.}
For $i=1,2$, let $\rho_{\theta,i}$ be the solution of
\begin{align}
    \rho_{\theta,i}^{t,t'} = R_\theta^i(t,t') - \int_{t'}^t \ab(\nabla_\theta h_s(\theta_i^s) + \Gamma^i(s)) \rho_{\theta,i}^{s,t'} + \int_{t'}^s R_\ell^i(s,s') \rho_{\theta,i}^{s',t'} \de s') \de s \,.
\end{align}
Thus, we have
\begin{align}
     & \napier^{-\lambda t} \norm{\rho_{\theta,1}^{t,t'} - \rho_{\theta,2}^{t,t'}}_2 \notag                                                                                                                                                                                            \\
     & \leq \napier^{-\lambda t} \int_0^t (\norm{\nabla_\theta h_s(\theta_1^s)}_2 + \norm{\Gamma^1(s)}_2) \norm{\rho_{\theta,1}^{s,t'} - \rho_{\theta,2}^{s,t'}}_2 \de s \notag \\
     & \qquad + \napier^{-\lambda t} \int_0^t \ab(\norm{\nabla_\theta h_s(\theta_1^s) - \nabla_\theta h_s(\theta_2^2)}_2 + \norm{\Gamma^1(s) - \Gamma^2(s)}_2) \norm{\rho_{\theta,2}^{s,t'}} \de s \notag                           \\
     & \qquad + \napier^{-\lambda t} \int_0^t \int_{t'}^s \ab(\norm{R_\ell^1(s,s')}_2 \norm{\rho_{\theta,1}^{s',t'} - \rho_{\theta,2}^{s',t'}}_2 + \norm{R_\ell^1(s,s') - R_\ell^2(s,s')}_2 \norm{\rho_{\theta,2}^{s',t'}}_2) \de s' \de s \notag                           \\
     & \leq \int_0^t \napier^{-\lambda (t-s)}\ab(2M \napier^{-\lambda s}\norm{\rho_{\theta,1}^{s,t'} - \rho_{\theta,2}^{s,t'}}_2 + \Phi \int_{t'}^s \napier^{-\lambda s'}\norm{\rho_{\theta,1}^{s',t'} - \rho_{\theta,2}^{s',t'}}_2 \de s') \de s \notag                                           \\
     & \qquad + \int_0^t \napier^{-\lambda (t-s)}\ab(M \napier^{-\lambda s} \norm{\theta_1^s - \theta_2^2}_2 \norm{\rho_{\theta,2}^{s,t'}}_2 + \napier^{-\lambda s}\norm{\Gamma^1(s) - \Gamma^2(s)}_2 \norm{\rho_{\theta,2}^{s',t'}}_2) \de s \notag \\
     & \qquad + \int_0^t \napier^{-\lambda (t-s)} \int_{t'}^s \napier^{-\lambda s} \norm{R_\ell^1(s,s') - R_\ell^2(s,s')}_2 \norm{\rho_{\theta,2}^{s',t'}}_2  \de s' \de s \,.
\end{align}
By \cref{eq:bound_rho_theta}, we have $\norm{\rho_{\theta,2}^{s',t'}}_2 \leq \Phi$. Thus, we have
\begin{align}
     & \napier^{-\lambda t} \E \norm{\rho_{\theta,1}^{t,t'} - \rho_{\theta,2}^{t,t'}}_2  \notag \\
     & \leq \int_0^t \napier^{-\lambda (t-s)}\ab(2M \napier^{-\lambda s}\norm{\rho_{\theta,1}^{s,t'} - \rho_{\theta,2}^{s,t'}}_2 + \Phi \int_{t'}^s \napier^{-\lambda s'}\norm{\rho_{\theta,1}^{s',t'} - \rho_{\theta,2}^{s',t'}}_2 \de s') \de s \notag                                           \\
     & \qquad + \Phi \int_0^t \napier^{-\lambda (t-s)}\ab(M \napier^{-\lambda s} \E \norm{\theta_1^s - \theta_2^2}_2 + \napier^{-\lambda s}\norm{\Gamma^1(s) - \Gamma^2(s)}_2 ) \de s \notag \\
     & \qquad + \Phi \int_0^t \napier^{-\lambda (t-s)} \int_{t'}^s \napier^{-\lambda s} \norm{R_\ell^1(s,s') - R_\ell^2(s,s')}_2 \de s' \de s \notag \\
     & \leq \ab(\int_0^t \napier^{-\lambda (t-s)} \de s) \Bigg((2M + \Phi T) \sup_{0 \leq t' \leq s \leq t} \napier^{-\lambda s} \norm{\rho_{\theta,1}^{s,t'} - \rho_{\theta,2}^{s,t'}}_2 + \Phi M \sup_{0 \leq s \leq t} \napier^{-\lambda s} \sqrt{\E\norm{\theta_1^s - \theta_2^s}_2} \notag \\
     & \qquad + \Phi \sup_{0 \leq s \leq t} \napier^{-\lambda s} \norm{\Gamma^1(s) - \Gamma^2(s)}_2  + \Phi T \sup_{0 \leq t' \leq s' \leq s \leq t} \napier^{-\lambda s} \norm{R_\ell^1(s,s') - R_\ell^2(s,s')}_2 \Bigg) \notag                                                                                                                                  \\
     & \leq \frac{K}{\lambda} \sup_{0 \leq t' \leq s \leq t} \napier^{-\lambda s} \norm{\rho_{\theta,1}^{s,t'} - \rho_{\theta,2}^{s,t'}}_2 + \frac{K}{\lambda} \cdot \dist_\lambda(Y^1,Y^2) \,.
\end{align}
Taking $\lambda$ large enough such that $\lambda > 2K$, we have
\begin{align}
    \sup_{0 \leq s \leq t \leq T} \napier^{-\lambda t} \norm{R_\theta^1(t,s) - R_\theta^2(t,s)}_2 \leq \sup_{0 \leq s \leq t \leq T} \napier^{-\lambda t} \E\norm{\rho_{\theta,1}^{t,s} - \rho_{\theta,2}^{t,s}}_2 \leq \frac{2K}{\lambda} \cdot \dist_\lambda(Y^1,Y^2) \,.
\end{align}
Thus, for any $\eps > 0$, we can take $\lambda$ large enough such that
\begin{align}
    \dist_\lambda(R_\theta^1,R_\theta^2) \leq \eps \cdot \dist_\lambda(Y^1,Y^2) \,.
\end{align}

Collecting the above bounds, we have that for any $\eps > 0$, we can take $\lambda$ large enough such that
\begin{align}
    \dist_\lambda(X^1,X^2) = \dist_\lambda(C_\theta^1,C_\theta^2) + \dist_\lambda(R_\theta^1,R_\theta^2) \leq \eps \cdot \dist_\lambda(Y^1,Y^2) \,.
\end{align}

\section{Proof of Theorem \ref{thm:dmft_sgf}}
\label{app:proof_dmft_sgf}

We follow the three-step strategy outlined in the main text, which we repeat here for clarity.
\begin{enumerate}
    \item We discretize the SGF \eqref{eq:sgf} with time step $\gamma > 0$ and analyze the discretization error (\cref{lem:sgf_discrete}).
    \item We apply the AMP theory to characterize the asymptotic behavior of the discretized SGF using a low-dimensional state evolution recursion (\cref{lem:dmft_sgf_discrete}).
    \item We take the continuous-time limit $\gamma \to 0$ and show that the state evolution converges to the unique solution of the DMFT equation $\mathfrak{S}$ (\cref{lem:dmft_discrete}).
\end{enumerate}

We consider the discretization of the SGF \eqref{eq:sgf} with time step $\gamma > 0$.
Let $\floor{t} \coloneqq \max\{k\gamma : k\gamma \leq t, k \in \naturals\}$.
We define $\btheta_\gamma^t$ and $\br_\gamma^t$ as the solution of the following equations:
\begin{align}
    \diff{}{t} \btheta_\gamma^{t} = - \ab(h_{\floor{t}}(\btheta_\gamma^{\floor{t}})  + \frac{1}{\delta} \bX^\transpose \ell_{\floor{t}}(\br_\gamma^{\floor{t}};\bz)) \de t + \sqrt{\frac{\tau}{\delta}} \sum_{i=1}^n \bx_i \ell_{\floor{t}}(r_{\gamma,i}^{\floor{t}};z_i)^\transpose \de B_i^t \,, \label{eq:sgf_disc_cont}
\end{align}
and $\br_\gamma^{t} = \bX \btheta_\gamma^{t}$ with initial condition $\btheta_\gamma^0 = \btheta^0$.
At discrete time points $t_1,t_2,\dots$ where we define $t_k = k\gamma$ for $k \in \naturals$, $\btheta_\gamma^{t_k}$ satisfies the following recursion:
\begin{align}
    \btheta_\gamma^{t_{k+1}} & = \btheta_\gamma^{t_k} - \gamma \ab(h_{t_k}(\btheta_\gamma^{t_k})  + \frac{1}{\delta} \bX^\transpose \ell_{t_k}(\br_\gamma^{t_k};\bz)) + \sqrt{\frac{\tau}{\delta}} \sum_{i=1}^n \bx_i \ell_{t_k}(r_{\gamma,i}^{t_k};z_i)^\transpose (B_i^{t_{k+1}} - B_i^{t_k}) \,.  \label{eq:sgf_disc}
\end{align}

First, we control the discretization error between the SGF \eqref{eq:sgf} and its time-discretized version \eqref{eq:sgf_disc}.
We prove it in \cref{app:proof_sgf_discrete}.

\begin{lemma} \label{lem:sgf_discrete}
    Under the assumptions of \cref{thm:dmft_sgf}, for any $T > 0$, there exists a constant $C > 0$ such that we have, almost surely over the randomness of $\bX,\bz,\btheta^0$,
    \begin{align}
        \limsup_{n,d\to\infty} \frac{1}{d} \E_{\bB} \ab[\sup_{0\leq t \leq T} \frobnorm{\btheta^t - \btheta_\gamma^t}^2] \leq C \gamma \,, \quad \limsup_{n,d\to\infty} \frac{1}{n} \E_{\bB} \ab[\sup_{0\leq t \leq T} \frobnorm{\br^t - \br_\gamma^t}^2] \leq C \gamma \,.
    \end{align}
    Furthermore, for any $L \in \naturals$ and $0 \leq t_1,\dots,t_L \leq T$, we have, almost surely,
    \begin{align}
        \lim_{\gamma \to 0} \limsup_{n,d\to\infty} \E_{\bB} W_2\ab(\hat\sfP(\btheta^{t_1},\dots,\btheta^{t_L}),\hat\sfP(\btheta_\gamma^{t_1},\dots,\btheta_\gamma^{t_L}))^2 & = 0 \,, \\
        \lim_{\gamma \to 0} \limsup_{n,d\to\infty} \E_{\bB} W_2\ab(\hat\sfP(\br^{t_1},\dots,\br^{t_L},\bz),\hat\sfP(\br_\gamma^{t_1},\dots,\br_\gamma^{t_L},\bz))^2             & = 0 \,.
    \end{align}
\end{lemma}

Next, we relate the discretized SGF to the discretized DMFT equation $\mathfrak{S}^\gamma$ defined in \cref{app:dmft_disc}.
The following lemma shows that the unique solution of $\mathfrak{S}^\gamma$ characterizes the asymptotic behavior of the discretized SGF \eqref{eq:sgf_disc}.
We prove it in \cref{app:proof_dmft_sgf_discrete}.

\begin{lemma} \label{lem:dmft_sgf_discrete}
    Under the assumptions of \cref{thm:dmft_sgf}, for any $T > 0,L\in\naturals$ and $0\leq t_1<\dots<t_L\leq T$, we have
    \begin{align}
        \plim_{n,d\to\infty} W_2\ab(\hat\sfP(\btheta_{\gamma}^{t_1},\dots,\btheta_{\gamma}^{t_L}),\sfP(\theta_{\gamma}^{t_1},\dots,\theta_{\gamma}^{t_L})) & = 0 \,, \\
        \plim_{n,d\to\infty} W_2\ab(\hat\sfP(\br_{\gamma}^{t_1},\dots,\br_{\gamma}^{t_L},\bz),\sfP(r_{\gamma}^{t_1},\dots,r_{\gamma}^{t_L},z))             & = 0 \,.
    \end{align}
\end{lemma}

Finally, we establish the convergence of the discretized DMFT equation $\mathfrak{S}^\gamma$ to the original DMFT equation $\mathfrak{S}$ as $\gamma \to 0$.
We prove it in \cref{app:proof_dmft_discrete}.

\begin{lemma} \label{lem:dmft_discrete}
    Under the assumptions of \cref{thm:dmft_sgf}, for any $T > 0,L\in\naturals$ and $t_1<\dots<t_L\in[0,T]$, we have
    \begin{align}
        \lim_{\gamma \to 0} W_2\ab(\sfP(\theta_{\gamma}^{t_1},\dots,\theta_{\gamma}^{t_L}),\sfP(\theta^{t_1},\dots,\theta^{t_L})) & = 0 \,, \\
        \lim_{\gamma \to 0} W_2\ab(\sfP(r_{\gamma}^{t_1},\dots,r_{\gamma}^{t_L},z),\sfP(r^{t_1},\dots,r^{t_L},z))                 & = 0 \,.
    \end{align}
\end{lemma}

We are now ready to prove \cref{thm:dmft_sgf}.

\begin{proof}[Proof of \cref{thm:dmft_sgf}]
    We prove for $\btheta^t$; the proof for $\br^t$ is similar.

    For any $0 \leq t_1 < \dots < t_L \leq T$ and $\gamma > 0$, by the triangle inequality, we have
    \begin{align}
        E_{d} \coloneqq W_2\ab(\hat\sfP(\btheta^{t_1},\dots,\btheta^{t_L}),\sfP(\theta^{t_1},\dots,\theta^{t_L})) \leq E_{d,\gamma}^{(1)} + E_{d,\gamma}^{(2)} + E_\gamma^{(3)} \,,
    \end{align}
    where we defined
    \begin{align}
        E_{d,\gamma}^{(1)} & \coloneqq W_2\ab(\hat\sfP(\btheta^{t_1},\dots,\btheta^{t_L}),\hat\sfP(\btheta_{\gamma}^{t_1},\dots,\btheta_{\gamma}^{t_L})) \,, \\
        E_{d,\gamma}^{(2)} & \coloneqq W_2\ab(\hat\sfP(\btheta_{\gamma}^{t_1},\dots,\btheta_{\gamma}^{t_L}),\sfP(\theta_{\gamma}^{t_1},\dots,\theta_{\gamma}^{t_L})) \,,               \\
        E_\gamma^{(3)}     & \coloneqq W_2\ab(\sfP(\theta_{\gamma}^{t_1},\dots,\theta_{\gamma}^{t_L}),\sfP(\theta^{t_1},\dots,\theta^{t_L})) \,.
    \end{align}
    By the union bound, we have, for any $\eps > 0$,
    \begin{align}
        \prob\{E_d \geq \eps\} \leq \prob\{E_{d,\gamma}^{(1)} \geq \eps/3\} + \prob\{E_{d,\gamma}^{(2)} \geq \eps/3\} + \prob\{E_\gamma^{(3)} \geq \eps/3\} \,.
    \end{align}
    Taking the limit $n,d \to \infty$ and applying \cref{lem:dmft_sgf_discrete}, the second term vanishes. Furthermore, by Markov's inequality, we have
    \begin{align}
        \limsup_{n,d \to \infty} \prob\{E_d \geq \eps\} \leq \frac{9}{\eps^2} \limsup_{n,d\to\infty} \E [(E_{d,\gamma}^{(1)})^2] + \prob\{E_\gamma^{(3)} \geq \eps/3\} \,.
    \end{align}
    Since the left-hand side does not depend on $\gamma$, we can take the limit $\gamma \to 0$ and apply \cref{lem:sgf_discrete,lem:dmft_discrete} to obtain
    \begin{align}
        \lim_{n,d \to \infty} \prob\{E_{d} \geq \eps \} = 0 \,,
    \end{align}
    and thus $E_d \to 0$ in probability as $n,d \to \infty$.
\end{proof}

\subsection{Proof of Lemma \ref{lem:sgf_discrete}} \label{app:proof_sgf_discrete}

In the following, $C$ denotes a constant independent of $n,d,\gamma$, which may change from line to line.

We utilize the general results in \cref{app:sde_discretize}.
We check that the SDE \eqref{eq:sgf} satisfies \cref{ass:sde_discretize}.

\paragraph{Drift term.}
For the drift coefficient $\bb(t,\btheta) = h_t(\btheta) + \frac{1}{\delta} \bX^\transpose \ell_{t}(\bX \btheta;\bz)$, we use the Lipschitz continuity of $h$ and $\ell$ and that $\norm{\bX}_2 \leq C$ and $\frobnorm{\ell_t(0;\bz)} \leq C n$ almost surely for sufficiently large $n,d$ by assumption to obtain
\begin{align}
    \frobnorm{\bb(t,\theta)}^2 & \leq 2\frobnorm{h_t(\btheta)}^2 + \frac{2}{\delta^2} \norm{\bX}_2^2 \frobnorm{\ell_t(\bX\btheta;\bz)}^2 \notag \\
    & \leq C (\abs{h_t(0)} d + \frobnorm{\btheta}^2) + C (\frobnorm{\ell_t(0;\bz)}^2 + \frobnorm{\bX\btheta}^2) \notag \\
    & \leq C (d + \frobnorm{\btheta}^2) \,, \\
    \frobnorm{\bb(t_1,\btheta_1)-\bb(t_2,\btheta_2)}^2 & \leq 2\frobnorm{h_{t_1}(\btheta_1)-h_{t_2}(\btheta_2)}^2 + \frac{2}{\delta^2} \norm{\bX}^2 \frobnorm{\ell_{t_1}(\bX \btheta_1; \bz) - \ell_{t_2}(\bX \btheta_2; \bz)}^2 \notag \\
                                                        & \leq C (d \abs{t_1 - t_2}^2 + \frobnorm{\btheta_1-\btheta_2}^2) + C (n \abs{t_1-t_2}^2 + \frobnorm{\bX (\btheta_1 - \btheta_2)}^2) \notag                                        \\
                                                        & \leq C (d\abs{t_1-t_2}^2 + \frobnorm{\btheta_1 - \btheta_2}^2) \,.
\end{align}

\paragraph{Diffusion term.}
For the diffusion coefficient $\bsigma_i(t,\btheta) = \sqrt{\tau/\delta} \bx_i \ell_t(\btheta^\transpose \bx_i;z_i)^\transpose$, we proceed similarly as above to obtain
\begin{align}
    \sum_{i=1}^n \frobnorm{\bsigma_i(t,\btheta)}^2 & = \frac{\tau}{\delta} \sum_{i=1}^n \norm{\bx_i}_2^2 \norm{\ell_t(\btheta^\transpose \bx_i;z_i)}_2^2 \leq C (n + \frobnorm{\bX \btheta}^2) \notag \\
    & \leq C (d + \frobnorm{\btheta}^2) \,, \\
    \sum_{i=1}^n \frobnorm{\bsigma_i(t_1,\btheta_1) - \bsigma_i(t_2,\btheta_2)}^2 & = \frac{\tau}{\delta} \sum_{i=1}^n \norm{\bx_i}_2^2 \norm{\ell_{t_1}(\btheta_1^\transpose \bx_i;z_i) - \ell_{t_2}(\btheta_2^\transpose \bx_i;z_i)}_2^2 \notag \\
                                                                                    & \leq C (d\abs{t_1-t_2}^2 + \frobnorm{\btheta_1 - \btheta_2}^2) \,.
\end{align}

Therefore, by \cref{lem:sde_discretize_error}, we have that, for sufficiently large $d$,
\begin{align}
    \E\ab[\sup_{0\leq t \leq T} \frobnorm{\btheta^t - \btheta_\gamma^{t}}^2] \leq C \gamma (d + \frobnorm{\btheta^0}^2) \,.
\end{align}
By \cref{ass:data}, $\frobnorm{\btheta^0}^2 < Cd$ holds almost surely for sufficiently large $d$.
Therefore, there exists a constant $C$ independent of $n,d,\gamma$ such that
\begin{align}
    \limsup_{n,d\to\infty} \frac{1}{d} \E_{\bB} \ab[\sup_{0\leq t \leq T} \frobnorm{\btheta^t - \btheta_\gamma^{t}}^2] \leq C \gamma \,.
\end{align}
This shows the first claim. Furthermore, we have
\begin{align}
    \E_{\bB} W_2\ab(\hat\sfP(\btheta^{t_1},\dots,\btheta^{t_L}),\hat\sfP(\btheta_{\gamma}^{t_1},\dots,\btheta_{\gamma}^{t_L}))^2 & \leq \E_{\bB} \ab[\frac{1}{d} \sum_{l=1}^L \frobnorm{\btheta^{t_l} - \btheta^{t_l}_\gamma}^2] \notag \\
    & \leq \frac{L}{d} \E_{\bB} \ab[\sup_{0 \leq t \leq T} \frobnorm{\btheta^t - \btheta_\gamma^t}^2] \,.
\end{align}
Taking the limit $n,d \to \infty$ followed by $\gamma \to 0$ shows the second claim.
The claim for $\br^t$ follows from
\begin{align}
    \frobnorm{\br^t - \br_\gamma^t} = \frobnorm{\bX(\btheta^t - \btheta^t_\gamma)} \leq \norm{\bX}_2 \frobnorm{\btheta^t - \btheta^t_\gamma} \,.
\end{align}

\subsection{Proof of Lemma \ref{lem:dmft_sgf_discrete}} \label{app:proof_dmft_sgf_discrete}

\subsubsection{Reduction to AMP}

For notational simplicity, we omit the subscript $\gamma$ and denote $\btheta_\gamma^t,\br_\gamma^t$ by $\btheta^t,\br^t$. As we only work in discrete time, there is no risk of confusion.

Let $\bG^k = (\bB^{t_{k+1}} - \bB^{t_k})/\sqrt{\gamma} \sim \normal(0,\bI_n)$. Then the recursion \eqref{eq:sgf_disc} can be rewritten as
\begin{align}
    \btheta^{t_{k+1}} = \btheta^{t_k} - \gamma h_{t_k}(\btheta^{t_k}) - \frac{\gamma}{\delta} \bX^\transpose ((\bone_n + \sqrt{\tau\delta/\gamma}\bG^k) \odot \ell_{t_k}(\br^{t_k};\bz)) \,, \quad \br^{t_k} = \bX \btheta^{t_k} \,. \label{eq:sgf_disc_rewrite}
\end{align}
Let $M > 0$ be a constant and let $[\cdot]_M \colon x \mapsto \max\{-M,\min\{x,M\}\}$ be the clipping function.
We clip the Gaussian vector in \eqref{eq:sgf_disc_rewrite} entry-wise as
\begin{align}
    \check \btheta^{t_{k+1}} = \check \btheta^{t_k} - \gamma h_{t_k}(\check \btheta^{t_k}) - \frac{\gamma}{\delta} \bX^\transpose ((\bone_n + \sqrt{\tau\delta/\gamma}[\bG^k]_M) \odot \ell_{t_k}(\check \br^{t_k};\bz)) \,, \quad \check \br^{t_k} = \bX \check \btheta^{t_k} \,. \label{eq:sgf_disc_rewrite_clip}
\end{align}

We first control the difference between \eqref{eq:sgf_disc_rewrite} and \eqref{eq:sgf_disc_rewrite_clip}.
\begin{lemma} \label{lem:sgf_discrete_clip}
    We have, almost surely over the randomness of $\bX,\bz,\btheta^0$,
    \begin{align}
        \lim_{M \to \infty} \limsup_{n,d\to\infty} \frac{1}{d} \E_{\bB}\ab[\max_{0 \leq k \leq T/\gamma} \frobnorm{\btheta^{t_k} - \check\btheta^{t_k}}^2] = 0 \,,
    \end{align}
    and thus
    \begin{align}
        \lim_{M \to \infty} \limsup_{n,d\to\infty} \E_{\bB} W_2\ab(\hat\sfP(\btheta^{t_1},\dots,\btheta^{t_L}),\hat\sfP(\check\btheta^{t_1},\dots,\check\btheta^{t_L}))^2 & = 0 \,, \\
        \lim_{M \to \infty} \limsup_{n,d\to\infty} \E_{\bB} W_2\ab(\hat\sfP(\br^{t_1},\dots,\br^{t_L},\bz),\hat\sfP(\check\br^{t_1},\dots,\check\br^{t_L},\bz))^2           & = 0 \,,
    \end{align}
\end{lemma}

\begin{proof}
    We have
    \begin{align}
        & \frobnorm{\btheta^{t_{k+1}} - \check\btheta^{t_{k+1}}} \notag \\
        & \leq \frobnorm{\btheta^{t_k} - \check\btheta^{t_k}} + \gamma \frobnorm{h_{t_k}(\btheta^{t_k}) - h_{t_k}(\check\btheta^{t_k})} \notag \\
        & \qquad + \frac{\gamma}{\delta} \frobnorm*{\bX^\transpose \ab((\bone_n + \sqrt{\tau\delta/\gamma} \bG^k) \odot \ell_{t_k}(\br^{t_k};\bz) - (\bone_n + \sqrt{\tau\delta/\gamma} [\bG^k]_M) \odot \ell_{t_k}(\check\br^{t_k};\bz))} \notag \\
        & \leq C \frobnorm{\btheta^{t_k} - \check\btheta^{t_k}} + C \frobnorm{(\bG^k - [\bG^k]_M) \odot \ell_{t_k}(\br^{t_k};z)} + C \frobnorm{[\bG^k]_M \odot (\ell_{t_k}(\br^{t_k};\bz) - \ell_{t_k}(\check\br^{t_k};\bz))} \,.
    \end{align}
    Since $\bG^k$ is independent of $\br^{t_k}$ and $\check\br^{t_k}$, we have
    \begin{align}
        & \E \frobnorm{(\bG^k - [\bG^k]_M) \odot \ell_{t_k}(\br^{t_k};z)}^2 = \sum_{i=1}^n \E[(G_i^k - [G_i^k]_M)^2 \norm{\ell_{t_k}(r_i^{t_k};z_i)}_2^2] \notag \\
        & = f(M) \sum_{i=1}^n \E[\norm{\ell_{t_k}(\br_i^{t_k};z_i)}_2^2] = f(M) \E\frobnorm{\ell_{t_k}(\br_i^{t_k};\bz)}^2 \leq C f(M) \E\frobnorm{\btheta^{t_k}}^2 \,, \\
        & \E \frobnorm{[\bG^k]_M \odot (\ell_{t_k}(\br^{t_k};\bz) - \ell_{t_k}(\check\br^{t_k};\bz))}^2 = \sum_{i=1}^n \E[[G_i^k]_M^2 \norm{\ell_{t_k}(\br_i^{t_k};z_i) - \ell_{t_k}(\check\br_i^{t_k};z_i)}_2^2] \notag \\
        & = \E[[G]_M^2] \sum_{i=1}^n \E[\norm{\ell_{t_k}(\br_i^{t_k};z_i) - \ell_{t_k}(\check\br_i^{t_k};z_i)}_2^2] \leq C \E[G^2] \E\frobnorm{\br^{t_k} - \check\br^{t_k}}^2 \leq C \E\frobnorm{\btheta^{t_k} - \check\btheta^{t_k}}^2 \,,
    \end{align}
    where $f(M) = \E[(G - [G]_M)^2]$ for $G \sim \normal(0,1)$.
    Therefore, we have
    \begin{align}
        \frac{1}{d} \E\frobnorm{\btheta^{t_{k+1}} - \check\btheta^{t_{k+1}}}^2 \leq \frac{C}{d} \E \frobnorm{\btheta^{t_k} - \check\btheta^{t_k}}^2 + C f(M) \cdot \frac{1}{d} \E \frobnorm{\btheta^{t_k}}^2 \,.
    \end{align}
    Iterating this inequality yields
    \begin{align}
        \frac{1}{d} \E\frobnorm{\btheta^{t_k} - \check\btheta^{t_k}}^2 \leq C f(M) \cdot \frac{1}{d} \sup_{k \leq T/\gamma} \E \frobnorm{\btheta^{t_k}}^2 \,.
    \end{align}
    As $n,d\to\infty$, $\E\frobnorm{\btheta^{t_k}}^2/d$ is uniformly bounded in $t_k$ almost surely for large $d$ by \cref{lem:sde_disc_norm_bound}.
    As $M \to \infty$, we have $f(M) \to 0$ by the dominated convergence theorem, and the first claim follows.
    Convergence of the 2-Wasserstein distances follows from the first claim by the same argument as in the proof of \cref{lem:sgf_discrete}.
\end{proof}

Consider the following AMP iteration. Given sequences of functions $f_i \colon \reals^{(i+1)m+i+2}\to\reals^m$ and $g_i \colon \reals^{im+m}\to\reals^m$ ($i \geq 0$) that are Lipschitz in the first $(i+1)m$ and $im$ arguments, respectively, we generate sequences of matrices $\ba^{i+1}\in\reals^{d \times m}$ and $\bb^i \in\reals^{n \times m}$ ($i\geq 0$) as follows.
\begin{align}
    \ba^{i+1} & = -\frac{1}{\delta} \bX^\transpose f_i(\bb^0,\dots,\bb^i;\bz,[\bG^0]_M,\dots,[\bG^i]_M) + \sum_{j=0}^i g_j(\ba^1,\dots,\ba^j;\btheta^0) \xi_{i,j}^\transpose \,, \\
    \bb^i     & = \bX g_i(\ba^1,\dots,\ba^i;\btheta^0) + \frac{1}{\delta} \sum_{j=0}^{i-1} f_j(\bb^0,\dots,\bb^j;\bz,[\bG^0]_M,\dots,[\bG^j]_M) \zeta_{i,j}^\transpose \,,
\end{align}
with initial values $g_0(\btheta^0)=\btheta^0,\bb^0=\bX\btheta^0$. Here, $f_i$ and $g_i$ are applied row-wise, and $\{\xi_{i,j}\}_{0\leq j\leq i},\allowbreak\{\zeta_{i,j}\}_{0\leq j\leq i-1}\subset\reals^{m \times m}$ are defined as follows.
We define a sequence of centered Gaussian random variables $\{\bar u^{i+1},\bar w^i\}_{i\geq 0}$ recursively as
\begin{align}
    \E[\bar w^i(\bar w^j)^\transpose]         & = \E[g_i(\bar u^1,\dots,\bar u^i;\theta^0)g_j(\bar u^1,\dots,\bar u^j;\theta^0)^\transpose]\,,                                                                     \label{eq:se_w} \\
    \E[\bar u^{i+1}(\bar u^{j+1})^\transpose] & = \frac{1}{\delta} \E[f_i(\bar w^0,\dots,\bar w^i;z,[\bar G^0]_M,\dots,[\bar G^i]_M)f_j(\bar w^0,\dots,\bar w^j;z,[\bar G^0]_M,\dots,[\bar G^j]_M)^\transpose] \,, \label{eq:se_u}
\end{align}
for $0 \leq j \leq i$, and set $\zeta_{i,j},\xi_{i,j}$ as
\begin{align}
    \zeta_{i,j} & = \E\ab[\diffp{}{\bar u^{j+1}} g_i(\bar u^1,\dots,\bar u^i;\theta^0)] \,,                      &  & 0 \leq j \leq i-1 \,, \label{eq:se_zeta} \\
    \xi_{i,j}   & = \E\ab[\diffp{}{\bar w^j} f_i(\bar w^0,\dots,\bar w^i;z,[\bar G^0]_M,\dots,[\bar G^i]_M)] \,, &  & 0 \leq j \leq i \,, \label{eq:se_xi}
\end{align}
where the expectations are taken over $\bar u^i$, $\bar w^i$, $\theta^0 \sim \sfP(\theta^0)$, $z \sim \sfP(z)$, and $\bar G^i \sim \normal(0,1)$.

This AMP iteration can be mapped to the recursion \eqref{eq:sgf_disc_rewrite} by considering the specific choice of $f_i$ and $g_i$ as follows.
\begin{align}
    g_i(\ba^1,\dots,\ba^i;\btheta^0)                     & = \check \btheta^{t_i} \,,                                                           \\
    f_i(\bb^0,\dots,\bb^i;\bz,[\bG^0]_M,\dots,[\bG^i]_M) & = (\bone_n + \sqrt{\tau\delta/\gamma} [\bG^i]_M) \odot \ell_{t_i}(\br^{t_i};\bz) \,.
\end{align}
We show that $g_i$ is indeed a function of $\ba^1,\dots,\ba^i,\btheta^0$ and Lipschitz in $\ba^j$, and that $f_i$ is a function of $\bb^0,\dots,\bb^i,\bz,[\bG^0]_M,\dots,[\bG^i]_M$ and Lipschitz in $\bb^j$.
It can be shown by the Lipschitz continuity of $\ell$, boundedness of $[G]_M$, and induction over $i$ as follows.
\begin{align}
    \check \btheta^{t_{i+1}} & = \check \btheta^{t_i} - \gamma h_{t_i}(\check \btheta^{t_i}) - \frac{\gamma}{\delta} \bX^\transpose f_i(\bb^0,\dots,\bb^i;\bz,[\bG^0]_M,\dots,[\bG^i]_M) \notag                           \\
                             & = \check \btheta^{t_i} - \gamma h_{t_i}(\check \btheta^{t_i}) + \gamma \ab(\ba^{i+1} - \sum_{j=0}^i g_j(\ba^1,\dots,\ba^j;\btheta^0) \xi_{i,j}^\transpose) \,,  \label{eq:recursion_theta} \\
    \br^{t_i}                & = \bX \btheta^{t_i} = \bX g_i(\ba^1,\dots,\ba^i;\btheta^0) \notag                                                                                                                          \\
                             & = \bb^i - \frac{1}{\delta} \sum_{j=0}^{i-1} f_j(\bb^0,\dots,\bb^j;\bz,[\bG^0]_M,\dots,[\bG^j]_M) \zeta_{i,j}^\transpose \,. \label{eq:recursion_r}
\end{align}

By \citet[Theorem 2.21]{wang2024universality}, for any second order pseudo-Lipschitz functions $\psi \colon \reals^{im+m} \to \reals$ and $\tilde\psi \colon \reals^{(i+1)m+i+2} \to \reals$, we have almost surely
\begin{align}
    \lim_{n,d\to\infty} \frac{1}{d}\sum_{j=1}^d \psi(a^1_j,\dots,a^i_j;\theta^0_j)                  & = \E[\psi(\bar u^1,\dots,\bar u^i;\theta^0)]\,,                        \\
    \lim_{n,d\to\infty} \frac{1}{n}\sum_{j=1}^n \tilde\psi(b^0_j,\dots,b^i_j;z_j,G^0_j,\dots,G^i_j) & = \E[\tilde\psi(\bar w^0,\dots,\bar w^i;z,\bar G^0,\dots,\bar G^i)]\,.
\end{align}
Since $\check \btheta^{t_i}$ is a Lipschitz function of $\ba^1,\dots,\ba^i,\btheta^0$, we can take a Lipschitz function $h_\theta$ such that $\check \btheta^{t_i} = h_\theta(\ba^1,\dots,\ba^i;\btheta^0)$ and define $\bar \theta^i \coloneqq h_\theta(\bar u^1,\dots,\bar u^i;\theta^0)$.
Similarly, we can take a Lipschitz function $h_r$ such that $\br^{t_i} = h_r(\bb^0,\dots,\bb^i;\bz,[\bG^0]_M,\dots,[\bG^i]_M)$ and define $\bar r^i \coloneqq h_r(\bar w^0,\dots,\bar w^i;z,\allowbreak [\bar G^0]_M,\dots,[\bar G^i]_M)$.
Considering the composition of $\psi,\tilde\psi$ with $h_\theta,h_r$, we have almost surely
\begin{align}
    \lim_{n,d\to\infty} \frac{1}{d}\sum_{j=1}^d \psi(\check \theta^{t_1}_j,\dots,\check \theta^{t_i}_j;\theta^0_j) & = \E[\psi(\bar \theta^1,\dots,\bar \theta^i;\theta^0)]\,,                       \\
    \lim_{n,d\to\infty} \frac{1}{n}\sum_{j=1}^n \tilde\psi(r^0_j,\dots,r^{t_i}_j;z_j,[G^0_j]_M,\dots,[G^i_j]_M)    & = \E[\tilde\psi(\bar r^0,\dots,\bar r^i;z,[\bar G^0]_M,\dots,[\bar G^i]_M)] \,.
\end{align}
As $M \to \infty$, by the dominated convergence theorem, we have $W_2([G]_M,G) \to 0$ for $G \sim \normal(0,1)$. Combining this with the above and \cref{lem:sgf_discrete_clip}, we have
\begin{align}
    \plim_{n,d\to\infty} W_2\ab(\hat\sfP(\btheta^{t_0},\dots,\btheta^{t_i}),\sfP(\theta^0,\bar \theta^1,\dots,\bar \theta^i)) = 0 \,, \\
    \plim_{n,d\to\infty} W_2\ab(\hat\sfP(\br^{t_0},\dots,\br^{t_i},\bz),\sfP(\bar r^0,\dots,\bar r^i;z,\bar G^0,\dots,\bar G^i)) = 0 \,.
\end{align}

\subsubsection{Mapping the state evolution to DMFT}

It remains to show that the state evolution process $(\bar \theta^i,\bar r^i)_{i \geq 0}$ defined above satisfies the discretized DMFT equations \eqref{eq:dmft_disc_process}--\eqref{eq:dmft_disc_func}.

By \cref{eq:se_w,eq:se_u}, we have
\begin{align}
    \E[\bar w^{i}(\bar w^j)^\transpose] & = \E[\bar \theta^i (\bar \theta^j)^\transpose] \,, \notag \\
    \E[\bar u^{i+1}(\bar u^{j+1})^\transpose] & = \frac{1}{\delta} \E[(1 + \sqrt{\tau\delta/\gamma} \bar G^i) \ell_{t_i}(\bar r^i;z) (1 + \sqrt{\tau\delta/\gamma} \bar G^j) \ell_{t_j}(\bar r^j;z)^\transpose] \,.
\end{align}
Define $\bar U^i$ and $\bar L^i$ as
\begin{align}
    \bar U^i \coloneqq \gamma \sum_{j=1}^{i} \bar u^i \,, \quad \bar L^i \coloneqq \gamma \sum_{j=0}^{i-1} (1 + \sqrt{\tau\delta/\gamma} \bar G^j) \ell_{t_j}(\bar r^j;z) \,.
\end{align}
Then, we have
\begin{align}
    \E[\bar U^i(\bar U^j)^\transpose] & = \frac{1}{\delta} \E[\bar L^i (\bar L^j)^\transpose] \,.
\end{align}

By \cref{eq:recursion_theta}, $\bar \theta^i$ follows the following recursion.
\begin{align}
    \bar \theta^{i+1} & = \bar\theta^i - \gamma h_{t_i}(\bar \theta^i) + \gamma\ab(\bar u^{i+1} - \sum_{j=0}^i \xi_{i,j} \bar \theta^j) \notag \\
    & = \bar\theta^i + \gamma \ab( \bar u^{i+1} - h_{t_i}(\bar\theta^i) - \xi_{i,i}\bar\theta^i - \sum_{j=0}^{i-1} \xi_{i,j} \bar \theta^j ) \,.
\end{align}
Thus,
\begin{align}
    \bar \theta^i = \bar \theta^0 + \sum_{j=0}^{i-1} (\bar \theta^{j+1} - \bar \theta^j) & = \bar \theta^0 + \gamma \sum_{j=0}^{i-1} \bar u^{j+1} - \gamma \sum_{j=0}^{i-1} \ab( h_{t_j}(\bar\theta^j) + \xi_{j,j}\bar\theta^j + \sum_{k=0}^{j-1} \xi_{j,k} \bar \theta^k ) \notag \\
                                                                                         & = \bar \theta^0 + \bar U^i - \gamma \sum_{j=0}^{i-1} \ab( h_{t_j}(\bar\theta^j) + \xi_{j,j}\bar\theta^j + \sum_{k=0}^{j-1} \xi_{j,k} \bar \theta^k ) \,. \label{eq:se_theta}
\end{align}
Furthermore, $\difsp{\bar \theta^i}{\bar u^{j+1}}$ satisfies
\begin{align}
    \diffp{\bar \theta^i}{\bar u^{j+1}} & = \gamma I_m - \gamma \sum_{k=j+1}^{i-1} \ab( (\nabla_\theta h_{t_k}(\bar\theta^k) + \xi_{k,k})\diffp{\bar\theta^k}{\bar u^{j+1}} + \sum_{l=j+1}^{k-1} \xi_{k,l} \diffp{\bar \theta^l}{\bar u^{j+1}} ) \,. \label{eq:se_rhotheta}
\end{align}

By \cref{eq:recursion_r}, $\bar r^i$ follows the following recursions.
\begin{align}
    \bar r^i = \bar w^i - \frac{1}{\delta} \sum_{j=0}^{i-1} \zeta_{i,j} \ell_{t_j}(\bar r^j;z) (1 + \sqrt{\tau\delta/\gamma} \bar G^j) \,. \label{eq:se_r}
\end{align}
Furthermore, $\difsp{\ell_{t_i}(\bar r^i;z)}{\bar w^j}$ satisfies
\begin{align}
    \diffp{\ell_{t_i}(\bar r^i;z)}{\bar w^j} & = \nabla_r \ell_{t_i}(\bar r^i;z) \diffp{\bar r^i}{\bar w^j} \,,                                                                                                                                          \label{eq:dell_dw_disc}               \\
    \diffp{\bar r^i}{\bar w^j}               & = - \frac{1}{\delta} \sum_{k=j+1}^{i-1} \zeta_{i,k} \diffp{\ell_{t_k}(\bar r^k;z)}{\bar w^j} (1 + \sqrt{\tau\delta/\gamma} \bar G^k) - \frac{1}{\delta} \zeta_{i,j} \nabla_r \ell_{t_j}(\bar r^j;z) (1 + \sqrt{\tau\delta/\gamma} \bar G^j) \,.
\end{align}
Let $\{\bar \rho_\ell^{i,j}\}_{0 \leq j < i}$ be the stochastic process satisfying
\begin{align}
    \bar \rho_\ell^{i,j} = \nabla_r \ell_{t_i}(\bar r^i;z) \bar \rho_r^{i,j} \,, \quad \bar \rho_r^{i,j} = - \frac{1}{\delta} \sum_{k=j+1}^{i-1} \zeta_{i,k} \bar \rho_\ell^{k,j} (1 + \sqrt{\tau\delta/\gamma} \bar G^k) - \frac{1}{\delta} \zeta_{i,j} \nabla_r \ell_{t_j}(\bar r^j;z) \,. \label{eq:se_rhol}
\end{align}
Using the linearity of \cref{eq:dell_dw_disc}, we have
\begin{align}
    \diffp{\ell_{t_i}(\bar r^i;z)}{\bar w^j} = \bar \rho_\ell^{i,j} (1 + \sqrt{\tau\delta/\gamma} \bar G^j) \,.
\end{align}
Then, $\xi_{i,j}$ satisfies
\begin{align}
    \xi_{i,j} = \E\ab[\diffp{\ell_{t_i}(\bar r^i;z)}{\bar w^j}] = \E\ab[\bar \rho_\ell^{i,j} (1 + \sqrt{\tau\delta/\gamma} \bar G^j)] = \E[\bar \rho_\ell^{i,j}] + \sqrt{\frac{\tau\delta}{\gamma}} \E\ab[\diffp{\bar \rho_\ell^{i,j}}{\bar G^j}] \,,
\end{align}
where we used Stein's lemma (Gaussian integration by parts) in the last equality.
By \cref{eq:se_rhol}, $\difsp{\bar \rho_\ell^{i,j}}{\bar G^j}$ satisfies
\begin{align}
    \diffp{\bar \rho_\ell^{i,j}}{\bar G^j} & = \nabla_r \ell_{t_i}(\bar r^i;z) \diffp{\bar \rho_r^{i,j}}{\bar G^j} + \diffp{\nabla_r \ell_{t_i}(\bar r^i;z)}{\bar G^j} \bar\rho_r^{i,j} \notag                                                                                                                                \\
                                           & = \nabla_r \ell_{t_i}(\bar r^i;z) \ab(- \frac{1}{\delta} \sum_{k=j+1}^{i-1} \zeta_{i,k} \diffp{\bar \rho_\ell^{k,j}}{\bar G^j} (1 + \sqrt{\tau\delta/\gamma} \bar G^k)) + \nabla_r^2 \ell_{t_i}(\bar r^i;z)\ab[\diffp{\bar r^i}{\bar G^j}] \bar \rho_r^{i,j} \,, \label{eq:se_Drhol}
\end{align}
where $\difsp{\bar r^i}{\bar G^j}$ satisfies
\begin{align}
    \diffp{\bar r^i}{\bar G^j} = - \frac{1}{\delta} \sum_{k=j+1}^{i-1} \zeta_{i,k} \nabla_r \ell_{t_k}(\bar r^k;z) \diffp{\bar r^k}{\bar G^j} (1 + \sqrt{\tau\delta/\gamma} \bar G^k) - \sqrt{\frac{\tau}{\delta\gamma}} \zeta_{i,j} \ell_{t_j}(\bar r^j;z) \,. \label{eq:se_Dr}
\end{align}

These state evolution recursions exactly correspond to the discrete DMFT equations $\mathfrak{S}^\gamma$ shown in \cref{eq:dmft_disc_process,eq:dmft_disc_func} by the following mappings:
\begin{align}
    \theta_\gamma^{t_i}            & \stackrel{\mathrm{d}}{=} \bar \theta^i \,,                                                   &  & i \geq 0 \,, \tag{Compare \eqref{eq:dmft_disc_theta} with \eqref{eq:se_theta}}                                                \\
    \rho_{\theta,\gamma}^{t_i,t_j} & \stackrel{\mathrm{d}}{=} \frac{1}{\gamma} \diffp{\bar\theta^i}{\bar u^{j+1}} \,,             &  & 0 \leq j < i \,, \tag{Compare \eqref{eq:dmft_disc_rhotheta} with \eqref{eq:se_rhotheta}}                                      \\
    r_\gamma^{t_i}                 & \stackrel{\mathrm{d}}{=} \bar r^i \,,                                                        &  & i \geq 0 \,, \tag{Compare \eqref{eq:dmft_disc_r} with \eqref{eq:se_r}}                                                        \\
    \rho_{\ell,\gamma}^{t_i,t_j}   & \stackrel{\mathrm{d}}{=} \frac{1}{\gamma} \bar\rho_{\ell}^{i,j} \,,                          &  & 0 \leq j < i \,, \tag{Compare \eqref{eq:dmft_disc_rhol}, \eqref{eq:dmft_disc_rhor} with \eqref{eq:se_rhol}}                   \\
    D_{\ell,\gamma}^{t_i,t_j}      & \stackrel{\mathrm{d}}{=} \frac{1}{\sqrt{\gamma}} \diffp{\bar\rho_{\ell}^{i,j}}{\bar G^j} \,, &  & 0 \leq j < i \,, \tag{Compare \eqref{eq:dmft_disc_Drhol}, \eqref{eq:dmft_disc_Dr} with \eqref{eq:se_Drhol}, \eqref{eq:se_Dr}} \\
    U_\gamma^{t_i}                 & \stackrel{\mathrm{d}}{=} \bar U^i \,,                                                        &  & i \geq 0 \,,                                                                                                                  \\
    w_\gamma^{t_i}                 & \stackrel{\mathrm{d}}{=} \bar w^i \,,                                                        &  & i \geq 0 \,,                                                                                                                  \\
    B^{t_{i+1}} - B^{t_i}          & \stackrel{\mathrm{d}}{=} \sqrt{\gamma} \bar G^i \,,                                          &  & i \geq 0 \,,
\end{align}
and
\begin{align}
    C_\theta^\gamma(t_i,t_j)    & = \E[\bar \theta^i (\bar \theta^j)^\transpose] \,, &  & i,j \geq 0 \,,   \\
    \Sigma_\ell^\gamma(t_i,t_j) & = \E[\bar L^i (\bar L^j)^\transpose] \,,           &  & i,j \geq 0 \,,   \\
    R_\theta^\gamma(t_i,t_j)    & = \zeta_{i,j}/\gamma \,,                           &  & 0 \leq j < i \,, \\
    R_\ell^\gamma(t_i,t_j)      & = \xi_{i,j}/\gamma \,,                             &  & 0 \leq j < i \,, \\
    \Gamma^\gamma(t_i)          & = \xi_{i,i}\,,                                     &  & 0 \leq i \,.
\end{align}

\subsection{Proof of Lemma \ref{lem:dmft_discrete}} \label{app:proof_dmft_discrete}

We first embed the discretized DMFT equation $\mathfrak{S}^\gamma$ defined for discrete time knots $t_k = k\gamma$ ($k \geq 0$) into continuous time $t \in [0,T]$ in a piecewise constant manner.
We define the stochastic processes $\{\theta_\gamma^t,r_\gamma^t\}_{t \in [0,T]}$ and $\{\rho_{\theta,\gamma}^{t,t'},\rho_{\ell,\gamma}^{t,t'},D_{\ell,\gamma}^{t,t'} \}_{t \geq t' \geq 0}$ by the following equations.
\begin{align}
    \theta_\gamma^t             & = \theta^0 + U_\gamma^t - \int_0^{\floor{t}} \ab(h_{\floor{s}}(\theta_\gamma^s) + \Gamma^\gamma(s)\theta_\gamma^s) + \int_0^{\floor{s}} R_\ell^\gamma(s,s') \theta_\gamma^{s'} \de s') \de s \,,                                                                                                                   \\
    \rho_{\theta,\gamma}^{t,t'} & = I_m - \ind(\ceil{t'} \leq \floor{t}) \int_{\ceil{t'}}^{\floor{t}} \ab((\nabla_\theta h_{\floor{s}}(\theta_\gamma^s) + \Gamma^\gamma(s)) \rho_{\theta,\gamma}^{s,t'} + \int_{\ceil{t'}}^{\floor{s}} R_\ell^\gamma(s,s') \rho_{\theta,\gamma}^{s',t'} \de s') \de s \,,                                                           \\
    r_\gamma^t                  & = w_\gamma^t -\frac{1}{\delta} \int_0^{\floor{t}} R_\theta^\gamma(t,s) \ell_{\floor{s}}(r_\gamma^s;z) (\de s + \sqrt{\tau \delta} \de B^s) \,,                                                                                                                                                           \\
    \rho_{\ell,\gamma}^{t,t'}   & = \nabla_r \ell_{\floor{t}}(r_\gamma^t;z) \rho_{r,\gamma}^{t,t'} \,,                                                                                                                                                                                                                                     \\
    \rho_{r,\gamma}^{t,t'}      & = -\frac{1}{\delta} \ind(\ceil{t'} \leq \floor{t}) \int_{\ceil{t'}}^{\floor{t}} R_\theta^\gamma(t,s) \rho_{\ell,\gamma}^{s,t'} (\de s + \sqrt{\tau\delta} \de B^s) - \frac{1}{\delta} R_\theta^\gamma(t,t') \nabla_r \ell_{\floor{t'}}(r_\gamma^{t'};z) \,,                                              \\
    D_{\ell,\gamma}^{t,t'}      & = \nabla_r \ell_{\floor{t}}(r_\gamma^t;z) \ab( -\frac{1}{\delta} \ind(\ceil{t'} \leq \floor{t}) \int_{\ceil{t'}}^{\floor{t}} R_\theta^\gamma(t,s) D_{\ell,\gamma}^{s,t'} (\de s + \sqrt{\tau\delta} \de B^s)) \notag \\
    & \qquad + \nabla_r^2 \ell_{\floor{t}}(r_\gamma^t;z)[D_{r,\gamma}^{t,t'}] \rho_{r,\gamma}^{t,t'} \,, \\
    D_{r,\gamma}^{t,t'}         & = -\frac{1}{\delta} \ind(\ceil{t'} \leq \floor{t}) \int_{\ceil{t'}}^{\floor{t}} R_\theta^\gamma(t,s) \nabla_r \ell_{\floor{s}}(r_\gamma^s;z) D_{r,\gamma}^{s,t'} (\de s + \sqrt{\tau\delta} \de B^s) \notag \\
    & \qquad - \sqrt{\frac{\tau}{\delta}} R_\theta^\gamma(t,t') \ell_{\floor{t'}}(r_\gamma^{t'};z) \,,
\end{align}
where $U_\gamma^t,w_\gamma^t$ are centered Gaussian processes with covariance kernels $\Sigma_\ell^\gamma / \delta$ and $C_\theta^\gamma$ respectively.
Then, set $C_\theta^\gamma,R_\theta^\gamma,\Sigma_\ell^\gamma,R_\ell^\gamma,\Gamma^\gamma$ as
\begin{align}
    \begin{gathered}
        C_\theta^\gamma(t,t') = \E[\theta_\gamma^t \theta_\gamma^{t'\transpose}] \,, \quad R_\theta^\gamma(t,t') = \E[\rho_{\theta,\gamma}^{t,t'}] \,, \\
        \Sigma_\ell^\gamma(t,t') = \E[L_\gamma^t L_\gamma^{t'\transpose}] \,, \quad L_\gamma^t \coloneqq \int_0^{\floor{t}} \ell_{\floor{s}}(r_\gamma^s;z) (\de s + \sqrt{\tau\delta} \de B^s) \,, \\
        R_\ell^\gamma(t,t') = \E[\rho_{\ell,\gamma}^{t,t'}] + \sqrt{\tau\delta} \E[D_{\ell,\gamma}^{t,t'}] \,, \quad \Gamma^\gamma(t) = \E\ab[\nabla_r\ell_{\floor{t}}(r_\gamma^t;z)] \,,
    \end{gathered}
\end{align}
where we set $R_\theta^\gamma(t,t') = R_\ell^\gamma(t,t') = 0$ for $\floor{t} < \ceil{t'}$.

The above equation agrees with the discretized DMFT equation $\mathfrak{S}^\gamma$ at discrete time points $t = t_i = i\gamma$ ($i \geq 0$), and has a unique solution since it is piecewise constant in each interval $[t_i,t_{i+1})$.

We then define mappings $\calT^\gamma_{\theta \to \ell} \colon (C_\theta,R_\theta) \mapsto (\Sigma_\ell,R_\ell,\Gamma)$ and $\calT^\gamma_{\ell \to \theta} \colon (\Sigma_\ell,R_\ell,\Gamma) \mapsto (C_\theta,R_\theta)$ similarly to $\calT_{\theta \to \ell}$ and $\calT_{\ell \to \theta}$ in \cref{app:proof_dmft_sol} but with the DMFT equation $\mathfrak{S}$ replaced by its discretized version $\mathfrak{S}^\gamma$.
We also define their composition $\calT^\gamma \coloneqq \calT^\gamma_{\ell \to \theta} \circ \calT^\gamma_{\theta \to \ell}$.

Since the solution to the discretized DMFT equation $\mathfrak{S}^\gamma$ is determined by the values at discrete time points $\{t_i = i\gamma : i \geq 0\}$, the solution exists uniquely by induction.
Let $X^\gamma = (C_\theta^\gamma,R_\theta^\gamma)$ and $Y^\gamma = (\Sigma_\ell^\gamma,R_\ell^\gamma,\Gamma^\gamma)$ be the solution to $\mathfrak{S}^\gamma$.
Then, $X^\gamma$ is the unique fixed point of $\calT^\gamma$, i.e., $\calT^\gamma(X^\gamma) = X^\gamma$.

We show that the unique solutions $X^\gamma$ and $Y^\gamma$ belong to some admissible spaces defined in \cref{app:function_spaces}.

\begin{lemma}
    There exist admissible spaces $\calS_\theta$ and $\calS_\ell$ such that $X^\gamma \in \calS_\theta$ and $Y^\gamma \in \calS_\ell$.
\end{lemma}

\begin{proof}
    Since $C_\theta^\gamma,\Sigma_\ell^\gamma,R_\theta^\gamma,R_\ell^\gamma$, and $\Gamma^\gamma$ are piecewise constant, the continuity conditions are automatically satisfied.
    Since the solutions are bounded, we can take the spaces $\calS_\theta$ and $\calS_\ell$ large enough so that the boundedness conditions are also satisfied.
\end{proof}

Let $X = (C_\theta,R_\theta)$ be the unique fixed point of $\calT$ shown in \cref{thm:dmft_sol}.
We control their distance as follows.
\begin{align}
    \dist_\lambda(X,X^\gamma) = \dist_\lambda(\calT(X),\calT^\gamma(X^\gamma)) \leq \dist_\lambda(\calT(X),\calT(X^\gamma)) + \dist_\lambda(\calT(X^\gamma),\calT^\gamma(X^\gamma)) \,.
\end{align}
As in the proof of \cref{thm:dmft_sol}, we can take $\lambda$ large enough so that $\calT$ is a contraction and the first term is bounded by $(1/2)\dist_\lambda(X,X^\gamma)$.
Then, we have
\begin{align}
    \dist_\lambda(X,X^\gamma) \leq 2 \cdot \dist_\lambda(\calT(X^\gamma),\calT^\gamma(X^\gamma)) \,.
\end{align}
The next lemma bounds the right-hand side. We prove it in \cref{app:proof_dmft_disc_dist}.

\begin{lemma} \label{lem:dmft_disc_dist}
    For any $\lambda > 0$, there exists a constant $K > 0$ independent of $\gamma$ such that
    \begin{align}
        \dist_\lambda(\calT(X^\gamma),\calT^\gamma(X^\gamma)) \leq K \sqrt{\gamma} \,.
    \end{align}
\end{lemma}

Therefore, for sufficiently large $\lambda$, we have
\begin{align}
    \dist_\lambda(X,X^\gamma) \leq K \sqrt{\gamma} \,, \quad \dist_\lambda(Y,Y^\gamma) \leq K \sqrt{\gamma} \,.
\end{align}

Finally, we couple $\theta^t$ and $\theta_\gamma^t$ so that they are close. We prove it in \cref{app:proof_dmft_process_coupling}.

\begin{lemma} \label{lem:dmft_process_coupling}
    There exist a constant $K > 0$ independent of $\gamma$ and a coupling of the processes $\theta^t$ and $\theta_\gamma^t$ such that
    \begin{align}
        \sup_{0 \leq t \leq T} \sqrt{\E\norm{\theta^t - \theta_\gamma^t}_2^2} & \leq K (\sqrt{\gamma} + \dist_\lambda(Y,Y^\gamma)) \,, \\
        \sup_{0 \leq t \leq T} \sqrt{\E\norm{r^t - r_\gamma^t}_2^2} & \leq K (\sqrt{\gamma} + \dist_\lambda(X,X^\gamma)) \,.
    \end{align}
\end{lemma}

By the above Lemmas, we have
\begin{align}
    W_2\ab(\sfP(\theta^{t_1},\dots,\theta^{t_L}), \sfP(\theta_\gamma^{t_1},\dots,\theta_\gamma^{t_L})) & \leq \sqrt{\sum_{i=1}^L \E\norm{\theta^{t_i} - \theta_\gamma^{t_i}}_2^2} \leq \sqrt{KL} \cdot \sqrt{\gamma} \,, \\
    W_2\ab(\sfP(r^{t_1},\dots,r^{t_L},z), \sfP(r_\gamma^{t_1},\dots,r_\gamma^{t_L},z))                 & \leq \sqrt{\sum_{i=1}^L \E\norm{r^{t_i} - r_\gamma^{t_i}}_2^2} \leq \sqrt{KL} \cdot \sqrt{\gamma} \,.
\end{align}
Sending $\gamma \to 0$ completes the proof of \cref{lem:dmft_discrete}.

\subsubsection{Proof of Lemma \ref{lem:dmft_disc_dist}} \label{app:proof_dmft_disc_dist}

Let $\bar Y^\gamma = (\bar \Sigma_\ell^\gamma,\bar R_\ell^\gamma,\bar \Gamma^\gamma) = \calT_{\theta \to \ell}(X^\gamma)$ and $Y^\gamma = (\Sigma_\ell^\gamma,R_\ell^\gamma,\Gamma^\gamma) = \calT_{\theta \to \ell}^\gamma(X^\gamma)$.
Also, let $\bar X^\gamma = (\bar C_\theta^\gamma,\bar R_\theta^\gamma) = \calT_{\ell \to \theta}(\bar Y^\gamma)$.

In the following, $K$ denotes a positive constant independent of $\gamma$ whose value may change from line to line.
Note that $K$ can depend on $\lambda$ since we fix $\lambda$ and send $\gamma \to 0$.

The proof proceeds as follows.
\begin{enumerate}
    \item We show that $\dist_\lambda(\bar Y^\gamma,Y^\gamma) \leq K \sqrt{\gamma}$ for some constant $K > 0$.
    \item We show that $\dist_\lambda(\bar X^\gamma,X^\gamma) \leq K (\sqrt{\gamma} + \dist_\lambda(\bar Y^\gamma,Y^\gamma))$ for some constant $K > 0$.
\end{enumerate}
Together they prove the \cref{lem:dmft_disc_dist}.

\paragraph{Bound of $\dist_\lambda(\bar \Sigma_\ell^\gamma,\Sigma_\ell^\gamma)$}
Let $w_\gamma \sim \GP(0,C_\theta^\gamma)$. Let $\bar r_\gamma^t$ and $r_\gamma^t$ be the solutions to the following equations:
\begin{align}
    \bar r_\gamma^t & = w_\gamma^t - \frac{1}{\delta} \int_0^t R_\theta^\gamma(t,s) \ell_s(\bar r_\gamma^s;z) (\de s + \sqrt{\tau\delta} \de B^s) \,,                \\
    r_\gamma^t      & = w_\gamma^t -\frac{1}{\delta} \int_0^{\floor{t}} R_\theta^\gamma(t,s) \ell_{\floor{s}}(r_\gamma^s;z) (\de s + \sqrt{\tau \delta} \de B^s) \,.
\end{align}
We have
\begin{align}
    \bar r_\gamma^t - r_\gamma^t & = -\frac{1}{\delta} \int_0^{\floor{t}} R_\theta^\gamma(t,s) (\ell_s(\bar r_\gamma^s;z) - \ell_{\floor{s}}(r_\gamma^s;z)) (\de s + \sqrt{\tau \delta} \de B^s) \notag \\
                                 & \qquad - \frac{1}{\delta} \int_{\floor{t}}^t R_\theta^\gamma(t,s) \ell_s(\bar r_\gamma^s;z) (\de s + \sqrt{\tau\delta} \de B^s) \,,
\end{align}
and thus
\begin{align}
    \E\norm{\bar r_\gamma^t - r_\gamma^t}_2^2 & \leq 2 \ab(\frac{T}{\delta^2} + \frac{\tau}{\delta}) \int_0^{\floor{t}} \norm{R_\theta^\gamma(t,s)}_2^2 \E\norm{\ell_s(\bar r_\gamma^s;z) - \ell_{\floor{s}}(r_\gamma^s;z)}_2^2 \de s \notag                                                    \\
                                              & \qquad + 2 \ab(\frac{T}{\delta^2} + \frac{\tau}{\delta}) \int_{\floor{t}}^t \norm{R_\theta^\gamma(t,s)}_2^2 \E\norm{\ell_s(\bar r_\gamma^s;z)}_2^2 \de s \notag                                                                                 \\
                                              & \leq 2 \ab(\frac{T}{\delta^2} + \frac{\tau}{\delta}) \Phi^2 M^2 \int_0^{\floor{t}} \E\ab[(\abs{s - \floor{s}} + \norm{\bar r_\gamma^s - r_\gamma^s}_2)^2] \de s + 2 \ab(\frac{T}{\delta^2} + \frac{\tau}{\delta}) \Phi^3 (t - \floor{t}) \notag \\
                                              & \leq 2 \ab(\frac{T}{\delta^2} + \frac{\tau}{\delta}) \Phi^2 M^2 \int_0^{\floor{t}} (2 \gamma^2 + 2 \E\norm{\bar r_\gamma^s - r_\gamma^s}_2^2) \de s + 2 \ab(\frac{T}{\delta^2} + \frac{\tau}{\delta}) \Phi^3 \gamma \notag                      \\
                                              & \leq K \gamma + K \int_0^t \E\norm{\bar r_\gamma^s - r_\gamma^s}_2^2 \de s \,.
\end{align}
By Gr\"onwall's inequality, we have
\begin{align}
    \sup_{t \in [0,T]} \E\norm{\bar r_\gamma^t - r_\gamma^t}_2^2 & \leq K \gamma \napier^{K T} \leq K \gamma \,. \label{eq:dist_r_gamma}
\end{align}

Let
\begin{align}
    \bar L_\gamma^t \coloneqq \int_0^t \ell_s(\bar r_\gamma^s;z) (\de s + \sqrt{\tau \delta} \de B^s) \,, \quad L_\gamma^t \coloneqq \int_0^{\floor{t}} \ell_{\floor{s}}(r_\gamma^s;z) (\de s + \sqrt{\tau\delta} \de B^s) \,.
\end{align}
Then, we have
\begin{align}
    \E\norm{\bar L_\gamma^t - L_\gamma^t}_2^2 & \leq 2 \ab(T + \tau\delta) \int_0^{\floor{t}} \E\norm{\ell_s(\bar r_\gamma^s;z) - \ell_{\floor{s}}(r_\gamma^s;z)}_2^2 \de s + 2 \ab(T + \tau\delta) \int_{\floor{t}}^t \E\norm{\ell_s(\bar r_\gamma^s;z)}_2^2 \de s \notag \\
                                              & \leq 2 \ab(T + \tau\delta) M^2 \int_0^{\floor{t}} \E\ab[(\abs{s - \floor{s}} + \norm{\bar r_\gamma^s - r_\gamma^s}_2)^2] \de s + 2 \ab(T + \tau\delta) \Phi (t - \floor{t}) \notag                                         \\
                                              & \leq K \gamma \,.
\end{align}
Let $\{(\bar U_\gamma^t,U_\gamma^t) \}_{t \in [0,T]}$ be a centered Gaussian process with covariance $\E\ab[\pmat{\bar L_\gamma^t \\ L_\gamma^t} \pmat{L_\gamma^t \\ L_\gamma^t}^\transpose] / \delta$.
Since $\bar U_\gamma$ and $U_\gamma$ have covariance kernels $\bar \Sigma_\ell^\gamma/\delta$ and $\Sigma_\ell^\gamma/\delta$ respectively, we have
\begin{align}
    \dist_\lambda(\bar \Sigma_\ell^\gamma,\Sigma_\ell^\gamma) \leq \sup_{t \in [0,T]} \napier^{-\lambda t} \sqrt{\E \norm{\bar U_\gamma^t - U_\gamma^t}_2^2} = \sup_{t \in [0,T]} \napier^{-\lambda t} \sqrt{\E\norm{\bar L_\gamma^t - L_\gamma^t}_2^2 / \delta} \leq K \sqrt{\gamma} \,. \label{eq:dist_Sigmal_gamma}
\end{align}

\paragraph{Bound of $\dist_\lambda(\bar\Gamma^\gamma,\Gamma^\gamma)$.}

By \cref{eq:dist_r_gamma}, we have
\begin{align}
    \dist_\lambda(\bar\Gamma^\gamma,\Gamma^\gamma) & = \sup_{t \in [0,T]} \napier^{-\lambda t} \norm{\bar\Gamma^\gamma(t) - \Gamma^\gamma(t)}_2 \leq \sup_{t \in [0,T]} \sqrt{\E \norm{\nabla_r \ell_t(\bar r_\gamma^t;z) - \nabla_r \ell_{\floor{t}}(r_\gamma^t;z)}_2^2} \notag \\
                                                   & \leq M \sup_{t \in [0,T]} \sqrt{2(t-\floor{t})^2 + 2\E \norm{\bar r_\gamma^t - r_\gamma^t}_2^2} \leq K \sqrt{\gamma} \,. \label{eq:dist_Gamma_gamma}
\end{align}

\paragraph{Bound of $\dist_\lambda(R_\ell^1,R_\ell^2)$.}

Let $\bar \rho_{r,\gamma}^{t,t'}$ and $\rho_{r,\gamma}^{t,t'}$ be the stochastic processes satisfying
\begin{align}
    \bar \rho_{r,\gamma}^{t,t'} & = -\frac{1}{\delta} \int_{t'}^{t} R_\theta^\gamma(t,s) \bar \rho_{\ell,\gamma}^{s,t'} (\de s + \sqrt{\tau\delta} \de B^s) - \frac{1}{\delta} R_\theta^\gamma(t,t') \nabla_r \ell_{t'}(\bar r_\gamma^{t'};z) \,,                                             \\
    \rho_{r,\gamma}^{t,t'}      & = -\frac{1}{\delta} \ind(\ceil{t'} \leq \floor{t}) \int_{\ceil{t'}}^{\floor{t}} R_\theta^\gamma(t,s) \rho_{\ell,\gamma}^{s,t'} (\de s + \sqrt{\tau\delta} \de B^s) - \frac{1}{\delta} R_\theta^\gamma(t,t') \nabla_r \ell_{\floor{t'}}(r_\gamma^{t'};z) \,.
\end{align}
Let $\bar \rho_{\ell,\gamma}^{t,t'} = \nabla_r \ell_t(\bar r_\gamma^t;z) \bar \rho_{r,\gamma}^{t,t'}$ and $\rho_{\ell,\gamma}^{t,t'} = \nabla_r \ell_{\floor{t}}(r_\gamma^t;z) \rho_{r,\gamma}^{t,t'}$.
Similarly, define $\bar D_{\ell,\gamma}^{t,t'}$ and $D_{\ell,\gamma}^{t,t'}$ for processes defined by $\calT_{\theta\to\ell}$ and $\calT_{\theta\to\ell}^\gamma$, respectively.
Then, we have
\begin{align}
    \norm{\bar R_\ell^\gamma(t,t') - R_\ell^\gamma(t,t')}_2^2 & \leq \E \norm{\bar\rho_{\ell,\gamma}^{t,t'} - \rho_{\ell,\gamma}^{t,t'}}_2^2 + \tau\delta \E \norm{\bar D_{\ell,\gamma}^{t,t'} - D_{\ell,\gamma}^{t,t'}}_2^2 \,.
\end{align}
We first bound $\E \norm{\bar\rho_{\ell,\gamma}^{t,t'} - \rho_{\ell,\gamma}^{t,t'}}_2^2$. We have
\begin{align}
    \E \norm{\bar\rho_{\ell,\gamma}^{t,t'} - \rho_{\ell,\gamma}^{t,t'}}_2^2 & = \E \norm{\nabla_r \ell_t(\bar r_\gamma^t;z) \bar \rho_{r,\gamma}^{t,t'} - \nabla_r \ell_{\floor{t}}(r_\gamma^t;z) \rho_{r,\gamma}^{t,t'}}_2^2 \notag      \\
                                                              & \leq 2 \sqrt{\E\norm{\bar \rho_{r,\gamma}^{t,t'}}_2^4} \sqrt{\E\norm{\nabla_r \ell_t(\bar r_\gamma^t;z) - \nabla_r \ell_{\floor{t}}(r_\gamma^t;z)}_2^4} + 2 M^2 \E\norm{\bar \rho_{r,\gamma}^{t,t'} - \rho_{r,\gamma}^{t,t'}}_2^2 \notag \\
                                                              & \leq K \sqrt{\gamma^4 + \E\norm{\bar r_\gamma^t - r_\gamma^t}_2^4} + K \E\norm{\bar \rho_{r,\gamma}^{t,t'} - \rho_{r,\gamma}^{t,t'}}_2^2 \,.
\end{align}

First, we bound $\E \norm{\bar r_\gamma^t - r_\gamma^t}_2^4$. We have
\begin{align}
    & \E\norm{\bar r_\gamma^t - r_\gamma^t}_2^4 \notag \\
    & \leq K \int_0^{\floor{t}} \norm{R_\theta^\gamma(t,s)}_2^4 \E\norm{\ell_s(\bar r_\gamma^s;z) - \ell_{\floor{s}}(r_\gamma^s;z)}_2^4 \de s + K (t - \floor{t}) \int_{\floor{t}}^t \norm{R_\theta^\gamma(t,s)}_2^4 \E\norm{\ell_s(\bar r_\gamma^s;z)}_2^4 \de s \notag \\
                                              & \leq K \int_0^{\floor{t}} (\gamma^4 + \E\norm{\bar r_\gamma^s - r_\gamma^s}_2^4) \de s + K \gamma^2 \leq K \gamma^2 + K \int_0^t \E\norm{\bar r_\gamma^s - r_\gamma^s}_2^4 \de s \,.
\end{align}
By Gr\"onwall's inequality, we have
\begin{align}
    \sup_{t \in [0,T]} \E \norm{r_1^t - r_2^t}_2^4 \leq K \gamma^2 \,.
\end{align}
Next, we bound $\E\norm{\bar \rho_{r,\gamma}^{t,t'} - \rho_{r,\gamma}^{t,t'}}_2^2$. We have
\begin{align}
    \norm{\rho_{r,\gamma}^{t,t'} - \bar \rho_{r,\gamma}^{t,t'}}_2 & \leq \frac{1}{\delta} \norm*{\int_{t'}^{\ceil{t'}} R_\theta^\gamma(t,s) \bar \rho_{\ell,\gamma}^{s,t'} (\de s + \sqrt{\tau\delta} \de B^s)}_2 + \frac{1}{\delta} \norm*{\int_{\floor{t}}^{t} R_\theta^\gamma(t,s) \bar \rho_{\ell,\gamma}^{s,t'} (\de s + \sqrt{\tau\delta} \de B^s)}_2 \notag \\
                                                                  & \qquad + \frac{1}{\delta} \norm*{\int_{\ceil{t'}}^{\floor{t}} R_\theta^\gamma(t,s) (\rho_{\ell,\gamma}^{s,t'} - \bar \rho_{\ell,\gamma}^{s,t'}) (\de s + \sqrt{\tau\delta} \de B^s)}_2 \notag                                                                                                  \\
                                                                  & \qquad + \frac{1}{\delta} \norm{R_\theta^\gamma(t,t') (\nabla_r \ell_t(\bar r_\gamma^{t'};z) - \nabla_r \ell_{\floor{t'}}(r_\gamma^{t'};z))}_2 \,,
\end{align}
Thus, we have
\begin{align}
    \E\norm{\rho_{r,\gamma}^{t,t'} - \bar \rho_{r,\gamma}^{t,t'}}_2^2 & \leq K \int_{t'}^{\ceil{t'}} \norm{R_\theta^\gamma(t,s)}_2^2 \E \norm{\bar \rho_{\ell,\gamma}^{s,t'}}_2^2 \de s + K \int_{\floor{t}}^{t} \norm{R_\theta^\gamma(t,s)}_2^2 \E \norm{\bar \rho_{\ell,\gamma}^{s,t'}}_2^2 \de s \notag \\
                                                                      & \qquad + K \int_{\ceil{t'}}^{\floor{t}} \norm{R_\theta^\gamma(t,s)}_2^2 \E \norm{\rho_{\ell,\gamma}^{s,t'} - \bar \rho_{\ell,\gamma}^{s,t'}}_2^2 \de s \notag                                                                      \\
                                                                      & \qquad + K \norm{R_\theta^\gamma(t,t')}_2^2 \E \norm{\nabla_r \ell_t(\bar r_\gamma^{t'};z) - \nabla_r \ell_{\floor{t'}}(r_\gamma^{t'};z)}_2^2 \notag                                                                               \\
                                                                      & \leq K \gamma + K \int_{t'}^{t} \E \norm{\rho_{\ell,\gamma}^{s,t'} - \bar \rho_{\ell,\gamma}^{s,t'}}_2^2 \de s \,.
\end{align}
By Gr\"onwall's inequality, we have
\begin{align}
    \sup_{0 \leq t' \leq t \leq T} \E \norm{\rho_{r,\gamma}^{t,t'} - \bar \rho_{r,\gamma}^{t,t'}}_2^2 \leq K \gamma \,.
\end{align}

Next, we bound $\E \norm{\bar D_{\ell,\gamma}^{t,t'} - D_{\ell,\gamma}^{t,t'}}_2^2$ similarly. We do not repeat the details, but following the same argument, we have
\begin{align}
    \E \norm{\bar D_{\ell,\gamma}^{t,t'} - D_{\ell,\gamma}^{t,t'}}_2^2 \leq K \gamma \,.
\end{align}
Therefore, we have
\begin{align}
    \norm{\bar R_\ell^\gamma(t,t') - R_\ell^\gamma(t,t')}_2^2 & \leq K \gamma \,. \label{eq:dist_Rl_gamma}
\end{align}

By \cref{eq:dist_Sigmal_gamma,eq:dist_Gamma_gamma,eq:dist_Rl_gamma}, we have
\begin{align}
    \dist_\lambda(\bar Y^\gamma,Y^\gamma) \leq K \sqrt{\gamma} \,.
\end{align}

\paragraph{Bound of $\dist_\lambda(\bar C_\theta^\gamma,C_\theta^\gamma)$.}
Let $\bar U_\gamma \sim \GP(0,\bar \Sigma_\ell^\gamma/\delta)$ and $U_\gamma \sim \GP(0,\Sigma_\ell^\gamma/\delta)$ be Gaussian processes coupled such that
\begin{align}
    \sup_{t \in [0,T]} \napier^{-\lambda t} \sqrt{\E \norm{\bar U_\gamma^t - U_\gamma^t}_2^2} \leq 2 \cdot \dist_\lambda(\bar \Sigma_\ell^\gamma,\Sigma_\ell^\gamma) \,.
\end{align}
Let $\bar \theta_\gamma^t$ and $\theta_\gamma^t$ be the solution of
\begin{align}
    \bar \theta_\gamma^t & = \theta^0 + \bar U_\gamma^t - \int_0^t \ab(h_s(\bar \theta_\gamma^s) + \bar \Gamma^\gamma(s)\bar \theta_\gamma^s + \int_0^s \bar R_\ell^\gamma(s,s') \bar \theta_\gamma^{s'} \de s') \de s \,,      \\
    \theta_\gamma^t      & = \theta^0 + U_\gamma^t - \int_0^{\floor{t}} \ab(h_{\floor{s}}(\theta_\gamma^s) + \Gamma^\gamma(s)\theta_\gamma^s + \int_0^{\floor{s}} R_\ell^\gamma(s,s') \theta_\gamma^{s'} \de s') \de s \,.
\end{align}
Then, we have
\begin{align}
    \E \norm{\bar \theta_\gamma^t - \theta_\gamma^t}_2^2 & \leq K \ab(\gamma + \E\norm{\bar U_\gamma^t - U_\gamma^t}_2^2 + \sup_{t \in [0,T]} \norm{\bar \Gamma^\gamma(t) - \Gamma^\gamma(t)}_2^2 + \sup_{0 \leq t' \leq t \leq T} \norm{\bar R_\ell^\gamma(t,t') - R_\ell^\gamma(t,t')}_2^2) \notag \\
                                                         & \qquad + K \int_0^t \E \norm{\bar \theta^s_\gamma - \theta^s_\gamma}_2^2 \de s \notag                                                                                                                                                     \\
                                                         & \leq K \ab(\gamma + \dist_\lambda(\bar Y^\gamma,Y^\gamma)^2) + K \int_0^t \E \norm{\bar \theta^s_\gamma - \theta^s_\gamma}_2^2 \de s \,.
\end{align}
By Gr\"onwall's inequality, we have
\begin{align}
    \sup_{t \in [0,T]} \E \norm{\bar \theta_\gamma^t - \theta_\gamma^t}_2^2 \leq K \ab(\gamma + \dist_\lambda(\bar Y^\gamma,Y^\gamma)^2) \,. \label{eq:dist_theta_gamma}
\end{align}

Let $\{(\bar w_\gamma^t,w_\gamma^t) \}_{t \in [0,T]}$ be a centered Gaussian process with covariance $\E\ab[\pmat{\bar \theta^t_\gamma \\ \theta_\gamma^t} \pmat{\bar \theta^t_\gamma \\ \theta_\gamma^t}^\transpose]$.
Since $\bar w_\gamma$ and $w_\gamma$ have covariance kernels $\bar C_\theta^\gamma$ and $C_\theta^\gamma$ respectively, we have
\begin{align}
    \dist_\lambda(\bar C_\theta^\gamma,C_\theta^\gamma) \leq \sup_{t \in [0,T]} \napier^{-\lambda t} \sqrt{\E \norm{\bar w_\gamma^t - w_\gamma^t}_2^2} = \sup_{t \in [0,T]} \napier^{-\lambda t} \sqrt{\E\norm{\bar \theta_\gamma^t - \theta_\gamma^t}_2^2} \leq K (\sqrt{\gamma} + \dist_\lambda(\bar Y^\gamma,Y^\gamma)) \,. \label{eq:dist_Ctheta_gamma}
\end{align}

\paragraph{Bound of $\dist_\lambda(R_\theta^1,R_\theta^2)$.}
Let $\bar \rho_{\theta,\gamma}^{t,t'}$ and $\rho_{\theta,\gamma}^{t,t'}$ be the stochastic processes satisfying
\begin{align}
    \bar \rho_{\theta,\gamma}^{t,t'} & = I_m - \int_{t'}^{t} \ab((\nabla_\theta h_s(\bar\theta_\gamma^s) + \bar \Gamma^\gamma(s)) \bar \rho_{\theta,\gamma}^{s,t'} + \int_{t'}^{s} \bar R_\ell^\gamma(s,s') \bar \rho_{\theta,\gamma}^{s',t'} \de s') \de s \,,                                                  \\
    \rho_{\theta,\gamma}^{t,t'}      & = I_m - \ind(\ceil{t'} \leq \floor{t}) \int_{\ceil{t'}}^{\floor{t}} \ab((\nabla_\theta h_{\floor{s}}(\theta_\gamma^s) + \Gamma^\gamma(s)) \rho_{\theta,\gamma}^{s,t'} + \int_{\ceil{t'}}^{\floor{s}} R_\ell^\gamma(s,s') \rho_{\theta,\gamma}^{s',t'} \de s') \de s \,,
\end{align}
and set $\bar R_\theta^\gamma(t,t') = \E[\bar \rho_{\theta,\gamma}^{t,t'}]$ and $R_\theta^\gamma(t,t') = \E[\rho_{\theta,\gamma}^{t,t'}]$.
We have
\begin{align}
    \norm{\bar \rho_{\theta,\gamma}^{t,t'} - \rho_{\theta,\gamma}^{t,t'}}_2 & \leq K \ab(\gamma + \sup_{t \in [0,T]} \norm{\bar \Gamma^\gamma(t) - \Gamma^\gamma(t)}_2 + \sup_{0 \leq t' \leq t \leq T} \norm{\bar R_\ell^\gamma(t,t') - R_\ell^\gamma(t,t')}_2) \notag \\
                                                                & \qquad + K \int_{t'}^{t} \norm{\bar \rho_{\theta,\gamma}^{s,t'} - \rho_{\theta,\gamma}^{s,t'}}_2 \de s \notag                                                                                         \\
                                                                & \leq K (\gamma + \dist_\lambda(\bar Y^\gamma,Y^\gamma)) + K \int_{t'}^{t} \norm{\bar \rho_{\theta,\gamma}^{s,t'} - \rho_{\theta,\gamma}^{s,t'}}_2 \de s \,.
\end{align}
By Gr\"onwall's inequality, we have
\begin{align}
    \sup_{0 \leq t' \leq t \leq T} \norm{\bar \rho_{\theta,\gamma}^{t,t'} - \rho_{\theta,\gamma}^{t,t'}}_2 \leq K (\gamma + \dist_\lambda(\bar Y^\gamma,Y^\gamma)) \,.
\end{align}
Thus, we have
\begin{align}
    \dist_\lambda(\bar R_\theta^\gamma, R_\theta^\gamma) \leq \sup_{0 \leq t' \leq t \leq T} \norm{\bar R_\theta^\gamma(t,t') - R_\theta^\gamma(t,t')}_2 & \leq \sup_{0 \leq t' \leq t \leq T} \E \norm{\bar \rho_{\theta,\gamma}^{t,t'} - \rho_{\theta,\gamma}^{t,t'}}_2 \notag \\
    & \leq K (\gamma + \dist_\lambda(\bar Y^\gamma,Y^\gamma)) \,. \label{eq:dist_Rtheta_gamma}
\end{align}

By \cref{eq:dist_Ctheta_gamma,eq:dist_Rtheta_gamma}, we have
\begin{align}
    \dist_\lambda(\bar X^\gamma,X^\gamma) \leq K (\sqrt{\gamma} + \dist_\lambda(\bar Y^\gamma,Y^\gamma)) \,.
\end{align}

\subsubsection{Proof of Lemma \ref{lem:dmft_process_coupling}} \label{app:proof_dmft_process_coupling}

Following the same calculation that led to \cref{eq:dist_r_gamma}, we can show that there exists a coupling of the processes $r^t$ and $r_\gamma^t$ such that
\begin{align}
    \sup_{t \in [0,T]} \sqrt{\E \norm{r^t - r_\gamma^t}_2^2} \leq K \ab(\sqrt{\gamma} + \dist_\lambda(X,X^\gamma)) \,.
\end{align}
Following the same calculation that led to \cref{eq:dist_theta_gamma}, we can show that there exists a coupling of the processes $\theta^t$ and $\theta_\gamma^t$ such that
\begin{align}
    \sup_{t \in [0,T]} \sqrt{\E \norm{\theta^t - \theta_\gamma^t}_2^2} \leq K \ab(\sqrt{\gamma} + \dist_\lambda(Y,Y^\gamma)) \,.
\end{align}

\section{Details of Applications and Special Cases}

\subsection{Infinite Data Limit}
\label{app:infinite_data_limit}

\subsubsection{Derivation of the Reduced DMFT Equation}

We show that in the infinite data limit $\delta \to \infty$, the DMFT equations reduce to the simple form in \cref{eq:dmft_infinite_data}.
It is easy to see that as $\delta \to \infty$, $R_\ell(t,t') \to 0$ and hence $R_\theta(t,t') \to I_m$. Then, the equations for $\theta^t$ and $r^t$ reduce to
\begin{align}
    \theta^t = \theta^0 + U^t - \int_0^t (h_s(\theta^s) + \Gamma(s) \theta^s) \de s \,, \quad U \sim \GP(0,\tilde \Sigma_\ell) \,, \quad r^t = w^t \,, \quad w \sim \GP(0,C_\theta) \,,
\end{align}
where we introduced the rescaled variable $\tilde \Sigma_\ell \coloneqq \Sigma_\ell/\delta$. In addition, the equation for $\tilde\Sigma_\ell$ reduces to
\begin{align}
    \tilde \Sigma_\ell(t,t') = \tau \int_0^{t \land t'} \E[\ell_s(w^s;z) \ell_s(w^s;z)^\transpose] \de s \,.
\end{align}
Defining $C_\ell(t) = \E[\ell_t(w^t;z) \ell_t(w^t;z)^\transpose]$, $\tilde \Sigma_\ell$ is a covariance kernel of an integrated Brownian motion:
\begin{align}
    U^t = \int_0^t \sqrt{\tau C_\ell(s)} \de W^s \,,
\end{align}
where $W^t$ is a standard Brownian motion in $\reals^m$. Thus, $\theta^t$ follows the following SDE:
\begin{align}
    \de \theta^t = -(h_t(\theta^t) + \Gamma(t) \theta^t) \de t + \sqrt{\tau C_\ell(t)} \de W^t \,.
\end{align}
These equations recover \cref{eq:dmft_infinite_data}.

For planted models of the form in \cref{eq:sgf_planted}, the DMFT equation in the infinite data limit \eqref{eq:dmft_infinite_data} becomes
\begin{equation}
    \begin{gathered}
        \de \theta^t  = -\ab(h_t(\theta^t) + \Gamma(t)\theta^t + \Gamma^*(t) \theta^*) \de t + \sqrt{\tau C_\ell(t)}\de W^t \,, \\
        C_\theta(t,t') = \E[\theta^t \theta^{t'\transpose}] \quad (t,t' \in [0,T]\cup\{*\}) \,, \quad C_\ell(t) = \E[\ell_t(w^t,w^*;z) \ell_t(w^t,w^*;z)^\transpose] \,, \\
        \Gamma(t) = \E[\nabla_w\ell_t(w^t,w^*;z)] \,, \quad \Gamma^*(t) = \E[\nabla_{w^*}\ell_t(w^t,w^*;z)] \,, \quad w \sim \GP(0,C_\theta) \,.
    \end{gathered} \label{eq:dmft_infinite_data_planted}
\end{equation}

\subsubsection{Example: Linear Regression}

Consider the linear regression setting described in \cref{sec:linear_regression}. This corresponds to the choice $m=1$, $h_t = 0$, and $\ell_t(r,r^*;z) = r - r^* - z$. Then, \cref{eq:dmft_infinite_data_planted} reduces to
\begin{equation}
    \begin{gathered}
        \de \theta^t  = -(\theta^t - \theta^*) \de t + \sqrt{\tau C_\ell(t)}\de W^t \,, \\
        C_\theta(t,t') = \E[\theta^t \theta^{t'}] \quad (t,t' \in [0,T]\cup\{*\}) \,, \quad C_\ell(t) = \E[(w^t - w^* - z)^2] \,, \quad w \sim \GP(0,C_\theta) \,.
    \end{gathered}
\end{equation}
Using the Fokker--Planck equation, the density of $\theta^t$ given by the above SDE is given by the following partial differential equation (PDE):
\begin{align}
    \partial_t \mu(t,\theta) = \partial_\theta \ab((\theta - \theta^*) \mu(t,\theta)) + \frac{\tau C_\ell(t)}{2} \partial_\theta^2 \mu(t,\theta) \,.
\end{align}
This equation coincides with the PDE derived in \citet{wang2017scaling}.

Furthermore, the training dynamics for test errors $\Etrain(t) = \E[(\theta - \theta^*)^2] + \sigma^2$ can be obtained in closed form, where $\sigma^2 \coloneqq \E[z^2]$.
Let $\rho^2 \coloneqq \E[(\theta^*)^2]$. Then, we have $C_\ell(t) = C_\theta(t,t) - 2 C_\theta(t,*) + \rho^2 + \sigma^2$. By using It\^o's lemma on $C_\theta(t,*) = \E[\theta^t \theta^*]$ and $C_\theta(t,t) = \E[(\theta^t)^2]$, we obtain the following system of ODEs for $C_\theta(t,*)$ and $C_\theta(t,t)$:
\begin{align}
    \diff{}{t} C_\theta(t,*) & = - C_\theta(t,*) + \rho^2 \,,                                                                        \\
    \diff{}{t} C_\theta(t,t) & = -2 C_\theta(t,t) + 2 C_\theta(t,*) + \tau (C_\theta(t,t) - 2 C_\theta(t,*) + \rho^2 + \sigma^2) \,.
\end{align}
These are linear ODEs and can be solved in closed form as follows (assuming zero initialization, i.e., $\theta^0 = C_\theta(0,0) = C_\theta(0,*) = 0$):
\begin{align}
    C_\theta(t,*) = \rho^2 (1 - \napier^{-t}) \,, \quad C_\theta(t,t) = \rho^2 (1 - 2 \napier^{-t} + \napier^{-(2 - \tau) t}) + \frac{\tau\sigma^2}{2 - \tau} (1 - \napier^{-(2 - \tau) t}) \,.
\end{align}

\subsection{Planted Models}
\label{app:dmft_planted}

\subsubsection{The DMFT Equation}
\label{app:dmft_eq_planted}

We state the DMFT equation $\mathfrak{S}^*$ for planted model \eqref{eq:sgf_planted} for functions $R_\ell,C_\theta \colon (\reals_{\geq 0} \cup \{*\})^2 \to \reals^{m \times m}$, $R_\theta,\Sigma_\ell\colon \reals_{\geq 0}^2 \to \reals^{m \times m}$, and $\Gamma \colon \reals_{\geq 0} \to \reals^{m \times m}$ self-consistently as follows.
First, given $\Sigma_\ell,R_\ell,\Gamma$, define stochastic processes $\{\theta^t \in \reals^m \}_{t \geq 0}$ and $\{\rho_\theta^{t,t'} \in \reals^{m \times m} \}_{t \geq t' \geq 0}$ by the following equations.
\begin{align}
    \theta^t           & = \theta^0 + U^t - \int_0^t \ab((h_s(\theta^s) + \Gamma(s)\theta^s) + \int_0^s R_\ell(s,s') \theta^{s'} \de s' + R_\ell(t,*) \theta^*) \de s \,, \label{eq:theta_planted}          \\
    \rho_\theta^{t,t'} & = I_m - \int_{t'}^t \ab((\nabla_\theta h_s(\theta^s) + \Gamma(s)) \rho_\theta^{s,t'} + \int_{t'}^s R_\ell(s,s') \rho_\theta^{s',t'} \de s') \de s \,, \label{eq:rho_theta_planted}
\end{align}
where $U \sim \GP(0,\Sigma_\ell/\delta)$.
Then, set $C_\theta,R_\theta$ as
\begin{align}
    C_\theta(t,t') = \E[\theta^t \theta^{t'\transpose}] \quad (s,t \in [0,T] \cup \{*\}) \,, \quad R_\theta(t,t') = \E[\rho_\theta^{t,t'}] \quad (t \geq t') \,,
\end{align}
and $R_\theta(t,t') = 0$ for $t < t'$.

Next, given $C_\theta,R_\theta$, define stochastic processes $\{r^t \in \reals^m \}_{t \geq 0}$, $r^* \in \reals^m$, and $\{\rho_\ell^{t,t'},\rho_{\ell*}^{t,t'},D_\ell^{t,t'},\allowbreak D_{\ell*}^{t,t'} \in \reals^{m \times m}\}_{t \geq t' \geq 0}$ by the following equations.
\begin{align}
    r^t                 & = w^t -\frac{1}{\delta} \int_0^t R_\theta(t,s) \ell_s(r^s,r^*;z) (\de s + \sqrt{\tau\delta} \de B^s) \,, \quad r^* = w^* \,, \quad \quad w \sim \GP(0,C_\theta) \,, \label{eq:r_planted}                                           \\
    \rho_\ell^{t,t'}    & = \nabla_r \ell_t(r^t,r^*;z) \rho_r^{t,t'} \,, \label{eq:rho_ell_planted}                                                                                                                                                          \\
    \rho_{\ell*}^{t,t'} & = \nabla_r \ell_t(r^t,r^*;z) \rho_{r*}^{t,t'} \,, \label{eq:rho_ell_star_planted}                                                                                                                                                  \\
    D_\ell^{t,t'}       & = \nabla_r \ell_t(r^t,r^*;z) \ab( -\frac{1}{\delta} \int_{t'}^t R_\theta(t,s) D_\ell^{s,t'} (\de s + \sqrt{\tau\delta} \de B^s)) + \nabla_r^2 \ell_t(r^t,r^*;z) [D_r^{t,t'}]  \rho_r^{t,t'} \,, \label{eq:D_ell_planted}           \\
    D_{\ell*}^{t,t'}    & = \nabla_r \ell_t(r^t,r^*;z) \ab( -\frac{1}{\delta} \int_{t'}^t R_\theta(t,s) D_{\ell*}^{s,t'} (\de s + \sqrt{\tau\delta} \de B^s)) + \nabla_r^2 \ell_t(r^t,r^*;z) [D_r^{t,t'}] \rho_{r*}^{t,t'} \,, \label{eq:D_ell_star_planted}
\end{align}
where $B^t$ is a Brownian motion in $\reals$, and we defined the auxiliary processes $\rho_r^{t,t'}$, $\rho_{r*}^{t,t'} \in \reals^{m \times m}$, and $D_r^{t,t'} \in \reals^m$ as
\begin{align}
    \rho_r^{t,t'}    & = -\frac{1}{\delta} \int_{t'}^t R_\theta(t,s) \rho_\ell^{s,t'} (\de s + \sqrt{\tau\delta} \de B^s) - \frac{1}{\delta} R_\theta(t,t') \nabla_r \ell_{t'}(r^{t'},r^*;z) \,, \label{eq:rho_r_planted}                     \\
    \rho_{r*}^{t,t'} & = -\frac{1}{\delta} \int_{t'}^t R_\theta(t,s) \rho_{\ell*}^{s,t'} (\de s + \sqrt{\tau\delta} \de B^s) - \frac{1}{\delta} R_\theta(t,t') \nabla_{r^*} \ell_{t'}(r^{t'},r^*;z) \,, \label{eq:rho_r_star_planted}         \\
    D_r^{t,t'}       & = -\frac{1}{\delta} \int_{t'}^t R_\theta(t,s) \nabla_r \ell_s(r^s,r^*;z) D_r^{s,t'} (\de s + \sqrt{\tau\delta} \de B^s) - \sqrt{\frac{\tau}{\delta}} R_\theta(t,t') \ell_{t'}(r^{t'},r^*;z) \,. \label{eq:D_r_planted}
\end{align}
Then, set $\Sigma_\ell,R_\ell,\Gamma$ as
\begin{align}
    \Sigma_\ell(t,t') & = \E[L^t L^{t'\transpose}] \,, \quad L^t \coloneqq \int_0^t \ell_s(r^s,r^*;z) (\de s + \sqrt{\tau\delta} \de B^s) \,, \label{eq:Sigma_ell_planted}                      \\
    R_\ell(t,t')      & = \E[\rho_\ell^{t,t'}] + \sqrt{\tau\delta} \E[D_\ell^{t,t'}] \quad (t \geq t') \,, \label{eq:R_ell_planted}                                                             \\
    R_\ell(t,*)       & = \E[\nabla_{r^*}\ell_t(r^t,r^*;z)] + \int_{0}^t \E[\rho_{\ell*}^{t,t'} + \sqrt{\tau\delta} D_{\ell*}^{t,t'}] \de t' \quad (t \geq 0) \,, \label{eq:R_ell_star_planted} \\
    \Gamma(t)         & = \E[\nabla_r\ell_t(r^t,r^*;z)] \,, \label{eq:Gamma_planted}
\end{align}
and $R_\ell(t,t') = 0$ for $t < t'$.

Then, the solution of the DMFT system $\mathfrak{S}^*$ is defined as a fixed point of the above two mappings.

In the informal notation, the DMFT equation for the planted model can be written as
\begin{equation}
    \begin{aligned}
        \diff{}{t} \theta^t & = u^t - (h_t(\theta^t) + \Gamma(t)\theta^t) - \int_0^t R_\ell(t,s) \theta^{s} \de s - R_\ell(t,*) \theta^* \,, \quad u \sim \GP(0,C_\ell/\delta) \,,             \\
        r^t                 & = w^t -\frac{1}{\delta} \int_0^t R_\theta(t,s) \ell_s(r^s,r^*;z) (\de s + \sqrt{\tau\delta} \de B^s) \,, \quad r^* = w^* \,, \quad w \sim \GP(0,C_\theta) \,,    \\
        C_\theta(t,t')      & = \E[\theta^t \theta^{t'\transpose}] \quad (t,t' \in [0,T] \cup \{*\}) \,, \quad R_\theta(t,t') = \E\ab[\diffp{\theta^t}{u^{t'}}] \,,                            \\
        C_\ell(t,t')        & = \E[\ell_t(r^t,r^*;z) (1 + \sqrt{\tau\delta} \dot B^t) \ell_{t'}(r^{t'},r^*;z)^\transpose (1 + \sqrt{\tau\delta} \dot B^{t'})]\,,                               \\
        R_\ell(t,t')        & = \E\ab[\diffp{\ell_t(r^t,r^*;z)}{w^{t'}}] \,, \quad R_\ell(t,*) = \E\ab[\diffp{\ell_t(r^t,r^*;z)}{w^*}] \,, \quad \Gamma(t) = \E[\nabla_r\ell_t(r^t,r^*;z)] \,.
    \end{aligned}
\end{equation}

\subsubsection{Proof of Corollary \ref{cor:sgf_dmft_planted}}
\label{app:proof_sgf_dmft_planted}

We transform the DMFT equation $\mathfrak{S}$ applied to the planted model \eqref{eq:sgf_planted} to the DMFT equation $\mathfrak{S}^*$ defined above.
We distinguish the variables in $\mathfrak{S}$ by adding a bar over them, e.g., $\bar\theta^t$.
The variables in $\mathfrak{S}$ have dimensions $2m$. Identifying components that are trivially zero, we see that the solution to $\mathfrak{S}$ is of the form
\begin{equation}
    \begin{gathered}
        \bar\theta^t = \pmat{\bar\theta_1^t \\ \theta^*} \,, \quad \bar U^t = \pmat{\bar U_1^t \\ 0} \,, \quad \bar\rho_\theta^{t,t'} = \pmat{\bar\rho_{\theta,1}^{t,t'} & {-} \\ 0 & I_m} \,, \\
        \bar C_\theta(t,t') = \pmat{\bar C_\theta^{11}(t,t') & \bar C_\theta^{12}(t,t') \\ \bar C_\theta^{12}(t,t')^\transpose & \bar C_\theta^{22}(t,t')} \,, \quad \bar R_\theta(t,t') = \pmat{\bar R_\theta^{1}(t,t') & {-} \\ 0 & I_m} \,, \\
        \bar r^t = \pmat{\bar r_1^t \\ \bar r_2^t} \,, \quad \bar w^t = \pmat{\bar w_1^t \\ \bar w_2^t} \,, \quad \bar\rho_\ell^{t,t'} = \pmat{\bar\rho_{\ell,1}^{t,t'} & \bar \rho_{\ell,2}^{t,t'} \\ 0 & 0} \,, \quad \bar D_\ell^{t,t'} = \pmat{\bar D_{\ell,1}^{t,t'} & \bar D_{\ell,2}^{t,t'} \\ 0 & 0} \,, \\
        \bar \rho_r^{t,t'} = \pmat{\bar \rho_{r,1}^{t,t'} & \bar \rho_{r,2}^{t,t'} \\ 0 & 0} \,, \quad \bar D_r^{t,t'} = \pmat{\bar D_{r,1}^{t,t'} \\ 0} \,, \\
        \bar\Sigma_\ell(t,t') = \pmat{\bar\Sigma_\ell^{1}(t,t') & 0 \\ 0 & 0} \,, \quad \bar R_\ell(t,t') = \pmat{\bar R_\ell^{1}(t,t') & \bar R_\ell^{2}(t,t') \\ 0 & 0} \,, \quad \bar\Gamma(t) = \pmat{\bar\Gamma_{1}(t) & \bar\Gamma_{2}(t) \\ 0 & 0} \,.
    \end{gathered}
\end{equation}
Here, we indicated by ${-}$ the irrelevant variables. These variables satisfy
\begin{align}
    \bar\theta_1^t             & = \theta^0 + \bar U_1^t - \int_0^t \ab(h_s(\bar\theta_1^s) + \bar\Gamma_{1}(s) \bar\theta_1^s + \int_0^s \bar R_\ell^{1}(s,s') \bar\theta_1^{s'} \de s') \de s \notag                                                                                                       \\
                               & \qquad - \int_0^t \ab(\bar\Gamma_2(s) + \int_0^s \bar R_\ell^{2}(s,s') \de s') \theta^* \de s \,, \quad \bar U_1 \sim \GP(0,\bar\Sigma_\ell^{11}/\delta) \,,                                                                                                                \\
    \bar\rho_{\theta,1}^{t,t'} & = I_m - \int_{t'}^t \ab((\nabla_\theta h_s(\bar\theta_1^s) + \bar\Gamma_{1}(s)) \bar\rho_{\theta,1}^{s,t'} + \int_{t'}^s \bar R_\ell^{1}(s,s') \bar\rho_{\theta,1}^{s',t'} \de s') \de s \,,                                                                                \\
    \bar C_\theta^{11}(t,t')   & = \E[\bar\theta_1^t \bar\theta_1^{t'\transpose}] \,, \quad \bar C_\theta^{12}(t,t') = \E[\bar\theta_1^t \theta^{*\transpose}] \,, \quad \bar C_\theta^{22}(t,t') = \E[\theta^* \theta^{*\transpose}] \,, \quad \bar R_\theta^{1}(t,t') = \E[\bar\rho_{\theta,1}^{t,t'}] \,, \\
    \bar r_1^t                 & = \bar w_1^t -\frac{1}{\delta} \int_0^t \bar R_\theta^{1}(t,s) \ell_s(\bar r_1^s,\bar r_2^s;z) (\de s + \sqrt{\tau\delta} \de B^s) \,, \quad \bar r_2^t = \bar w_2^t \,, \quad \pmat{\bar w_1                                                                               \\ \bar w_2} \sim \GP(0,\bar C_\theta) \,, \\
    \bar\rho_{\ell,1}^{t,t'}   & = \nabla_r \ell_t(\bar r_1^t,\bar r_2^t;z) \bar\rho_{r,1}^{t,t'} \,, \quad \bar\rho_{\ell,2}^{t,t'} = \nabla_r \ell_t(\bar r_1^t,\bar r_2^t;z) \bar\rho_{r,2}^{t,t'} \,,                                                                                                    \\
    \bar D_{\ell,1}^{t,t'}     & = \nabla_r \ell_t(\bar r_1^t,\bar r_2^t;z) \ab( -\frac{1}{\delta} \int_{t'}^t \bar R_\theta^{1}(t,s) \bar D_{\ell,1}^{s,t'} (\de s + \sqrt{\tau\delta} \de B^s)) + \nabla_{rr}^2 \ell_t(\bar r_1^t,\bar r_2^t;z) [\bar D_{r,1}^{t,t'}] \bar\rho_{r,1}^{t,t'} \,,            \\
    \bar D_{\ell,2}^{t,t'}     & = \nabla_r \ell_t(\bar r_1^t,\bar r_2^t;z) \ab( -\frac{1}{\delta} \int_{t'}^t \bar R_\theta^{1}(t,s) \bar D_{\ell,2}^{s,t'} (\de s + \sqrt{\tau\delta} \de B^s)) + \nabla_{rr}^2 \ell_t(\bar r_1^t,\bar r_2^t;z) [\bar D_{r,1}^{t,t'}] \bar\rho_{r,2}^{t,t'} \,,            \\
    \bar\rho_{r,1}^{t,t'}      & = -\frac{1}{\delta} \int_{t'}^t \bar R_\theta^{1}(t,s) \bar\rho_{\ell,1}^{s,t'} (\de s + \sqrt{\tau\delta} \de B^s) - \frac{1}{\delta} \bar R_\theta^{1}(t,t') \nabla_r \ell_{t'}(\bar r_1^{t'},\bar r_2^{t'};z) \,,                                                        \\
    \bar\rho_{r,2}^{t,t'}      & = -\frac{1}{\delta} \int_{t'}^t \bar R_\theta^{1}(t,s) \bar\rho_{\ell,2}^{s,t'} (\de s + \sqrt{\tau\delta} \de B^s) - \frac{1}{\delta} \bar R_\theta^{1}(t,t') \nabla_w \ell_{t'}(\bar r_1^{t'},\bar r_2^{t'};z) \,,                                                        \\
    \bar D_{r,1}^{t,t'}        & = -\frac{1}{\delta} \int_{t'}^t \bar R_\theta^{1}(t,s) \nabla_r \ell_s(\bar r_1^s,\bar r_2^s;z) \bar D_{r,1}^{s,t'} (\de s + \sqrt{\tau\delta} \de B^s) - \sqrt{\frac{\tau}{\delta}} \bar R_\theta^{1}(t,t') \ell_{t'}(\bar r_1^{t'},\bar r_2^{t'};z) \,,                   \\
    \bar\Sigma_\ell^{1}(t,t')  & = \E[\bar L_1^t \bar L_1^{t'\transpose}] \,, \quad \bar L_1^t \coloneqq \int_0^t \ell_s(\bar r_1^s,\bar r_2^s;z) (\de s + \sqrt{\tau\delta} \de B^s) \,,                                                                                                                    \\
    \bar R_\ell^{1}(t,t')      & = \E[\bar\rho_{\ell,1}^{t,t'}] + \sqrt{\tau\delta} \E[\bar D_{\ell,1}^{t,t'}] \,, \quad \bar R_\ell^{2}(t,t') = \E[\bar\rho_{\ell,2}^{t,t'}] + \sqrt{\tau\delta} \E[\bar D_{\ell,2}^{t,t'}] \,,                                                                             \\
    \bar\Gamma_{1}(t)          & = \E[\nabla_r\ell_t(\bar r_1^t,\bar r_2^t;z)] \,, \quad \bar\Gamma_{2}(t) = \E[\nabla_{r^*}\ell_t(\bar r_1^t,\bar r_2^t;z)] \,.
\end{align}
Since $\bar C_\theta^{22}(t,t')$ is constant, $\bar r_2^t = \bar w_2^t$ is constant.
Then, $\mathfrak{S}$ reduces to $\mathfrak{S}^*$ by identifying
\begin{align}
    \begin{gathered}
        \bar\theta^t = \theta^t \,, \quad \bar U_1^t = U^t \,, \quad \bar\rho_{\theta,1}^{t,t'} = \rho_\theta^{t,t'} \,, \\
        \bar C_\theta^{11}(t,t') = C_\theta(t,t') \,, \quad \bar C_\theta^{12}(t,t') = C_\theta(t,*) \,, \quad \bar C_\theta^{22}(t,t') = C_\theta(*,*) \,, \quad \bar R_\theta^{1}(t,t') = R_\theta(t,t') \,, \\
        \bar r_1^t = r^t \,, \quad \bar w_1^t = w^t \,, \quad \bar r_2^t = \bar w_2^t = w^* \,, \quad \bar \rho_{\ell,1}^{t,t'} = \rho_{\ell}^{t,t'} \,, \bar \rho_{\ell,2}^{t,t'} = \rho_{\ell*}^{t,t'} \,, \quad \bar D_{\ell,1}^{t,t'} = D_{\ell}^{t,t'} \,, \bar D_{\ell,2}^{t,t'} = D_{\ell*}^{t,t'} \,, \\
        \bar \rho_{r,1}^{t,t'} = \rho_r^{t,t'} \,, \quad \bar \rho_{r,2}^{t,t'} = \rho_{r*}^{t,t'} \,, \quad \bar D_{r,1}^{t,t'} = D_r^{t,t'} \,, \\
        \bar \Sigma_\ell^{1}(t,t') = \Sigma_\ell(t,t') \,, \quad \bar R_\ell^{1}(t,t') = R_\ell(t,t') \,, \quad \bar\Gamma_{1}(t) = \Gamma(t) \,, \quad \bar \Gamma_2(t) + \int_0^t \bar R_\ell^2(t,t') \de t' = R_\ell(t,*) \,.
    \end{gathered}
\end{align}

\subsection{Linear Regression}
\label{app:proof_dmft_ridge}

In this section, we derive the DMFT equation for the SGF for the ridge regression (regularized linear regression) problem.
The loss for the ridge regression problem is defined as
\begin{align}
    \calL(\btheta) = \frac{1}{2n} \norm{\bX \btheta - \by}_2^2 + \frac{\lambda}{2d} \norm{\btheta}_2^2 \,, \quad \by = \bX \btheta^* + \bz \,,
\end{align}
with regularization parameter $\lambda \geq 0$. The SGF for this problem is given by
\begin{align}
    \de \btheta^t = -\ab(\lambda \btheta^t + \frac{1}{\delta} \bX^\transpose (\br^t - \br^* - \bz)) \de t + \sqrt{\frac{\tau}{\delta}} \sum_{i=1}^n \bx_i (r_i^t - r_i^* - z_i) \de B^t_i \,. \label{eq:sgf_ridge}
\end{align}
where $\br^t = \bX\btheta^t$ and $\br^* = \bX \btheta^*$.
For simplicity, we consider the case of zero initialization $\btheta^0 = 0$.
This is a special case of the SGF for planted models \eqref{eq:sgf_planted} with $m=1$, $h_t(\theta) = \lambda\theta$, and $\ell_t(r,r^*;z) = r - r^* - z$.

In this section, we prove the following proposition, which generalizes \cref{prop:dmft_linear} to the case of ridge regression with $\lambda \geq 0$.
\begin{proposition}[DMFT characterization of SGF for ridge regression] \label{prop:dmft_ridge}
    Assume that the noise $\bz \in \reals^n$ and the target parameter $\btheta^* \in \reals^d$ satisfy the same assumptions as in \cref{cor:sgf_dmft_planted}.
    Let $\rho^2 \coloneqq \E[(\theta^*)^2]$ and $\sigma^2 \coloneqq \E[z^2]$.
    Define the training error $\Etrain$ and the test error $\Etest$ for the parameter $\btheta$ as
    \begin{align}
        \Etrain(\btheta) = \frac{1}{n} \sum_{i=1}^n (\bx_i^\transpose \btheta - y_i)^2 \,, \quad \Etest(\btheta) = \E_{(\bx,y)}[(\bx^\transpose \btheta - y)^2] = \frac{1}{d} \norm{\btheta - \btheta^*}_2^2 + \sigma^2 \,.
    \end{align}

    Then, for any $0 \leq t_1,\dots,t_L < \infty$, we have
    \begin{align}
        \plim_{n,d\to\infty} \max_{l=1,\dots,L} \abs{\Etrain(\btheta^{t_l}) - \Etrain(t_l)} & = 0 \,, \quad \plim_{n,d\to\infty} \max_{l=1,\dots,L} \abs{\Etest(\btheta^{t_l}) - \Etest(t_l)} = 0 \,,
    \end{align}
    where $\Etrain(t)$ and $\Etest(t)$ solve the following system of linear Volterra equations:
    \begin{equation}
        \Etrain(t) = \Etrain_0(t) + \tau \int_0^t H_2(t-s) \Etrain(s) \de s \,, \quad \Etest(t) = \Etest_0(t) + \tau \int_0^t H_1(t-s) \Etrain(s) \de s \,, \label{eq:dmft_ridge_errors}
    \end{equation}
    where $H_i(t) \coloneqq \int x^i \napier^{-2(x + \lambda)t} \de\mu_\MP(x)$, and $\Etrain_0(t)$ and $\Etest_0(t)$ are the asymptotic train and test errors for the noiseless case $\tau = 0$, which are given by
    \begin{align}
        \Etrain_0(t) & = \rho^2 \int \frac{x (\lambda + x\napier^{-(x+\lambda)t})^2}{(x+\lambda)^2} \de\mu_\MP(x) + \frac{\sigma^2}{\delta} \int \frac{(\lambda + x\napier^{-(x+\lambda)t})^2}{(x+\lambda)^2} \de\mu_\MP(x) + \frac{\delta-1}{\delta}\sigma^2 \,, \label{eq:dmft_ridge_train0} \\
        \Etest_0(t)  & = \rho^2 \int \frac{(\lambda + x\napier^{-(x+\lambda)t})^2}{(x+\lambda)^2} \de\mu_\MP(x) + \frac{\sigma^2}{\delta} \int \frac{x}{(x+\lambda)^2} (1 - \napier^{-(x+\lambda)t})^2 \de\mu_\MP(x) + \sigma^2 \,. \label{eq:dmft_ridge_test0}
    \end{align}
\end{proposition}

\subsubsection{Simplifying the DMFT Equations}

We can derive the DMFT equation for ridge regression by specializing the DMFT equation $\mathfrak{S}^*$ for planted models in \cref{app:dmft_eq_planted}.
Since the loss is quadratic, i.e., $\partial_r \ell_t(r,r^*;z) = 1$, $\partial_{r^*} \ell_t(r,r^*;z) = -1$, and $\partial^2 \ell_t(r,r^*;z) = 0$, the DMFT equation simplifies significantly.
We have $\Gamma(t) = 1$, $D_\ell^{t,t'} = D_{\ell*}^{t,t'} = 0$, $\rho_r^{t,t'} = -\rho_{r*}^{t,t'}$, and $\rho_\ell^{t,t'} = -\rho_{\ell*}^{t,t'}$.
The DMFT equation for ridge regression reduces to
\begin{subequations}
    \begin{align}
        \theta^t           & = U^t - \int_0^t \ab( (1 + \lambda) \theta^s - \theta^* + \int_0^s R_\ell(s,s') (\theta^{s'} - \theta^*) \de s') \de s \,, \quad U \sim \GP(0,\Sigma_\ell/\delta) \,,     \\
        \rho_\theta^{t,t'} & = 1 - \int_{t'}^t \ab((1 + \lambda) \rho_\theta^{s,t'} + \int_{t'}^s R_\ell(s,s') \rho_\theta^{s',t'} \de s') \de s \,,                                                   \\
        C_\theta(t,t')     & = \E[\theta^t \theta^{t'}] \quad (s,t \in [0,T] \cup \{*\}) \,, \quad R_\theta(t,t') = \E[\rho_\theta^{t,t'}] \quad (t \geq t') \,,                                       \\
        r^t                & = w^t -\frac{1}{\delta} \int_0^t R_\theta(t,s) (r^s - r^* - z) (\de s + \sqrt{\tau\delta} \de B^s) \,, \quad w \sim \GP(0,C_\theta) \,,                                   \\
        \rho_\ell^{t,t'}   & =  -\frac{1}{\delta} \int_{t'}^t R_\theta(t,s) \rho_\ell^{s,t'} (\de s + \sqrt{\tau\delta} \de B^s) - \frac{1}{\delta} R_\theta(t,t') \,,                                 \\
        \Sigma_\ell(t,t')  & = \E[L^t L^{t'}] \,, \quad L^t \coloneqq \int_0^t (r^s - r^* - z) (\de s + \sqrt{\tau\delta} \de B^s) \,, \quad R_\ell(t,t') = \E[\rho_\ell^{t,t'}] \quad (t \geq t') \,.
    \end{align} \label{eq:dmft_ridge}
\end{subequations}

Next, we eliminate the stochastic processes to close the system in terms of the correlation and response functions.
To this end, we discretize time with step size $\gamma$ as in \cref{app:dmft_disc}, reduce the equations, and then take the continuous-time limit $\gamma \to 0$. This operation can be justified along the lines of the proof of \cref{thm:dmft_sgf}.

Let $t_i = i \gamma$ for $i=0,1,\ldots,\lfloor T/\gamma \rfloor$.
The discretized version of \cref{eq:dmft_ridge} is given by
\begin{subequations}
    \begin{align}
        \theta^{t_i}      & = \gamma \sum_{j = 0}^{i-1} \ab( u^{t_i} - (1 + \lambda) \theta^{t_j} + \theta^* - \gamma \sum_{k=0}^{j-1} R_\ell(t_j,t_k) (\theta^{t_k} - \theta^*)) \,, \quad U \sim \GP(0,C_\ell/\delta) \,, \label{eq:theta_ridge} \\
        C_\theta(t_i,t_j) & = \E[\theta^{t_i} \theta^{t_j}] \,, \quad C_\theta(t_i,*) = \E[\theta^{t_i}\theta^*] \,, \quad R_\theta(t_i,t_j) = \gamma^{-1} \E\ab[\diffp{\theta^{t_i}}{u^{t_j}}] \quad (i > j) \,,                                  \\
        r^{t_i}           & = w^{t_i} -\frac{1}{\delta} \sum_{j=0}^{i-1} R_\theta(t_i,t_j) (r^{t_j} - r^* - z) (\gamma + \sqrt{\tau\delta} (B^{t_{j+1}} - B^{t_j})) \,, \quad w \sim \GP(0,C_\theta) \,, \label{eq:r_ridge}                        \\
        C_\ell(t_i,t_j)   & = \E\ab[(r^{t_i} - r^* - z)\ab(1 + \sqrt{\tau\delta}\frac{B^{t_{i+1}} - B^{t_i}}{\gamma})(r^{t_j} - r^* - z)\ab(1 + \sqrt{\tau\delta}\frac{B^{t_{j+1}} - B^{t_j}}{\gamma})] \,,                                        \\
        R_\ell(t_i,t_j)   & = \gamma^{-1} \E\ab[\diffp{r^{t_i}}{w^{t_j}}] \quad (i > j) \,.
    \end{align} \label{eq:dmft_ridge_disc}
\end{subequations}

Let $K = \floor{T/\gamma} + 1$ and let $\bC_\theta,\bR_\theta,\bC_\ell,\bR_\ell \in \reals^{K \times K}$ be the matrices with entries $(\bC_\theta)_{ij} = C_\theta(t_i,t_j)$, $(\bR_\theta)_{ij} = R_\theta(t_i,t_j)$, $(\bC_\ell)_{ij} = C_\ell(t_i,t_j)$, and $(\bR_\ell)_{ij} = R_\ell(t_i,t_j)$ for $i,j = 0,1,\ldots,K-1$.
Let $\bc^* \in \reals^K$ be the vector with entries $c^{*}_i = C_\theta(t_i,*)$ for $i=0,1,\ldots,K-1$.
Let $\btheta,\br,\bu,\bw,\bG \in \reals^K$ be the vectors with entries $\theta^{t_i},r^{t_i},u^{t_i},w^{t_i},G^{t_i} \coloneqq (B^{t_{i+1}} - B^{t_i})/\sqrt{\gamma}$ for $i=0,1,\ldots,K-1$.

\paragraph{Equation for $R_\theta$.}
Differentiating \cref{eq:theta_ridge} with respect to $u^{t_j}$ for $j < i$ and taking the expectation, we have
\begin{align}
    R_\theta(t_i,t_j) & = 1 - \gamma \sum_{k=j}^{i-1} \ab((1 + \lambda) R_\theta(t_k,t_j) + \gamma \sum_{l=0}^{k-1} R_\ell(t_k,t_l) R_\theta(t_l,t_j)) \,. \label{eq:R_theta_ridge_disc}
\end{align}
Let $\bT \coloneqq \reals^{K \times K}$ be the lower-triangular matrix with entries $T_{ij} = 1$ for $i > j$ and $T_{ij} = 0$ for $i \geq j$.
\cref{eq:R_theta_ridge_disc} can be written in matrix form as
\begin{align}
    \bR_\theta = \bT - \gamma \bT ((1 + \lambda) \bR_\theta + \gamma \bR_\ell \bR_\theta) \,.
\end{align}
Solving for $\bR_\theta$, we obtain
\begin{align}
    \bR_\theta = (\bI + \gamma \bT ((1 + \lambda) \bI + \gamma \bR_\ell))^{-1} \bT \,. \label{eq:R_theta_ridge_matrix}
\end{align}

\paragraph{Equation for $C_\theta$.}
Multiplying \cref{eq:theta_ridge} by $\theta^*$ and taking the expectation, we have
\begin{align}
    C_\theta(t_i,*) & = -\gamma \sum_{j=0}^{i-1} \ab( (1 + \lambda) C_\theta(t_j,*) - \rho^2 + \gamma \sum_{k=0}^{j-1} R_\ell(t_j,t_k) (C_\theta(t_k,*) - \rho^2)) \,.
\end{align}
In matrix form, this can be written as
\begin{align}
    \bc^* & = - \gamma \bT ((1 + \lambda) \bc^* - \rho^2 \bone + \gamma \bR_\ell (\bc^* - \rho^2 \bone)) \,.
\end{align}
Solving for $\bc^*$, we obtain
\begin{align}
    \bc^* & = \rho^2 \gamma (\bI + \gamma \bT ((1 + \lambda) \bI + \gamma \bR_\ell))^{-1} \bT (\bI + \gamma \bR_\ell) \bone = \rho^2 \gamma \bR_\theta (\bI + \gamma \bR_\ell) \bone \,. \label{eq:C_star_ridge_matrix}
\end{align}

Multiplying \cref{eq:theta_ridge} by $\theta^{t_j}$ and taking the expectation, we have
\begin{align}
     & C_\theta(t_i,t_j) \notag                                                                                                                                                                                                      \\
     & = \gamma \sum_{k=0}^{i-1} \ab(\E[u^{t_k} \theta^{t_j}] - (1 + \lambda) C_\theta(t_k,t_j) + C_\theta(t_j,*) - \gamma \sum_{l=0}^{j-1} R_\ell(t_k,t_l) (C_\theta(t_l,t_j) - C_\theta(t_j,*))) \,. \label{eq:C_theta_ridge_disc}
\end{align}
By Stein's lemma (the Gaussian integration by parts), we have
\begin{align}
    \E[u^{t_k} \theta^{t_j}] & = \sum_{l=0}^{j-1} \Cov(u^{t_k},u^{t_l}) \E\ab[\diffp{\theta^{t_j}}{u^{t_l}}] = \frac{\gamma}{\delta} \sum_{l=0}^{j-1} C_\ell(t_k,t_l) R_\theta(t_j,t_l) \,.
\end{align}
Therefore, \cref{eq:C_theta_ridge_disc} can be written in matrix form as
\begin{align}
    \bC_\theta & = \gamma \bT \ab(\frac{\gamma}{\delta} \bC_\ell \bR_\theta^\transpose - (1 + \lambda) \bC_\theta + \bone \bc^{*\transpose} - \gamma \bR_\ell (\bC_\theta - \bone \bc^{*\transpose})) \,.
\end{align}
Solving for $\bC_\theta$, we obtain
\begin{align}
    \bC_\theta & = \gamma (\bI + \gamma \bT ((1 + \lambda) \bI + \gamma \bR_\ell))^{-1} \bT (\gamma \bC_\ell \bR_\theta^\transpose + (\bI + \gamma \bR_\ell) \bone \bc^{*\transpose}) \notag \\
               & = \frac{\gamma^2}{\delta} \bR_\theta \bC_\ell \bR_\theta^\transpose + \frac{1}{\rho^2} \bc^* \bc^{*\transpose} \,. \label{eq:C_theta_ridge_matrix}
\end{align}
Let $C_\theta^*(t_i,t_j) \coloneqq \E[(\theta^{t_i} - \theta^*)(\theta^{t_j} - \theta^*)] = C_\theta(t_i,t_j) - C_\theta(t_i,*) - C_\theta(t_j,*) + \rho^2$.
Let $\bC_\theta^* \in \reals^{K \times K}$ be the matrix with entries $(\bC_\theta^*)_{ij} = C_\theta^*(t_i,t_j)$ for $i,j = 0,1,\ldots,K-1$.
From \cref{eq:C_theta_ridge_matrix}, we have
\begin{align}
    \bC_\theta^* & = \bC_\theta - \bc^* \bone^\transpose - \bone \bc^{*\transpose} + \rho^2 \bone \bone^\transpose = \frac{\gamma^2}{\delta} \bR_\theta \bC_\ell \bR_\theta^\transpose + \frac{1}{\rho^2} (\bc^* - \rho^2 \bone) (\bc^* - \rho^2 \bone)^\transpose \,. \label{eq:C_theta_star_ridge_matrix}
\end{align}

\paragraph{Equation for $R_\ell$.}
Differentiating \cref{eq:r_ridge} with respect to $w^{t_j}$ for $j < i$ and taking the expectation, we have
\begin{align}
    R_\ell(t_i,t_j) & = -\frac{1}{\delta} \sum_{k=j}^{i-1} R_\theta(t_i,t_k) R_\ell(t_k,t_j) - \frac{1}{\delta} R_\theta(t_i,t_j) \,. \label{eq:R_ell_ridge_disc}
\end{align}
This can be written in matrix form as
\begin{align}
    \bR_\ell = -\frac{1}{\delta} \bR_\theta \bR_\ell - \frac{1}{\delta} \bR_\theta \,.
\end{align}
Solving for $\bR_\ell$, we obtain
\begin{align}
    \bR_\ell = -(\delta \bI + \gamma \bR_\theta)^{-1} \bR_\theta \,. \label{eq:R_ell_ridge_matrix}
\end{align}

\paragraph{Equation for $C_\ell$.}
Since $G^{t_i} \sim \normal(0,1)$ are i.i.d.\ standard normal variables, we have
\begin{align}
    C_\ell(t_i,t_j) & = \E[(r^{t_i} - r^* - z) (r^{t_j} - r^* - z)] + + \frac{\tau\delta}{\gamma} \E[(r^{t_i} - r^* - z)^2] \ind(i = j) \notag                                 \\
                    & \qquad + \sqrt{\frac{\tau\delta}{\gamma}} \ab(\E[(r^{t_i} - r^* - z) (r_{t_j} - r^* - z) G^{t_j}] + \E[(r^{t_i} - r^* - z) (r^{t_j} - r^* - z) G^{t_i}]) \\
                    & = L(t_i,t_j) + \frac{\tau\delta}{\gamma} L(t_i,t_i) \ind(i = j) + \sqrt{\tau\delta} (M(t_i,t_j) + M(t_j,t_i)) \,, \label{eq:C_ell_ridge_disc}
\end{align}
where we defined
\begin{align}
    L(t_i,t_j) & = \E[(r^{t_i} - r^* - z) (r^{t_j} - r^* - z)] \,,                                                                                           \\
    M(t_i,t_j) & = \gamma^{-1/2} \E[(r^{t_i} - r^* - z) (r^{t_j} - r^* - z) G^{t_j}] = \gamma^{-1/2} \E\ab[\diffp{r^{t_i}}{G^{t_j}} (r^{t_j} - r^* - z)] \,.
\end{align}
In the definition of $M(t_i,t_j)$, we used Stein's lemma.

We first derive a closed equation for $M(t_i,t_j)$. Differentiating \cref{eq:r_ridge} with respect to $G^{t_j}$, multiplying by $\gamma^{-1/2} (r^{t_j} - r^* - z)$, and taking the expectation, we have
\begin{align}
    M(t_i,t_j) & = -\frac{\gamma}{\delta} \sum_{k=j}^{i-1} R_\theta(t_i,t_k) M(t_k,t_j) - \sqrt{\frac{\tau}{\delta}} R_\theta(t_i,t_j) L(t_j,t_j) \,.
\end{align}
In matrix form, this can be written as
\begin{align}
    \bM & = -\frac{\gamma}{\delta} \bR_\theta \bM - \sqrt{\frac{\tau}{\delta}} \bR_\theta \diag(\bL) \,,
\end{align}
where $\bM,\bL\in\reals^{K \times K}$ are the matrices with entries $M_{ij} = M(t_i,t_j)$ and $L_{ij} = L(t_i,t_j)$, and $\diag(\bL)$ is the diagonal matrix with diagonal entries equal to those of $\bL$.
Solving for $\bM$, we obtain
\begin{align}
    \bM & = -\sqrt{\tau\delta} (\delta \bI + \gamma \bR_\theta)^{-1} \bR_\theta \diag(\bL) = \sqrt{\tau\delta} \bR_\ell \diag(\bL) \,. \label{eq:M_ridge_matrix}
\end{align}

Next, we derive a closed equation for $L(t_i,t_j)$. Multiplying \cref{eq:r_ridge} by $r^{t_j} - r^* - z$ and taking the expectation, we have
\begin{align}
    L(t_i,t_j) & = \E[(w^{t_i} - r^* - z) (r^{t_j} - r^* - z)] - \frac{\gamma}{\delta} \sum_{k=0}^{i-1} R_\theta(t_i,t_k) (L(t_k,t_j) + \sqrt{\tau\delta} M(t_j,t_k)) \,. \label{eq:L_ridge_disc}
\end{align}
By Stein's lemma, we have
\begin{align}
     & \E[(w^{t_i} - r^* - z) (r^{t_j} - r^* - z)] \notag                                                                                                           \\
     & = \sum_{k=0}^{j} \Cov(w^{t_i}-r^*-z,w^{t_k}) \E\ab[\diffp{r^{t_j}}{w^{t_k}}] + \Cov(w^{t_i}-r^*-z,r^* + z) \ab(\E\ab[\diffp{r^{t_j}}{(r^* + z)}] - 1) \notag \\
     & = \gamma \sum_{k=0}^{j-1} (C_\theta(t_i,t_k) - C_\theta(t_k,*)) R_\ell(t_j,t_k) + C_\theta(t_i,t_j) - C_\theta(t_j,*) \notag                                 \\
     & \qquad - (C_\theta(t_i,*) - \rho^2 - \sigma^2) \ab(\gamma \sum_{k=0}^{j-1} R_\ell(t_j,t_k) + 1) \notag                                                       \\
     & = C_\theta(t_i,t_j) - C_\theta(t_i,*) - C_\theta(t_j,*) + \rho^2 + \sigma^2 \notag                                                                           \\
     & \qquad + \gamma \sum_{k=0}^{j-1} (C_\theta(t_i,t_k) - C_\theta(t_i,* ) - C_\theta(t_k,*) + \rho^2 + \sigma^2) R_\ell(t_j,t_k) \notag                         \\
     & = C_\theta^*(t_i,t_j) + \sigma^2 + \gamma \sum_{k=0}^{j-1} (C_\theta^*(t_i,t_k) + \sigma^2) R_\ell(t_j,t_k) \,.
\end{align}
Therefore, \cref{eq:L_ridge_disc} can be written in matrix form as
\begin{align}
    \bL & = (\bC_\theta^* + \sigma^2 \bone \bone^\transpose) (\bI + \gamma \bR_\ell^\transpose) - \frac{\gamma}{\delta} \bR_\theta (\bL + \sqrt{\tau\delta} \bM^\transpose) \,.
\end{align}
Solving for $\bL$, we obtain
\begin{align}
    \bL & = \delta (\delta \bI + \gamma \bR_\theta)^{-1} (\bC_\theta^* + \sigma^2 \bone \bone^\transpose) (\bI + \gamma \bR_\ell^\transpose) - \sqrt{\tau\delta} \gamma (\delta \bI + \gamma \bR_\theta)^{-1} \bR_\theta \bM^\transpose \notag \\
        & = (\bI + \gamma \bR_\ell) (\bC_\theta^* + \sigma^2 \bone \bone^\transpose) (\bI + \gamma \bR_\ell^\transpose) + \tau\delta \gamma \bR_\ell \diag(\bL) \bR_\ell^\transpose \,, \label{eq:L_ridge_matrix}
\end{align}
where we used that
\begin{align}
    \bI + \gamma \bR_\ell = \bI - \gamma (\delta \bI + \gamma \bR_\theta)^{-1} \bR_\theta & = \bI - (\delta \bI + \gamma \bR_\theta)^{-1} (\delta \bI + \gamma \bR_\theta - \delta \bI) \notag \\
                                                                                          & = \delta (\delta \bI + \gamma \bR_\theta)^{-1} \,.
\end{align}

\paragraph{Continuous-time limit.}
Taking the continuous-time limit $\gamma \to 0$ in \cref{eq:R_theta_ridge_disc,eq:C_star_ridge_matrix,eq:C_theta_star_ridge_matrix}, we obtain
\begin{align}
    R_\theta(t,t')   & = 1 - \int_{t'}^t \ab((1 + \lambda) R_\theta(s,t') + \int_{t'}^s R_\ell(s,s') R_\theta(s',t') \de s') \de s \,, \label{eq:R_theta_closed}                                                                  \\
    C_\theta(t,*)    & = \rho^2 \int_0^t R_\theta(t,s) \ab(1 + \int_0^s R_\ell(s,s') \de s') \de s \,, \label{eq:C_star_closed}                                                                                                   \\
    C_\theta^*(t,t') & = \frac{1}{\delta} \int_0^{t} \int_0^{t'} R_\theta(t,s) R_\theta(t',s') C_\ell(s,s') \de s' \de s + \frac{1}{\rho^2} (C_\theta(t,*) - \rho^2) (C_\theta(t',*) - \rho^2) \,. \label{eq:C_theta_star_closed}
\end{align}
Taking the continuous-time limit $\gamma \to 0$ in \cref{eq:R_ell_ridge_disc,eq:C_ell_ridge_disc,eq:L_ridge_matrix}, we obtain
\begin{align}
    R_\ell(t,t') & =  -\frac{1}{\delta} \int_{t'}^t R_\theta(t,s) R_\ell(s,t') \de s - \frac{1}{\delta} R_\theta(t,t') \,, \label{eq:R_ell_closed}              \\
    C_\ell(t,t') & = L(t,t') + \tau\delta (R_\ell(t',t) L(t',t') + R_\ell(t,t') L(t,t) + L(t,t) \dirac(t-t'))  \,, \label{eq:C_ell_closed}                    \\
    L(t,t')      & = \int_0^t \int_0^{t'} (\dirac(t - s) + R_\ell(t,s)) (\dirac(t' - s') + R_\ell(t',s')) (C_\theta^*(s,s') + \sigma^2) \de s' \de s \notag \\
                 & \qquad + \tau\delta \int_0^{t \wedge t'} R_\ell(t,s) R_\ell(t',s) L(s,s) \de s \,. \label{eq:L_closed}
\end{align}
Here, $\dirac(\cdot)$ is the Dirac delta function.

\subsubsection{Solving the DMFT Equations}

We now solve the DMFT equations derived in the previous section for ridge regression.

\begin{lemma} \label{lem:dmft_ridge_solution}
    Define the function $K_i(t)$ for $i \geq 0$ as
    \begin{align}
        K_i(t) & \coloneqq \int x^i \napier^{-(x + \lambda) t} \de \mu_\MP(x) \,,
    \end{align}
    where $\mu_\MP$ is the \emph{Marchenko--Pastur law}, whose density is given by \cref{eq:mp_law}.

    The solution of the DMFT equations for ridge regression in \cref{eq:R_theta_closed,eq:L_closed} is given by
    \begin{align}
        R_\theta(t,t')   & = K_0(t - t') \,, \quad R_\ell(t,t') = -\frac{1}{\delta} K_1(t - t') \,,                                          \\
        C_\theta^*(t,t') & = C_{\theta 0}^*(t,t') + \tau \int_0^{t \wedge t'} K_1(t + t' - 2s) L(s,s) \de s \,, \label{eq:C_theta_ridge_sol} \\
        L(t,t')          & = L_0(t,t') + \tau \int_0^{t \wedge t'} K_2(t + t' - 2s) L(s,s) \de s \,, \label{eq:L_ridge_sol}
    \end{align}
    where $C_{\theta 0}^*$ and $L_0$ are given by
    \begin{align}
        C_{\theta 0}^*(t,t') & = \rho^2 \int \frac{(\lambda + x\napier^{-(x + \lambda) t}) (\lambda + x\napier^{-(x + \lambda) t'})}{(x + \lambda)^2} \de\mu_\MP(x) \notag                                                        \\
                             & \qquad + \frac{\sigma^2}{\delta} \int \frac{x}{(x + \lambda)^2} (1 - \napier^{-(x + \lambda) t})(1 - \napier^{-(x + \lambda) t'}) \de\mu_\MP(x) \,,                                                \\
        L_0(t,t')            & = \rho^2 \int x \frac{(\lambda + x\napier^{-(x + \lambda) t}) (\lambda + x\napier^{-(x + \lambda) t'})}{(x + \lambda)^2} \de\mu_\MP(x) \notag                                                      \\
                             & \qquad + \frac{\sigma^2}{\delta} \int \frac{(\lambda + x\napier^{-(x + \lambda) t}) (\lambda + x\napier^{-(x + \lambda) t'})}{(x + \lambda)^2} \de\mu_\MP(x) + \frac{\delta-1}{\delta}\sigma^2 \,.
    \end{align}
\end{lemma}

Before proving \cref{lem:dmft_ridge_solution}, we summarize the necessary background on the \emph{Laplace transform}, which is a useful technique for analyzing linear differential equations and will be used extensively in the proof.
Given a function $f \colon \reals_{\geq 0} \to \reals$, its Laplace transform $\mathcal{L}[f] = \bar f$ is defined as
\begin{align}
    \bar f(p) \coloneqq \int_0^\infty f(t) \napier^{-p t} \de t \,,
\end{align}
for $p \in \complex$ with sufficiently large real part for the integral to be convergent.

We state several of its basic properties used in the proof.
\begin{itemize}
    \item \emph{Linearity:} For $f \colon \reals_{\geq 0} \to \reals$ and $a,b \in \reals$, we have $\mathcal{L}[af + b] = a\mathcal{L}[f] + b$.
    \item \emph{Laplace transforms of derivatives, integrals, and convolutions:} For $f,g \colon \reals_{\geq 0} \to \reals$, we have
        \begin{align}
            \mathcal{L}\ab[f'(t)](p)                      & = p \bar f(p) - f(0) \,,  \\
            \mathcal{L}\ab[\int_0^t f(s) \de s](p)        & = \frac{\bar f(p)}{p} \,, \\
            \mathcal{L}\ab[\int_0^t f(t-s) g(s) \de s](p) & = \bar f(p) \bar g(p) \,.
        \end{align}
    \item \emph{Laplace transform of the Dirac delta function:} We have $\mathcal{L}[\dirac(t)](p) = 1$, where $\dirac(t)$ is the Dirac delta function.
\end{itemize}

We also utilize a two-dimensional version of the Laplace transform, which is defined for $g \colon \reals_{\geq 0}^2 \to \reals$ as
\begin{align}
    \bar g(p,q) \coloneqq \int_0^\infty \int_0^\infty g(t,t') \napier^{-p t} \napier^{-q t'} \de t \de t' \,,
\end{align}
for $p,q \in \complex$ with sufficiently large real parts for the integral to be convergent. Its properties are similar to the one-dimensional case.

\begin{proof}
    We proceed as follows.
    First, we solve the equation for the noiseless case $\tau = 0$ in the frequency domain using the Laplace transform.
    Next, we perform the inverse Laplace transform to obtain the time-domain solution for $\tau = 0$. Along the way, we use ideas and techniques from the random matrix theory.
    Finally, we solve the full equations for $\tau > 0$.

    \paragraph{Solving in the frequency domain for $\tau = 0$.}
    Note that the equations for $R_\theta$, $R_\ell$, and $C_\theta(\cdot,*)$ do not depend on $\tau$ and are thus the same for $\tau = 0$ and $\tau > 0$.
    Let $C_{\theta 0}^*$ and $L_0$ be the solutions of \cref{eq:C_theta_star_closed,eq:L_closed} for $\tau = 0$.

    Since the equations for $R_\theta(t,t')$ \eqref{eq:R_theta_closed} and $R_\ell(t,t')$ \eqref{eq:R_ell_closed} depend on time only through the time difference $t - t'$, they are time-translation invariant, i.e., $R_\theta(t,t') = R_\theta(t - t')$ and $R_\ell(t,t') = R_\ell(t - t')$. Therefore, they satisfy the following one-dimensional integral equations:
    \begin{align}
        R_\theta(t) & = 1 - \int_0^t \ab((1 + \lambda) R_\theta(s) + \int_0^s R_\ell(s - s') R_\theta(s') \de s') \de s  \,, \\
        R_\ell(t)   & = -\frac{1}{\delta} R_\theta(t) - \frac{1}{\delta} \int_0^t R_\theta(t - s) R_\ell(s) \de s \,.
    \end{align}
    Taking the Laplace transforms of these equations, we have
    \begin{align}
        \bar R_\theta(p) & = \frac{1}{p} \ab(1 - (1 + \lambda) \bar R_\theta(p) -\bar R_\ell(p) \bar R_\theta(p)) \,,   \\
        \bar R_\ell(p)   & = - \frac{1}{\delta} \bar R_\theta(p) - \frac{1}{\delta} \bar R_\theta(p) \bar R_\ell(p) \,.
    \end{align}
    Rearranging, we obtain
    \begin{align}
        \bar R_\theta(p) & = -\frac{\delta (1 + p + \lambda)-1 - \sqrt{(\delta(1 + p + \lambda) - 1)^2 + 4\delta (p + \lambda)}}{2 (p + \lambda)} \,, \label{eq:R_theta_ridge_freq} \\
        \bar R_\ell(p)   & = -\frac{\bar R_\theta(p)}{\delta + \bar R_\theta(p)} = \frac{(p + \lambda) \bar R_\theta(p) - 1}{\delta} \,. \label{eq:R_ell_ridge_freq}
    \end{align}

    Taking the Laplace transforms of the correlation functions in \cref{eq:C_star_closed,eq:C_theta_star_closed,eq:L_closed}, we have
    \begin{align}
        \bar C_\theta(p,*)       & = \rho^2 \bar R_\theta(p) \cdot \frac{1 + \bar R_\ell(p)}{p} \,,                                                                                                                \\
        \bar C_{\theta 0}^*(p,q) & = \frac{1}{\delta} \bar R_\theta(p) \bar R_\theta(q) \bar L_0(p,q) + \frac{1}{\rho^2} \ab(\bar C_\theta(p,*) - \frac{\rho^2}{p}) \ab(\bar C_\theta(q,*) - \frac{\rho^2}{q}) \,, \\
        \bar L_0(p,q)            & = (1 + \bar R_\ell(p)) (1 + \bar R_\ell(q)) \ab(\bar C_{\theta 0}^*(p,q) + \frac{\sigma^2}{pq}) \,.
    \end{align}
    Simplifying these equations, we obtain
    \begin{align}
        \bar C_{\theta 0}^*(p,q) & = \frac{\rho^2 (p+\lambda)(q+\lambda)}{pq} \frac{\bar R_\theta(p)\bar R_\theta(q)}{1 - \delta \bar R_\ell(p)\bar R_\ell(q)} + \frac{\sigma^2}{pq} \frac{\delta \bar R_\ell(p)\bar R_\ell(q)}{1 - \delta \bar R_\ell(p)\bar R_\ell(q)} \,, \\
        \bar L_0(p,q)            & = \frac{\rho^2 \delta (p+\lambda)(q+\lambda)}{pq} \frac{\delta \bar R_\ell(p)\bar R_\ell(q)}{1 - \delta \bar R_\ell(p)\bar R_\ell(q)} \notag                                                                                              \\
                                 & \qquad + \frac{\sigma^2(p+\lambda)(q+\lambda)}{\delta pq} \frac{\bar R_\theta(p)\bar R_\theta(q)}{1 - \delta \bar R_\ell(p)\bar R_\ell(q)} + \frac{\delta-1}{\delta} \frac{\sigma^2}{pq} \,.
    \end{align}

    Using the relation
    \begin{align}
        1 - \delta \bar R_\ell(p)\bar R_\ell(q) = \frac{(q - p) \bar R_\theta(p) \bar R_\theta(q)}{\bar R_\theta(p) - \bar R_\theta(q)} \,,
    \end{align}
    we have
    \begin{align}
        \frac{\bar R_\theta(p)\bar R_\theta(q)}{1 - \delta \bar R_\ell(p)\bar R_\ell(q)} = \frac{\bar R_\theta(p) - \bar R_\theta(q)}{q - p} \,, \quad \frac{\delta \bar R_\ell(p)\bar R_\ell(q)}{1 - \delta \bar R_\ell(p)\bar R_\ell(q)} = \frac{\bar R_\ell(q) - \bar R_\ell(p)}{q - p} \,.
    \end{align}
    Thus, we obtain
    \begin{align}
        \bar C_{\theta 0}^*(p,q) & = \frac{\rho^2 (p+\lambda)(q+\lambda)}{pq} \frac{\bar R_\theta(p) - \bar R_\theta(q)}{q - p} + \frac{\sigma^2}{pq} \frac{\bar R_\ell(q) - \bar R_\ell(p)}{q - p} \,, \label{eq:C_theta_ridge_freq}                                                                                  \\
        \bar L_0(p,q)            & =  \frac{\rho^2 \delta (p+\lambda)(q+\lambda)}{pq} \frac{\bar R_\ell(q) - \bar R_\ell(p)}{q - p} + \frac{\sigma^2(p+\lambda)(q+\lambda)}{\delta pq} \frac{\bar R_\theta(p) - \bar R_\theta(q)}{q - p} + \frac{\delta-1}{\delta} \frac{\sigma^2}{pq} \,. \label{eq:C_ell_ridge_freq}
    \end{align}

    \paragraph{Solving in the time domain for $\tau = 0$.}
    We now perform the inverse Laplace transform to obtain the time-domain solution.
    Before proceeding, we introduce \emph{Stieltjes transform}.
    The Stieltjes transform $S: \complex \setminus I \to \complex$ of a (signed) measure $\mu$ on an interval $I \subseteq \reals$ is defined as follows:
    \begin{align}
        S(z) = \int_I \frac{1}{x-z} \de \mu(x) \,.
    \end{align}

    We check that the time-domain solution stated in \cref{lem:dmft_ridge_solution} has the same Laplace transform as the frequency-domain solution obtained above.
    First, we take the Laplace transform of $R_\theta$ to obtain
    \begin{align}
        \int_0^\infty R_\theta(t) \napier^{-pt} \de t & = \int_0^\infty \ab(\int \napier^{-(x+\lambda+p) t} \de\mu_\MP(x))\de t = \int \ab(\int_0^\infty \napier^{-(x+\lambda+p) t} \de t) \de\mu_\MP(x) \notag \\
                                                      & = \int \frac{1}{x+\lambda+p} \de\mu_\MP(x) = S_\MP(-(p+\lambda)) \,.
    \end{align}
    Here, $S_\MP(z)$ is the Stieltjes transform of the Marchenko--Pastur law which is given by
    \begin{align}
        S_\MP(z) = \frac{\delta(1 - z)-1 - \sqrt{(\delta(1 - z) - 1)^2 - 4\delta z}}{2 z} \,.
    \end{align}
    Setting $z = -(p + \lambda)$, we obtain $\bar R_\theta(p)$ given in \cref{eq:R_theta_ridge_freq}.

    Next, we check $R_\ell$.
    \begin{align}
        \int_0^\infty R_\ell(t) \napier^{-pt} \de t & = -\frac{1}{\delta} \int \frac{x}{x+p+\lambda} \de\mu_\MP(x) = -\frac{1}{\delta} \ab(-(p + \lambda) S_\MP(-(p+\lambda)) + 1) \,,
    \end{align}
    which is equal to $\bar R_\ell(p)$ in \cref{eq:R_ell_ridge_freq}.

    Finally, we check $C_{\theta 0}^*$ and $L_0$. Define $F_1(t,t'),F_2(t,t')$ as follows:
    \begin{align}
        F_1(t,t') = K_0(t + t') \,, \quad F_2(t,t') = \frac{1}{\delta} K_1(t + t') \,.
    \end{align}
    The Laplace transforms of these functions are
    \begin{align}
        \int_0^\infty \int_0^\infty F_1(t,t') \napier^{-pt-qt'} \de t \de t' & = \int \frac{1}{(x+p+\lambda)(x+q+\lambda)} \de \mu_\MP(x) = \frac{\bar R_\theta(p) - \bar R_\theta(q)}{q - p} \,,              \\
        \int_0^\infty \int_0^\infty F_2(t,t') \napier^{-pt-qt'} \de t \de t' & = \frac{1}{\delta} \int \frac{x}{(x+p+\lambda)(x+q+\lambda)} \de \mu_\MP(x) = \frac{\bar R_\ell(q) - \bar R_\ell(p)}{q - p} \,.
    \end{align}
    Furthermore, we have
    \begin{align}
        \mathcal{L}^{-1}\ab[\frac{(p+\lambda)(q+\lambda)}{pq} \frac{\bar R_\theta(p) - \bar R_\theta(q)}{q - p}] & = \mathcal{L}^{-1}\ab[\frac{(p+\lambda)(q+\lambda)}{pq}\bar F_1(p,q)] \notag                                                                  \\
                                                                                                                 & = \int_0^t \int_0^{t'} (\dirac(t-s)+\lambda)(\dirac(t'-s')+\lambda) K_0(s + s') \, \de s'\de t'   \notag                                  \\
                                                                                                                 & = \int \frac{(\lambda + x\napier^{-(x+\lambda)t})(\lambda + x\napier^{-(x+\lambda)t'})}{(x+\lambda)^2} \, \de\mu_\MP(x) \,,                   \\
        \mathcal{L}^{-1}\ab[\frac{1}{pq}\frac{\bar R_\ell(q) - \bar R_\ell(p)}{q - p}]                           & = \mathcal{L}^{-1}\ab[\frac{\bar F_2(p,q)}{pq}] = \frac{1}{\delta} \int_0^t \int_0^{t'} K_1(s,s') \de s'\de s \notag                          \\
                                                                                                                 & = \frac{1}{\delta} \int \frac{x}{(x + \lambda)^2} (1 - \napier^{-(x+\lambda)t})(1 - \napier^{-(x+\lambda)t'}) \de\mu_\MP(x) \,,               \\
        \mathcal{L}^{-1}\ab[\frac{(p+\lambda)(q+\lambda)}{pq}\frac{\bar R_\ell(q) - \bar R_\ell(p)}{q - p}]      & = \mathcal{L}^{-1}\ab[\frac{(p+\lambda)(q+\lambda)}{pq}\frac{\bar F_2(p,q)}{pq}] \notag                                                       \\
                                                                                                                 & = \frac{1}{\delta} \int_0^t \int_0^{t'} (\dirac(t-s)+\lambda)(\dirac(t'-s')+\lambda) K_1(s,s') \, \de s'\de t' \notag                     \\
                                                                                                                 & = \frac{1}{\delta} \int x \frac{(\lambda + x\napier^{-(x+\lambda)t})(\lambda + x\napier^{-(x+\lambda)s})}{(x+\lambda)^2} \, \de\mu_\MP(x) \,.
    \end{align}
    Thus, by \cref{eq:C_theta_ridge_freq,eq:C_ell_ridge_freq}, we have the desired expressions for $C_{\theta 0}^*$ and $L_0$.

    \paragraph{Solving for $\tau > 0$.}
    Equations for $\Delta C_\theta^*(t,t') \coloneqq C_\theta^*(t,t') - C_{\theta 0}^*(t,t')$ and $\Delta L(t,t') \coloneqq L(t,t') - L_0(t,t')$ are given by
    \begin{align}
         & \Delta C_\theta^*(t,t') \notag                                                                                                                   \\
         & = \frac{1}{\delta} \int_0^t \int_0^{t'} R_\theta(t-s) R_\theta(t'-s') \Delta L(s,s') \de s' \de s \notag                                         \\
         & + \tau \int_0^t \int_0^{t'} R_\theta(t-s) R_\theta(t'-s) (R_\ell(s-s') L(s',s') + R_\ell(s'-s) L(s,s) + L(s,s') \dirac(s-s')) \de s' \de s \,,
    \end{align}
    and
    \begin{align}
        \Delta L(t,t') & = \frac{1}{\delta} \int_0^t \int_0^{t'} (\dirac(t-s)+R_\ell(t-s)) (\dirac(t'-s')+R_\ell(t'-s')) \Delta C_\theta^*(s,s') \de s' \de s \notag \\
                       & \qquad + \tau \int_0^t \int_0^{t'} R_\ell(t-s) R_\ell(t'-s) L(s,s) \de s \,.
    \end{align}
    Define $L'(t,t') \coloneqq L(t,t) \dirac(t-t')$.
    Taking the Laplace transform of these equations, we obtain
    \begin{align}
        \Delta \bar C_\theta^*(p,q) & = \frac{1}{\delta} \bar R_\theta(p) \bar R_\theta(q) \Delta \bar L(p,q) + \tau \bar R_\theta(p) \bar R_\theta(q) (1 + \bar R_\ell(p) + \bar R_\ell(q)) \bar L'(p,q) \,, \\
        \Delta \bar L(p,q)          & = (1 + \bar R_\ell(p)) (1+\bar R_\ell(q)) \Delta \bar C_\theta^*(p,q) + \tau \delta \bar R_\ell(p) \bar R_\ell(q) \bar L'(p,q) \,.
    \end{align}
    Further simplification gives
    \begin{align}
        \Delta \bar C_\theta^*(p,q) & = \frac{\delta^2 \bar R_\ell(p) \bar R_\ell(q)}{1 - \delta \bar R_\ell(p) \bar R_\ell(q)} \tau \bar L'(p,q) = \frac{\delta (\bar R_\ell(q) - \bar R_\ell(p))}{q - p} \tau \bar L'(p,q) \,, \\
        \Delta \bar L(p,q)          & = \frac{\delta (\delta (\bar R_\ell(p) + \bar R_\ell(q)) + \delta + 1) R_\ell(p) R_\ell(q)}{1 - \delta \bar R_\ell(p) \bar R_\ell(q)} \tau \bar L'(p,q) \notag                             \\
                                    & = \frac{(\delta R_\ell(q)^2 + (\delta+1) R_\ell(q)) - (\delta R_\ell(p)^2 + (\delta+1) R_\ell(p))}{q - p} \tau \bar L'(p,q) \,.
    \end{align}
    We perform the inverse Laplace transform. For $\Delta \bar C_\theta^*(p,q)$, we have
    \begin{align}
        \Delta C_\theta^*(t,t') & = \tau \delta \int_0^t \int_0^{t'} F_2(t-s,t'-s') L'(s,s') \de s' \de s = \tau \int_0^{t \wedge t'} K_1(t + t' - 2s) L(s,s) \de s \,,
    \end{align}
    and we obtain \cref{eq:C_theta_ridge_sol}.
    For $\Delta \bar L(p,q)$, we use that
    \begin{align}
        \int_0^\infty \int_0^{\infty} K_2(t + t') \napier^{-pt-qt'} \de t' \de t & = \int \frac{x^2}{(x + \lambda + p)(x + \lambda + q)} \de\mu_\MP(x) \notag                                    \\
                                                                                 & = \frac{(\delta R_\ell(q)^2 + (\delta+1) R_\ell(q)) - (\delta R_\ell(p)^2 + (\delta+1) R_\ell(p))}{q - p} \,,
    \end{align}
    and proceeding similarly, we obtain \cref{eq:L_ridge_sol}.

\end{proof}

Finally, we prove \cref{prop:dmft_ridge}.
The asymptotic train and test errors can be expressed in terms of the DMFT solution as
\begin{align}
    \Etrain(\btheta^t) & = \frac{1}{n} \sum_{i=1}^n (\bx_i^\transpose \btheta^t - y_i)^2  \to \E[(r^t - r^* - z)^2] = L(t,t) \,,                               \\
    \Etest(\btheta^t)  & = \frac{1}{d} \norm{\btheta^t - \btheta^*}_2^2 + \sigma^2 \to \E[(\theta^t - \theta^*)^2] + \sigma^2 = C_\theta^*(t,t) + \sigma^2 \,.
\end{align}
By \cref{lem:dmft_ridge_solution}, the asymptotic train and test errors $\Etrain(t) \coloneqq L(t,t)$ and $\Etest(t) \coloneqq C_\theta^*(t,t) + \sigma^2$ satisfy
\begin{align}
    \Etrain(t) = \Etrain_0(t) + \tau \int_0^t K_2(2(t-s)) \Etrain(s) \de s \,, \quad \Etest(t) = \Etest_0(t) + \tau \int_0^t K_1(2(t-s)) \Etrain(s) \de s \,,
\end{align}
where $\Etrain_0(t) \coloneqq L_0(t,t)$ and $\Etest_0(t) \coloneqq C_{\theta 0}^*(t,t) + \sigma^2$ are the asymptotic train and test errors for the noiseless case $\tau = 0$, which can be written explicitly as
\begin{align}
    \Etrain_0(t) & = \rho^2 \int \frac{x (\lambda + x\napier^{-(x+\lambda)t})^2}{(x+\lambda)^2} \de\mu_\MP(x) + \frac{\sigma^2}{\delta} \int \frac{(\lambda + x\napier^{-(x+\lambda)t})^2}{(x+\lambda)^2} \de\mu_\MP(x) + \frac{\delta-1}{\delta}\sigma^2 \,, \\
    \Etest_0(t)  & = \rho^2 \int \frac{(\lambda + x\napier^{-(x+\lambda)t})^2}{(x+\lambda)^2} \de\mu_\MP(x) + \frac{\sigma^2}{\delta} \int \frac{x}{(x+\lambda)^2} (1 - \napier^{-(x+\lambda)t})^2 \de\mu_\MP(x) + \sigma^2 \,.
\end{align}
Setting $H_i(t) \coloneqq K_i(2t)$, we obtain Equations \eqref{eq:dmft_ridge_errors}--\eqref{eq:dmft_ridge_test0}. This concludes the proof of \cref{prop:dmft_ridge}.

\section{Details of Numerical Simulations}
\label{app:numerics_details}

We numerically solve the discretized DMFT equation given in \cref{app:dmft_disc} using Monte Carlo sampling.
Instead of directly working with the system $\mathfrak{S}^\gamma$ given in \cref{app:dmft_disc}, we work with the equivalent system \eqref{eq:dmft_informal_disc}, as it is simpler to implement.
We solve \cref{eq:dmft_informal_disc} by iterating the following steps until convergence:
\begin{enumerate}
    \setcounter{enumi}{-1}
    \item Start with a random guess of the DMFT solution $(C_\theta,R_\theta)$.
    \item Given the current estimate of $(C_\theta,R_\theta)$, sample $M$ instances of the stochastic processes $r^{t_i}$ and $\difsp{\ell_{t_i}(r^{t_i};z)}{w^{t_j}}$, and compute the functions $(C_\ell, R_\ell, \Gamma)$ by averaging over the samples.
    \item Given the current estimate of $(C_\ell,R_\ell,\Gamma)$, sample $M$ instances of the stochastic processes $\theta^{t_i}$ and $\difsp{\theta^{t_i}}{u^{t_j}}$, and compute the functions $\tilde C_\theta$ and $\tilde R_\theta$ by averaging over the samples.
    \item Update the DMFT solution as $(C_\theta,R_\theta) \leftarrow (1 - \alpha) (C_\theta,R_\theta) + \alpha (\tilde C_\theta, \tilde R_\theta)$ where $\alpha \in (0,1]$ is a damping factor.
\end{enumerate}
In our experiments, we set the number of samples to $M = 8000$, the damping factor to $\alpha = 0.8$, and the time step to $\gamma = 0.05$.
We observe that the above iteration converges in around 10 iterations for the settings considered in this paper.

For the logistic regression setting with $\ell_t(r,r^*;z) = -y/(1+\exp(yr))$ where $y=\sign(r^*+z)$, there is an issue with computing the function $R_\ell(t_i,*) = \E[\difsp{\ell_{t_i}(r^{t_i},r^*;z)}{r^*}]$ due to the non-differentiability of $\ell$ with respect to $r^*$ at $r^* = -z$.
However, we can avoid the differentiation by $r^*$ entirely by (heuristically) using Stein's lemma:
\begin{align}
    \E[r^* \ell_{t_i}(r^{t_i},r^*;z)] & = \sum_{j=0}^{i-1} \Cov(r^*,w^{t_j}) \E\ab[\diffp{\ell_{t_i}(r^{t_i},r^*;z)}{w^{t_j}}] + \Cov(r^*,r^*) \E\ab[\diffp{\ell_{t_i}(r^{t_i},r^*;z)}{r^*}] \notag \\
                                      & = \gamma \sum_{j=0}^{i-1} C_\ell(t_j,*) R_\ell(t_i,t_j) + \rho^2 R_\ell(t_i,*) \,,
\end{align}
and computing $R_\ell(t_i,*)$ as
\begin{align}
    R_\ell(t_i,*) = \frac{1}{\rho^2} \ab(\E[r^* \ell_{t_i}(r^{t_i},r^*;z)] - \gamma \sum_{j=0}^{i-1} R_\ell(t_i,t_j) C_\ell(t_j,*)) \,,
\end{align}
and calculating the expectation using Monte Carlo integration.

\section{Discretization of SDEs in High Dimensions}
\label{app:sde_discretize}

In this section, we analyze the discretization error of a general stochastic differential equation in high dimensions.
The results presented here are used in the proof of \cref{lem:sgf_discrete}.

Let $T > 0$, $\bb \colon [0,T] \times \reals^d \to \reals^d$, and $\bsigma \colon [0,T] \times \reals^d \to \reals^{d \times n}$.
Consider the following SDE in $\reals^d$:
\begin{align}
    \de \btheta^t = \bb(t,\btheta^t) \de t + \bsigma(t,\btheta^t) \de \bW^t \,, \label{eq:sde_general}
\end{align}
with given initial condition $\btheta^0 \in \reals^d$. Here, $\bW^t\in\reals^n$ is a Brownian motion.

We discretize the SDE using the Euler--Maruyama method with step size $\gamma > 0$ as follows:
\begin{align}
    \btheta_\gamma^{t_{k+1}} = \btheta_\gamma^{t_k} + \bb(t_k,\btheta_\gamma^{t_k}) \gamma + \bsigma(t_k,\btheta_\gamma^{t_k}) (\bW^{t_{k+1}} - \bW^{t_k}) \,, \quad \btheta_\gamma^0 = \btheta^0 \,, \label{eq:sde_general_disc}
\end{align}
where $t_k \coloneqq k\gamma$.

There is a standard result \citep[Theorem 10.2.2]{kloeden1992numerical} for bounding the difference between $\btheta^t$ and $\btheta_\gamma^t$. However, they treat $n,d$ as constants, and the dependence of the bound on $n,d$ is not obvious. Here, we show a bound that explicitly tracks the dependence on $n,d$.

\begin{assumption} \label{ass:sde_discretize}
    There exists a constant $L > 0$ independent of $n$ and $d$ such that the following hold.
    \begin{enumerate}
        \item (Proportional asymptotics): $1 / L \leq n/d \leq L$.
        \item (Lipschitz continuity): $\norm{\bb(t_1,\bx_1) - \bb(t_2,\bx_2)}_2^2 + \frobnorm{\bsigma(t_1,\bx_1) - \bsigma(t_2,\bx_2)}^2 \leq L(d(t_1-t_2)^2 + \norm{\bx_1-\bx_2}_2^2)$ for $t_1,t_2 \in [0,T], \, \bx_1,\bx_2 \in \reals^d$.
        \item (Linear growth): $\norm{\bb(t,\bx)}_2^2 + \frobnorm{\bsigma(t,\bx)}^2 \leq L(d + dt^2 + \norm{\bx}_2^2)$ for $t \in [0,T], \, \bx \in \reals^d$.
    \end{enumerate}
\end{assumption}

\begin{lemma}[Norm bound on the discretized SDE iterates] \label{lem:sde_disc_norm_bound}
    Under \cref{ass:sde_discretize} and $\gamma < 1$, there exists a constant $C \coloneqq C(L) > 0$ independent of $n,d,T,\gamma$ such that the following holds.
    \begin{align}
        \E\ab[\max_{0 \leq k \leq \floor{T/\gamma}} \norm{\btheta_\gamma^{t_k}}_2^2] \leq \napier^{C T}(T d + \norm{\btheta^0}_2^2) \,. \label{eq:sde_disc_norm_bound}
    \end{align}
\end{lemma}

\begin{proof}
    For notational simplicity, we denote $K \coloneqq \floor{T/\gamma}$, $\btheta^k \coloneqq \btheta_\gamma^{t_k}$, $\bb^k \coloneqq \bb(t_k,\btheta_\gamma^{t_k})$, $\bsigma^k \coloneqq \bsigma(t_k,\btheta_\gamma^{t_k})$, and $\bxi^k \coloneqq \bW^{t_{k+1}} - \bW^{t_k}$.
    From \cref{eq:sde_general_disc}, we have
    \begin{align}
        \norm{\btheta^{k+1}}_2^2 = \norm{\btheta^k}_2^2 + 2 \gamma \btheta^{k\transpose} \bb^k + \gamma^2 \norm{\bb^k}_2^2 + \norm{\bsigma^k \bxi^k}_2^2 + 2 (\btheta^{k} + \gamma \bb^{k})^\transpose \bsigma^k \bxi^k  \,.
    \end{align}
    Summing over $k$ steps, we have
    \begin{align}
        \norm{\btheta^{k}}_2^2 = \norm{\btheta^0}_2^2 + \underbrace{\sum_{j=0}^{k-1} \ab(2 \gamma \btheta^{j\transpose} \bb^j + \gamma^2 \norm{\bb^j}_2^2 + \norm{\bsigma^j \bxi^j}_2^2)}_{\eqqcolon A^k} + \underbrace{\sum_{j=0}^{k-1} 2 (\btheta^{j} + \gamma \bb^{j})^\transpose \bsigma^j \bxi^j}_{\eqqcolon M^k} \,.
    \end{align}
    We first bound $A^k$. We have
    \begin{align}
        \max_{0 \leq k \leq K} \abs{A^k} \leq \sum_{j=0}^{K-1} \ab(2 \gamma \abs{\btheta^{j\transpose} \bb^j} + \gamma^2 \norm{\bb^j}_2^2 + \norm{\bsigma^j \bxi^j}_2^2) \,.
    \end{align}
    By the linear growth condition, we have
    \begin{align}
        2 \gamma \abs{\btheta^{j\transpose} \bb^j} + \gamma^2 \norm{\bb^j}_2^2 & \leq 2\gamma \norm{\btheta^j}_2^2 + (2\gamma + \gamma^2) \norm{\bb^j}_2^2 \leq C \gamma ((1 + T^2) d + \norm{\btheta^j}_2^2) \,.
    \end{align}
    Using the covariance of $\bxi^j \sim \normal(0,\gamma \bI)$ and the linear growth condition, we have
    \begin{align}
        \E[\norm{\bsigma^j \bxi^j}_2^2 \mid \calF^j] = \gamma \frobnorm{\bsigma^j}^2 \leq C \gamma ((1 + T^2) d + \norm{\btheta^j}_2^2) \,,
    \end{align}
    where $\calF^j$ is the filtration generated by $\{\bxi^0,\ldots,\bxi^{j-1}\}$.
    Thus, taking the expectation, we have
    \begin{align}
        \E\ab[\max_{0 \leq k \leq K} \abs{A^k}] \leq C \gamma \sum_{j=0}^{K-1} ((1 + T^2) d + \E\norm{\btheta^j}_2^2) \,. \label{eq:sde_disc_drift_bound}
    \end{align}
    Next, we bound $M^k$. Notice that $M^k$ is a martingale with respect to the filtration $\calF^k$ since $\E[\bxi^j \mid \calF^j] = 0$.
    Thus, by the Burkholder--Davis--Gundy inequality, we have
    \begin{align}
        \E\ab[\max_{0 \leq k \leq K} \abs{M^k}] & \leq C \E\ab(\sum_{j=0}^{K-1} \E\ab[\ab(2 (\btheta^j + \gamma \bb^j)^\transpose \bsigma^j \bxi^j)^2 \mid \calF^j])^{1/2} \notag                                          \\
                                                & = C \E\ab(\sum_{j=0}^{K-1} 4 \gamma \norm{(\btheta^j + \gamma \bb^j)^\transpose \bsigma^j}_2^2)^{1/2} \notag                                                           \\
                                                & = C \E\ab(\gamma \sum_{j=0}^{K-1} \norm{\btheta^j + \gamma \bb^j}_2^2 \frobnorm{\bsigma^j}^2)^{1/2} \notag                                                             \\
                                                & \leq C \E\ab[\ab(\max_{0 \leq k \leq K} \norm{\btheta^k + \gamma \bb^k}_2^2)^{1/2} \ab(\gamma \sum_{j=0}^{K-1} \frobnorm{\bsigma^j}^2)^{1/2}] \notag                   \\
                                                & \leq C \E\ab[\ab(\gamma^2(1 + T^2) d + \max_{0 \leq k \leq K} \norm{\btheta^k}_2^2)^{1/2} \ab(\gamma \sum_{j=0}^{K-1} ((1 + T^2)d + \norm{\btheta^j}_2^2) )^{1/2}] \,,
    \end{align}
    where we used the linear growth condition in the last line. By Young's inequality, we have
    \begin{align}
        \E\ab[\max_{0 \leq k \leq K} \abs{M^k}] & \leq \frac{1}{2} \ab(\gamma^2 (1 + T^2) d + \E\ab[\max_{0 \leq k \leq K} \norm{\btheta^k}_2^2]) + C \gamma \sum_{j=0}^{K-1} ((1 + T^2)d + \E\norm{\btheta^j}_2^2) \,. \label{eq:sde_disc_martingale_bound}
    \end{align}
    Combining \cref{eq:sde_disc_drift_bound,eq:sde_disc_martingale_bound}, we have
    \begin{align}
        \E\ab[\max_{0 \leq k \leq K} \norm{\btheta^k}_2^2] \leq \norm{\btheta^0}_2^2 + C\gamma^2 (1 + T^2) d + C \gamma \sum_{j=0}^{K-1} ((1 + T^2) d + \E\norm{\btheta^j}_2^2) + \frac{1}{2} \E\ab[\max_{0 \leq k \leq K} \norm{\btheta^k}_2^2] \,.
    \end{align}
    Rearranging the terms, we have
    \begin{align}
        \E\ab[\max_{0 \leq k \leq K} \norm{\btheta^k}_2^2] \leq C(T (1 + T^2) d + \norm{\btheta^0}_2^2) + C \gamma \sum_{j=0}^{K-1} \E\ab[\max_{0 \leq j \leq k} \norm{\btheta^j}_2^2] \,.
    \end{align}
    The bound \eqref{eq:sde_disc_norm_bound} follows by applying Gr\"onwall's inequality and absorbing $1 + T^2$ into the exponential factor $\napier^{C T}$.
\end{proof}

\begin{remark}
    \Cref{lem:sde_disc_norm_bound} holds verbatim to the case where the Gaussian increments $\bW^{t_{k+1}} - \bW^{t_k}$ are replaced by independent random vectors $\bxi^k \in \reals^n$ with $\E[\bxi^k] = 0$ and $\E[\bxi^k \bxi^{k\transpose}] = \gamma \bI_n$, as the proof only uses up to the second moment of the increments.
\end{remark}

\begin{lemma}[Strong approximation of SDE] \label{lem:sde_discretize_error}
    Under \cref{ass:sde_discretize} and $\gamma < 1$, there exists a constant $C \coloneqq C(L) > 0$ that does not depend on $n,d,T,\gamma$ such that the following holds.
    \begin{align}
        \E\ab[\sup_{0 \leq t \leq T} \norm{\btheta^t - \btheta_\gamma^{t}}_2^2] \leq \napier^{C T} \gamma (T d + \norm{\btheta^0}_2^2) \,. \label{eq:sde_discretize_error}
    \end{align}
\end{lemma}

\begin{proof}
    \cref{eq:sde_general} can be written as
    \begin{align}
        \btheta^t = \btheta^0 + \int_0^t \bb(s,\btheta^s) \de s + \int_0^t \bsigma(s,\btheta^s) \de \bW^s \,. \label{eq:sde_int}
    \end{align}
    Let $\floor{t} \coloneqq \max\{k\gamma \mid k\gamma \leq t, k \in \naturals\}$ and consider the following stochastic process that embeds \cref{eq:sde_general_disc} into continuous time.
    \begin{align}
        \btheta_\gamma^t = \btheta^0 + \int_0^t \bb(\floor{s},\btheta_\gamma^{\floor{s}}) \de s + \int_0^t \bsigma(\floor{s},\btheta_\gamma^{\floor{s}}) \de \bW^s \,. \label{eq:sde_disc_int}
    \end{align}

    Using It\^o's lemma on $\norm{\btheta^t - \btheta_\gamma^t}_2^2$, we have
    \begin{align}
        \norm{\btheta^t - \btheta_\gamma^t}_2^2 & = 2 \int_0^t (\btheta^s - \btheta_\gamma^s)^{\transpose} (\bb(s,\btheta^s) - \bb(\floor{s},\btheta_\gamma^{\floor{s}})) \de s \notag                  \\
                                                & \qquad + \int_0^t \frobnorm{\bsigma(s,\btheta^s) - \bsigma(\floor{s},\btheta_\gamma^{\floor{s}})}^2 \de s + M^t \,,                                   \\
        M^t                                     & \coloneqq 2 \int_0^t (\btheta^s - \btheta_\gamma^s)^{\transpose} (\bsigma(s,\btheta^s) - \bsigma(\floor{s},\btheta_\gamma^{\floor{s}})) \de \bW^s \,.
    \end{align}
    By the Lipschitz assumption, we have
    \begin{align}
        2 (\btheta^s - \btheta_\gamma^s)^{\transpose} (\bb(s,\btheta^s) - \bb(\floor{s},\btheta_\gamma^{\floor{s}})) & \leq \norm{\btheta^s - \btheta_\gamma^s}_2^2 + \norm{\bb(s,\btheta^s) - \bb(\floor{s},\btheta_\gamma^{\floor{s}})}_2^2 \notag \\
                                                                                                                     & \leq C (\gamma^2 d + \norm{\btheta^s - \btheta_\gamma^{\floor{s}}}_2^2) \,,                                                   \\
        \frobnorm{\bsigma(s,\btheta^s) - \bsigma(\floor{s},\btheta_\gamma^{\floor{s}})}^2                            & \leq C (\gamma^2 d + \norm{\btheta^s - \btheta_\gamma^{\floor{s}}}_2^2) \,.
    \end{align}
    Therefore, we have
    \begin{align}
        \E\ab[\sup_{0 \leq t \leq T}\norm{\btheta^t - \btheta_\gamma^t}_2^2] & \leq C \int_0^T (\gamma^2 d + \E\norm{\btheta^s - \btheta_\gamma^{\floor{s}}}_2^2) \de s + \E\ab[\sup_{0 \leq t \leq T}\abs{M^t}] \,.
    \end{align}
    We next bound the martingale term $M^t$. By the Burkholder--Davis--Gundy inequality,
    \begin{align}
        \E\ab[\sup_{0 \leq t \leq T}\abs{M^t}] & \leq C \E\ab[\ab(4 \int_0^T \norm{(\btheta^s - \btheta_\gamma^s)^{\transpose} (\bsigma(s,\btheta^s) - \bsigma(\floor{s},\btheta_\gamma^{\floor{s}}))}_2^2 \de s)^{1/2}] \notag          \\
                                               & \leq C \E\ab[\ab(\int_0^T \norm{\btheta^s - \btheta_\gamma^s}_2^2 \frobnorm{\bsigma(s,\btheta^s) - \bsigma(\floor{s},\btheta_\gamma^{\floor{s}})}^2 \de s)^{1/2}] \notag                \\
                                               & \leq C \E\ab[\ab(\sup_{0\leq t \leq T} \norm{\btheta^t - \btheta_\gamma^t}_2^2)^{1/2}\ab(\int_0^T (\gamma^2 d + \norm{\btheta^s - \btheta_\gamma^{\floor{s}}}_2^2) \de s)^{1/2}] \notag \\
                                               & \leq \frac{1}{2} \E\ab[\sup_{0\leq t \leq T} \norm{\btheta^t - \btheta_\gamma^t}_2^2] + C \int_0^T (\gamma^2 d + \E\norm{\btheta^s - \btheta_\gamma^{\floor{s}}}_2^2) \de s \,.
    \end{align}
    In the last line, we used Young's inequality.
    Therefore, we have
    \begin{align}
        \E\ab[\sup_{0 \leq t \leq T} \norm{\btheta^t - \btheta_\gamma^t}_2^2] & \leq C \int_0^T (\gamma^2 d + \E\norm{\btheta^s - \btheta_\gamma^{\floor{s}}}_2^2) \de s + \frac{1}{2} \E\ab[\sup_{0\leq t \leq T} \norm{\btheta^t - \btheta_\gamma^t}_2^2] \,.
    \end{align}
    By rearranging the terms, we have
    \begin{align}
        \E\ab[\sup_{0 \leq t \leq T} \norm{\btheta^t - \btheta_\gamma^t}_2^2] & \leq C \gamma^2 Td + C \int_0^T \E\norm{\btheta^s - \btheta_\gamma^{\floor{s}}}_2^2 \de s \,.
    \end{align}

    Finally, we have $\norm{\btheta^s - \btheta_\gamma^{\floor{s}}}_2^2 \leq 2 \norm{\btheta^s - \btheta_\gamma^{s}}_2^2 + 2 \norm{\btheta_\gamma^{s} - \btheta_\gamma^{\floor{s}}}_2^2$ and
    \begin{align}
        \btheta_\gamma^{s} - \btheta_\gamma^{\floor{s}} & = \bb(\floor{s},\btheta_\gamma^{\floor{s}}) (s-\floor{s}) + \bsigma(\floor{s},\btheta_\gamma^{\floor{s}}) (\bW^s - \bW^{\floor{s}}) \,.
    \end{align}
    By the linear growth condition and \cref{lem:sde_disc_norm_bound}, we have
    \begin{multline}
        \E\norm{\btheta_\gamma^{s} - \btheta_\gamma^{\floor{s}}}_2^2 \leq C \gamma^2 \E\norm{\bb(\floor{s},\btheta_\gamma^{\floor{s}})}_2^2 + C \gamma \E\frobnorm{\bsigma(\floor{s},\btheta_\gamma^{\floor{s}})}^2 \\
        \leq C \gamma (d + \E\norm{\btheta_\gamma^{\floor{s}}}_2^2) \leq \napier^{C T} \gamma (d + \norm{\btheta^0}_2^2) \,.
    \end{multline}
    Combining the above bounds, we have
    \begin{align}
        \E\ab[\sup_{0 \leq t \leq T}\norm{\btheta^t - \btheta_\gamma^t}_2^2] \leq \napier^{C T} \gamma (T d + \norm{\btheta^0}_2^2) + C \int_0^T \E\ab[\sup_{0 \leq s \leq t} \norm{\btheta^s - \btheta_\gamma^s}_2^2] \de t \,.
    \end{align}
    Applying Gr\"onwall's inequality, we obtain the bound \eqref{eq:sde_discretize_error}.
\end{proof}

\end{document}